\documentclass[10pt,journal,compsoc]{IEEEtran}
\ifCLASSOPTIONcompsoc
  \usepackage[nocompress]{cite}
\else
  \usepackage{cite}
\fi
\ifCLASSINFOpdf
  \usepackage[pdftex]{graphicx}
\else
\fi
\usepackage{amsmath}
\usepackage{amsthm}
\usepackage{amssymb}
\usepackage{amsfonts}
\usepackage{mathtools}
\usepackage{color}
\usepackage[ruled,vlined,linesnumbered,lined,boxed,commentsnumbered]{algorithm2e}
\usepackage{subcaption}
\usepackage{multirow}
\usepackage{booktabs}
\usepackage{ragged2e}
\usepackage{enumitem}
\usepackage{bigints}
\usepackage{mathrsfs}
\usepackage[mathscr]{euscript}
\usepackage{multirow}
\usepackage{multicol}
\usepackage{balance}

\newtheorem{lemma}{Lemma}
\newtheorem{definition}{Definition}

\newtheorem{proposition}{Proposition}

\newcommand*{\Scale}[2][4]{\scalebox{#1}{$#2$}}%
\usepackage{multirow}
\usepackage{makecell}
\usepackage{booktabs}
\usepackage{cellspace}
\usepackage{array}
\usepackage{thmtools}
\usepackage{thm-restate}
\usepackage{hyperref}

\usepackage{titletoc}
\usepackage[page,header]{appendix}
\usepackage{nccmath}
\usepackage[framemethod=tikz]{mdframed}
\definecolor{softred}{HTML}{FE8A71}

\begin{document}
%
\title{A Tale of HodgeRank and Spectral Method: Target Attack Against Rank Aggregation Is\\the Fixed Point of Adversarial Game}
%
%
%
%
\author
{
    Ke~Ma,~\IEEEmembership{Member,~IEEE,}
    Qianqian~Xu$^*$,~\IEEEmembership{Senior Member,~IEEE,}
    Jinshan~Zeng,\\
    Guorong~Li,~
    Xiaochun~Cao,~\IEEEmembership{Senior Member,~IEEE,}
    and~Qingming~Huang$^*$,~\IEEEmembership{Fellow,~IEEE}
    \IEEEcompsocitemizethanks
    {
        \IEEEcompsocthanksitem K. Ma and G. Li are with the School of Computer Science and Technology, University of Chinese Academy of Sciences, Beijing 100049, China. E-mail: make@ucas.ac.cn,\ liguorong@ucas.ac.cn\protect\\
        \IEEEcompsocthanksitem Q. Xu is with the Key Laboratory of Intelligent Information Processing, Institute of Computing Technology, Chinese Academy of Sciences, Beijing 100190, China. E-mail: qianqian.xu@vipl.ict.ac.cn, xuqianqian@ict.ac.cn.\protect\\
        \IEEEcompsocthanksitem J. Zeng is with the School of Computer and Information Engineering, Jiangxi Normal University, Nanchang, Jiangxi 330022, China. E-mail: jinshanzeng@jxnu.edu.cn\protect\\ 
        \IEEEcompsocthanksitem X. Cao is with School of Cyber Science and Technology, Shenzhen Campus, Sun Yat-sen University, Shenzhen 518107, China. E-mail: caoxiaochun@mail.sysu.edu.cn.\protect\\
        \IEEEcompsocthanksitem Q. Huang is with the School of Computer Science and Technology, University of Chinese Academy of Sciences, Beijing 100049, China, also with the Key Laboratory of Big Data Mining and Knowledge Management (BDKM), University of Chinese Academy of Sciences, Beijing 100049, China, also with the Key Laboratory of Intelligent Information Processing, Institute of Computing Technology, Chinese Academy of Sciences, Beijing 100190, China, and also with Peng Cheng Laboratory, Shenzhen 518055, China. E-mail: qmhuang@ucas.ac.cn.\protect\\
        \IEEEcompsocthanksitem $^*$ Corresponding author.
    }
}

%
%

\markboth{Journal of \LaTeX\ Class Files,~Vol.~14, No.~8, August~2015}%
{Shell \MakeLowercase{\textit{et al.}}: Bare Demo of IEEEtran.cls for Computer Society Journals}
%



\IEEEtitleabstractindextext
{%
    \begin{abstract}
        \justifying
        Rank aggregation with pairwise comparisons has shown promising results in elections, sports competitions, recommendations, and information retrieval. However, little attention has been paid to the security issue of such algorithms, in contrast to numerous research work on the computational and statistical characteristics. Driven by huge profit, the potential adversary has strong motivation and incentives to manipulate the ranking list. Meanwhile, the intrinsic vulnerability of the rank aggregation methods is not well studied in the literature. To fully understand the possible risks, we focus on the purposeful adversary who desires to designate the aggregated results by modifying the pairwise data in this paper. From the perspective of the dynamical system, the attack behavior with a target ranking list is a fixed point belonging to the composition of the adversary and the victim. To perform the targeted attack, we formulate the interaction between the adversary and the victim as a game-theoretic framework consisting of two continuous operators while Nash equilibrium is established. Then two procedures against HodgeRank and RankCentrality are constructed to produce the modification of the original data. Furthermore, we prove that the victims will produce the target ranking list once the adversary masters the complete information. It is noteworthy that the proposed methods allow the adversary only to hold incomplete information or imperfect feedback and perform the purposeful attack. The effectiveness of the suggested target attack strategies is demonstrated by a series of toy simulations and several real-world data experiments. These experimental results show that the proposed methods could achieve the attacker's goal in the sense that the leading candidate of the perturbed ranking list is the designated one by the adversary.
    \end{abstract}  
    \begin{IEEEkeywords}
    Adversarial Learning, Pairwise Comparison, Ranking Aggregation.
    \end{IEEEkeywords}
}

\maketitle

\IEEEdisplaynontitleabstractindextext

%
\IEEEpeerreviewmaketitle

\IEEEraisesectionheading{\section{Introduction}\label{sec:introduction}}
\IEEEPARstart{G}{iven} partially observed pairwise comparisons, we are interested in aggregating these partial orders into a full ranking of all candidates. Such a statistical estimation problem arises in various disciplines, including the social choice theory\cite{arrow2012social}, statistics\cite{Jiang2011}, machine learning\cite{pmlr-v54-korba17a,DBLP:journals/ior/NegahbanOS17}, sports\cite{DBLP:conf/nips/HerbrichMG06}, recommendation system\cite{DBLP:conf/sigir/Cao0MAYH18}, information retrieval\cite{DBLP:books/daglib/0027504}, psychology\cite{CRITCHLOW1991294}, bioinformatics\cite{kolde2012robust}, \textit{etc}. Roughly speaking, the rank aggregation approach treats pairwise comparisons as access to estimate the underlying ``qualities'' or ``scores'' of the items being compared (\textit{e.g.}, skill levels of tennis players, preference for political candidates, and interest degree of advertisements). These latent preference scores represent the voters' choices. 

Besides the statistical \cite{DBLP:journals/jmlr/ShahBBPRW16,chen2019spectral} and computational issues \cite{DBLP:conf/icml/AgarwalPA18,DBLP:conf/aistats/0001SR20}, security is a new perspective for understanding the ranking aggregation problem. In pursuit of substantial economic benefits, the potential attackers have strong motivations and incentives to manipulate the aggregated results when the pairwise ranking algorithms are utilized in \textbf{\textit{high-stakes}} applications, \textit{e.g.}, elections (the so-called ``Condorcet winner'' \cite{grofman1986condorcet} is a winner who will win every pairwise comparison), sports competitions (in ``round-robin tournament'', each contestant meets all other contestants in turn), and recommendation (``click-through'' data is always in paired form). A profit-oriented adversary could try his/her best to designate the ranking list to favor his/her demands. The attacker could place the particular object at the top of the recommendation list or help the specific candidate win an election. If the attackers compromise the integrity of ranking results, these high-stakes applications' fairness and rationality will be profoundly damaged. 

However, the possible risk and potential threat of pairwise ranking algorithms have not been comprehensively examined. To the best of our knowledge, the existing adversarial arsenals do not cause damage to the ranking aggregation algorithms. Two significant disparities exist between the rank aggregation and the learning problems (\textit{e.g.} classification and regression).
\begin{itemize}
    \item Considering the process of learning problems, it separates into training and test phases where the used datasets are non-overlapping. Conversely, the {aggregation} problem needs to combine all collected comparisons into a single global ranking list in a one-shot way, and the test protocol does not exist. 
    \item  From the view of data, the typical {learning} data consists of features (always a continuous variable) and the desired output value (also called the supervisory information). The {aggregation} object is a pairwise comparison which involves two candidates and the partial order (\textit{a.k.a} preference) between them. A binary variable always represents such an order.
\end{itemize}
Consequently, the attack methods against {learning} and {aggregation} have their characteristics. The Attack against {learning} could occur in the training phase (poisoning attack \cite{DBLP:conf/icml/BiggioNL12}) and test phase (evasion attack, \textit{a.k.a} adversarial examples \cite{DBLP:journals/corr/GoodfellowSS14}) separately. Meanwhile, the attack against {aggregation}, which modifies the pairwise comparisons, should happen before the combination of pairwise comparisons. Furthermore, the attack against {learning} always conducts a continuous optimization problem which operates a single learning data point into the ``weapon'' of the adversary. Meanwhile, due to its discrete nature, modification of {aggregation} data naturally derives a discrete optimization problem, which decides to insert, delete or flip a pairwise comparison. To achieve the targeted attack against the ranking aggregation, the adversary must analyze the characteristics of the pairwise ranking problem, including the properties of both paired data and the victim algorithms. Consequently, the adversary would predict and induce the victim's behavior with parsimonious actions to achieve his/her goals.
 
Given these challenges, we propose a principle framework for adversarial perturbations of pairwise comparisons, aiming to manipulate the results of rank aggregation algorithms. In particular, we focus on two well-known procedures: least squared approach \cite{Jiang2011} and spectral method \cite{DBLP:journals/ior/NegahbanOS17}. These algorithms give the full ranking based on appropriate estimates of the latent preference scores. We make the following contributions of the paper: 
\begin{itemize}
    \item We identify the purposeful attack process as a fixed point problem in the dynamical system. Such a dynamical system is composed of two continuous operators characterized by the adversary and the victim. This perspective allows us to explore a general understanding of ranking aggregation and adversarial manipulation. 
    \vspace{0.15cm}
    \item The core of the proposed target attack procedure lies in a perspective of the continuous operator with optimal conditions, which models the interaction between the adversary and the rank aggregation algorithms. We prove that the target ranking of the adversary is a fixed point of the dynamical system consisting of the complete information adversary and the specific ranker, \textit{e.g.} \textbf{HodgeRank} and \textbf{RankCentrality}.
    \vspace{0.15cm}
    \item The targeted attack strategy against \textbf{HodgeRank} is established with the help of the optimal condition of the least squared problem. Moreover, we show that the proposed method allows the attacker to successfully manipulate \textbf{HodgeRank} with incomplete information or imperfect feedback.
    \vspace{0.15cm}
    \item Different target attacks against \textbf{RankCentrality} corresponding to the reversible and irreversible stochastic transition matrices are formulated into the inverse eigenvalue problem. The worst-case asymptotic variance analysis can help handle incomplete information or imperfect feedback with constructive solutions.
    \item A complexity model named ``Incremental Pairwise Comparison Oracle'' (\textbf{ICPO}) is proposed to evaluate the efficiency of purposeful attack strategies against rank aggregation methods. 
\end{itemize}

The rest of the paper is structured as follows: in the next section, we introduce the basic concepts of rank aggregation and two representative algorithms as \textbf{HodgeRank} and \textbf{RankCentrality}. We then establish a general framework for manipulating different rank aggregation algorithms in Sec. 3. We identify the manipulation process as a fixed point problem in the dynamical system. Such a dynamical system is composed of two continuous operators characterized by the adversary and the victim. The core of the proposed target attack procedure lies in a perspective of the continuous operator with the victim's optimal condition, which elicits the interaction between the adversary and the victim. We prove that the target ranking is a fixed point of the dynamical system consisting of the complete information adversary and the specific ranker. Based on the proposed framework, we derive the target attack strategies against \textbf{HodgeRank} and \textbf{RankCentrality} in Sec. 4 and Sec. 5 correspondingly. Moreover, we establish the sufficient conditions that the target rankings lists are the fixed points of the two cases by Theorem 2 (\textbf{HodgeRank}), Theorem 4 (\textbf{RankCentrality} with reversible stochastic transition matrix), and Theorem 6 (\textbf{RankCentrality} with irreversible stochastic transition matrix). To the best of our knowledge, this is the first systematic study of attacking rank aggregation with specific purposes. The extensive evaluations are conducted on several datasets from different high-stake domains. Our experiments demonstrate that the proposed attacks could achieve the attacker's goal in the sense that the leading candidate of the perturbed ranking list is the designated one by the adversary. 

\vspace{0.25cm}
\noindent\textbf{Notations.} In the remainder of this paper, we will use positive integers to indicate alternatives and voters. Let $\boldsymbol{V}$ be the set $[n]=\{1,\ \dots,\ n\}$ which  denotes a set of alternatives to be ranked. $\boldsymbol{U}=\{\boldsymbol{u}_1,\ \dots,\ \boldsymbol{u}_m\}$ denotes a set of voters. We will adopt the following notation from combinatorics:
\begin{equation*}
    \begin{bmatrix}\boldsymbol{V} \\
    l \end{bmatrix}:=\text{set of all}\ l\ \text{elements subset of}\ \boldsymbol{V}.
\end{equation*}
In particular, 
\begin{equation*}
    \begin{aligned}
        &  \begin{bmatrix}\boldsymbol{V} \\
            2 \end{bmatrix}&:=&\ \ \text{set of all unordered pairs of elements of}\ \boldsymbol{V}\\
        & &:=&\left\{[i,\ j]\ \Big\vert\ \forall\ i,\ j\in\boldsymbol{V},\ i\neq j\right\}.
    \end{aligned}
\end{equation*}
Moreover, for any $i,j\in\boldsymbol{V},\ i\neq j$, we write $i\succ j$ to mean that alternative $i$ is preferred over alternative $j$. Such a comparison could be converted into an ordered pair $(i,\ j)$. The set of ordered pair will be denoted as
\begin{equation*}
    \boldsymbol{V}\times\boldsymbol{V}:=\left\{(i,\ j)\ \Big\vert\ i\succ j,\ \forall\ i,\ j\in\boldsymbol{V},\ i\neq j\right\}. 
\end{equation*}
Ordered and unordered pairs will be delimited by parentheses $(i,j)$ and braces $\{i,j\}$ respectively. If we wish to emphasize the preference judgment from a particular voter $\boldsymbol{u}\in\boldsymbol{U}$, we will write $i\succ_{\boldsymbol{u}} j$. 

\section{Ranking with Pairwise Comparisons}
We begin with a formal description of the parametric model for binary comparisons, \textit{a.k.a} Bradley-Terry-Luce (\textbf{BTL}) model \cite{bradley1952rank}. Then we revisit the comparison graph and the Laplacian matrix which are essential for the ranking algorithms tailored to the \textbf{BTL} model. Two popular approaches which rank the items based on appropriate estimation of the latent preference scores, named \textbf{HodgeRank} \cite{Jiang2011} and \textbf{Rank Centrality} \cite{DBLP:journals/ior/NegahbanOS17}, are chosen as the victims to motivate our target attack strategies.  

\subsection{Parametric Model and Pairwise Comparisons}

Given a collection $\boldsymbol{V}$ of $n$ alternatives, the parametric model of pairwise comparisons assumes that each $i\in\boldsymbol{V}$ has a certain numeric quality score $\theta^*_i$. Suppose that $\boldsymbol{\theta}^*\in\mathbb{R}^n$
\begin{equation}
    \boldsymbol{\theta}^* = \left[\theta^*_1,\ \dots,\ \theta^*_n\right]^\top
\end{equation}
comprises the underlying preference scores assigned to each of the $n$ items. Without loss of generality, $\boldsymbol{\theta}^*$ could be positive as 
\begin{equation*}
    \theta^*_i > 0,\ \forall\ i\in[n].
\end{equation*}
Specifically, a comparison of any pair $\{i,j\}\in\begin{bmatrix}\boldsymbol{V}\\2 \end{bmatrix}$ is generated via the comparing between the corresponding scores $\theta^*_i$ and $\theta^*_j$ (in the presence of noise) by the \textbf{BTL} model. Let $y^*_{ij}$ denote the outcome of the comparison of the pair $i$ and $j$ based on $\boldsymbol{\theta}^*$, such that $y^*_{ij}=1$ if $i$ is preferred over $j$ and $y^*_{ij}=-1$ otherwise. Then, according to the \textbf{BTL} model,
\begin{equation}
    \label{eq:BTL}
    y^*_{ij} = \left\{
    \begin{array}{rl}
        1,  & \text{with probability}\ \theta^*_i/(\theta^*_i+\theta^*_j),\\[5pt]
        -1, & \text{otherwise}.
    \end{array}
    \right.
\end{equation}
Since the \textbf{BTL} model is invariant under the scaling of the scores, the latent preference score is not unique. Indeed, under the \textbf{BTL} model, a score vector $\boldsymbol{\theta}^*\in\mathbb{R}^n_+$ is the equivalence class
\begin{equation*}
    \boldsymbol{\Theta}^*=\left\{\boldsymbol{\theta}\ \Big\vert\ \text{there exists}\ \alpha>0\ \text{such that}\ \boldsymbol{\theta}=\alpha\boldsymbol{\theta}^*\right\}.
\end{equation*}
The outcome of a comparison depends on the equivalence class $\boldsymbol{\Theta}^*$.

\subsection{Comparison Graph and Combinatorial Laplacian}
A graph structure, named comparison graph, arises naturally from pairwise comparisons as follows. Let $\boldsymbol{G}=(\boldsymbol{V},\boldsymbol{E})$ stand for a comparison graph, where the vertex set $\boldsymbol{V}=[n]$ represents the $n$ candidates. In our problem setting, we pay attention to the complete graph setting: the directed edge set $\boldsymbol{E}=\boldsymbol{V}\times\boldsymbol{V}$ and $N:=|\boldsymbol{E}| = n(n-1)$. One can further associate weights $\boldsymbol{w}^*$ on $\boldsymbol{E}$ as voters $\boldsymbol{U}$ would have rated, \textit{i.e.} assigned cardinal scores or given an ordinal ordering to, the complete set of the alternatives $\boldsymbol{V}$. But no matter how incomplete the rated portion is, one may always convert such judgments into pairwise rankings that have no missing values as follows. For each voter $\boldsymbol{u}\in\boldsymbol{U}$, the pairwise ranking matrix is a skew-symmetric matrix $\boldsymbol{Y}^{\boldsymbol{u}}=\{y^{\boldsymbol{u}}_{ij}\}\in\{-1,0,1\}^{n\times n}$ as
\begin{equation}
    y^{\boldsymbol{u}}_{ij} = -y^{\boldsymbol{u}}_{ji},\ \forall\ (i,\ j)\in\boldsymbol{E},\ \forall\ \boldsymbol{u}\in\boldsymbol{U},
\end{equation}
where
\begin{equation}
    \label{eq:direction_ind}
    y^{\boldsymbol{u}}_{ij} = \left\{
    \begin{array}{rl}
        1, & \text{if}\ i\succ_{\boldsymbol{u}} j,\\[5pt]
        -1,& \text{if}\ j\succ_{\boldsymbol{u}} i,\\[5pt]
        0, & \text{otherwise}.
    \end{array}
    \right.
\end{equation}
Furthermore, we associate weight with each directed edge as $\boldsymbol{w}^*=[w^*_{12},w^*_{13},\dots,w^*_{n,n-1}]^\top\in\mathbb{Z}_+$
\begin{equation}
    \label{eq:weight}
    w^*_{ij}:= \underset{\boldsymbol{u}\in\boldsymbol{U}}{\sum}\ \mathbb{I}[y^{\boldsymbol{u}}_{ij}> 0] + \mathbb{I}[y^{\boldsymbol{u}}_{ji}<0], 
\end{equation}
where $\mathbb{I}[\cdot]$ is the Iverson bracket. Consequently, we can represent any pairwise ranking data as a comparison graph $\boldsymbol{G}$ with edge weights $\boldsymbol{w}^*$.

Given a graph $\boldsymbol{G}$ and weights $\boldsymbol{w}^*$, it is common to consider the weight matrix $\boldsymbol{W}^*$ with $w^*_{ij}$ as matrix elements, as well as the diagonal degree matrix $\boldsymbol{D}^*=\textbf{diag}(d^*_1,\dots,d^*_n)$ given by $d^*_{i} = \sum_{j\in\boldsymbol{V}}w^*_{ij}$, which represents the volume taken by each node in the graph $\boldsymbol{G}$. The combinatorial Laplacian $\mathcal{L}_0$ is defined as
\begin{equation}
    \mathcal{L}_0 = \boldsymbol{D}^* - \boldsymbol{W}^*. 
\end{equation}
In both solving process and the theoretical analysis, the combinatorial Laplacian $\mathcal{L}_0$ plays a vital role for the popular approaches based on the parametric model.

\subsection{HodgeRank}
\label{sec:hodgerank}
The \textbf{Hodgerank} method discussed in \cite{Jiang2011} consists in finding the relative ranking score by solving the following least-squares problem:
\begin{equation}
    \label{opt:Hodgerank}
    \begin{aligned}
        \underset{\boldsymbol{\theta}\in\mathbb{R}^n}{\textbf{\textit{minimize}}}\ \ \frac{\ 1\ }{\ 2\ }\underset{(i,j)\in\boldsymbol{E}}{\sum}\ w^*_{ij}\big(y_{ij}-\theta_j+\theta_i\big)^2
    \end{aligned}
\end{equation}
where $\boldsymbol{y}=[y_{12},y_{13},\dots,y_{n,n-1}]^\top$ represents the directions of edges. As $\boldsymbol{G}$ is a complete graph, we set $y_{ij} = 1$ which indicates a direct edge from node $i$ to $j$. Based on combinational Hodge theory \cite{Jiang2011}, the minimal norm solution of \eqref{opt:Hodgerank} is simply given as
\begin{equation}
    \label{eq:HodgeRank}
    \boldsymbol{\bar{\theta}} = -\mathcal{L}^{\dagger}_0\cdot\textbf{div}(\boldsymbol{y}),
\end{equation}
where $\mathcal{L}^{\dagger}_0$ is the Moore-Penrose pseudo-inverse of $\mathcal{L}_0$, and the divergence operator $\textbf{div}$ is defined as 
\begin{equation}
    [\textbf{div}(\boldsymbol{y})](i) = \underset{j:(i,j)\in\boldsymbol{E}}{\sum}\ w^*_{ij}y_{ij},\ \forall\ i\in[n].
\end{equation}

\subsection{Rank Centrality}
\label{sec:spectral}
The spectral ranking algorithm, or \textbf{RankCentrality} \cite{DBLP:journals/ior/NegahbanOS17}, is motivated by the connection between the pairwise comparisons $\boldsymbol{w}^*$ and a random walk over a directed graph $\boldsymbol{G}$. Spectral method constructs a random walk on $\boldsymbol{G}$ where at each time, the random walk is likely to go from vertex $i$ to vertex $j$ if items $i$ and $j$ were ever compared; and if so, the likelihood of going from $i$ to $j$ depends on how often $i$ lost to $j$. That is, the random walk is more likely to move to a neighbor who is more probable to “wins”. How frequently this walk visits a particular node in the long run, or equivalently the stationary distribution, is the score of the corresponding item. Thus, effectively this algorithm captures the preference of the given item versus all the others, not just immediate neighbors: the global effect induced by transitivity of comparisons is captured through the stationary distribution. 

A random walk can be represented by a time-independent transition matrix $\boldsymbol{P} = \{P_{i,j}\}_{1\leq i,j\leq n}\in\mathbb{R}^{n\times n}_+$, where $P_{i,j} = \mathbb{P}(X_{t+1}=j|X_{t}=i)$ and $X_t$ represents the state of the process (arriving node) at time $t$. By definition, the entries of a transition matrix are nonnegative and satisfy
\begin{equation}
    P_{i,j}+P_{j,i}=1,\ \ \forall\ i,\ j\in[n],\ i\neq j.
\end{equation}
One way to define a valid transition matrix $\boldsymbol{P}$ of a random walk on $\boldsymbol{G}$ is to scale all the edge weights by the maximum out-degree of a node, noted as $d_{\text{max}}$. This re-scaling ensures that each row-sum is at most one. Finally, to ensure that each row-sum is exactly one, the spectral method adds a self-loop to each node of $\boldsymbol{V}$. Concretely, the transition matrix $\boldsymbol{P}^*$ is converted from the pairwise comparison data $\boldsymbol{w}^*$ in such a way that
\begin{equation}
    \label{eq:empirical_matrix}
    P^*_{i,j} =
    \begin{dcases}
        \ \ \ \ \ \ \ \ \ \ \ \ \ \ \ \ \ \ \ \ \ \ \ \ \ \ \ \ \ \ \ \ \ \ \ 0, &\text{if}\ i\neq j,\ w^*_{ij}+w^*_{ji}=0,\\[3.5pt]
        \ \ \ \ \ \ \ \ \ \ \ \ \frac{1}{d_{\text{max}}}\ \frac{w^*_{ij}}{w^*_{ji}+w^*_{ij}}, &\text{if}\ i\neq j,\ w^*_{ij}+w^*_{ji}\neq 0,\\[2.5pt]
        1 - \frac{1}{d_{\text{max}}}\underset{k\neq i}{\sum}\ \frac{w^*_{ik}}{w^*_{ik}+w^*_{ki}}, &\text{otherwise}.
    \end{dcases}
\end{equation}
Rank centrality estimates the probability distribution obtained by applying matrix $\boldsymbol{P}^*$ repeated starting from any initial condition. Precisely, let $\theta_t(i)=\mathbb{P}(X_t=i)$ denote the distribution of the random walk at time $t$ with $\boldsymbol{\theta}_0=\{\theta_0(i)\}\in\mathbb{R}^n_+$ as an arbitrary starting distribution on $[n]$. Then the random walk holds
\begin{equation}
    \label{eq:transition_matrix_condition}
    \boldsymbol{\theta}^{\top}_{t+1} = \boldsymbol{\theta}^{\top}_t\boldsymbol{P}^*.
\end{equation}
One expects the stationary distribution of the sample version $\boldsymbol{P}^*$ to form a good estimate of true relative ranking score\footnote{\tiny The original paper assumes that the true relative scores are generated from the logistic pairwise comparison model, \textit{e.g.} Bradley-Terry-Luce (BTL) model, multinomial logit (MNL) and Plackett-Luce (PL) model.}, provided the sample size is sufficiently large. When the transition matrix has a unique left eigenvector $\boldsymbol{\theta}^*$ related to the largest eigenvalue, then starting from any initial distribution $\boldsymbol{\theta}_0$, the limiting distribution $\boldsymbol{\theta}_{t+1}$ is unique. This stationary distribution $\underset{t\rightarrow\infty}{\lim}\boldsymbol{\theta}_{t}$ is the top left eigenvector of $\boldsymbol{P}^*$ as 
\begin{equation}
    \label{eq:detailed_balance}
    \underset{t\rightarrow\infty}{\lim}\boldsymbol{\theta}_{t} = \boldsymbol{\bar{\theta}}\ \ \text{and}\ \ \boldsymbol{\bar{\theta}}^\top=\boldsymbol{\bar{\theta}}^\top\boldsymbol{P}^*,
\end{equation}
which only involves a simple eigenvector computation.

\section{General Framework}
In this section, we systematically introduce the general game-theoretic framework for target attacks on pairwise ranking. Specifically, we firstly propose the threat model of the purposeful adversary, then analyze the adversarial game from the view of the fixed point theorem, and finally present the definition of the continuous operator game and show the existence of equilibrium.

\subsection{Threat Model of Purposeful Adversary}
\label{sec:threat_model}
We provide here a detailed framework for target attacks against pairwise ranking algorithms. The framework consists of goal of the adversary, knowledge of the attacked method, and capability of manipulating the pairwise data, to eventually define the target attack strategies.

\vspace{0.3cm}
\noindent\textbf{The Goal of Adversary. }To execute the target attack, an adversary desires to provide the ranking algorithm with the tempered data. This action by the adversary will induce its opponent to produce a designated ranking result. If the ranking result meets the demand of adversary, we will say that the adversary has executed a successful attack. Let $\mathcal{A}$ and $\mathcal{R}$ be the adversary and the ranking algorithm respectively. Suppose that $\boldsymbol{\hat{w}}$ is the substitute of the original pairwise comparison data $\boldsymbol{w}^*$. The relative ranking scores $\boldsymbol{\bar{\theta}}$ and $\boldsymbol{\hat{\theta}}$ are produced by $\mathcal{R}$ with the original data $\boldsymbol{w}^*$ and the poisoned data $\boldsymbol{\hat{w}}$ accordingly. We adopt $\boldsymbol{\pi}_{\boldsymbol{\theta}}$ as the ranking list decided by the relative ranking score $\boldsymbol{\theta}$. $\mathcal{A}$ may have multiple intentions to manipulate the ranking result $\boldsymbol{\pi}_{\boldsymbol{\hat{\theta}}}$. \textbf{However, designating a given candidate to be the winner, \textit{a.k.a} the top-$1$ position of $\boldsymbol{\pi}_{\boldsymbol{\hat{\theta}}}$, must be the most attractive achievement.} Consequently, we are concerned about the following situation: $\mathcal{A}$ hopes that $\mathcal{R}$ would arrange $i_0\in[n]$ as $\boldsymbol{\pi}_{\boldsymbol{\hat{\theta}}}(1)$. The goal of $\mathcal{A}$ could be described as a set of ranking lists
\begin{equation}
    \label{eq:constraint_set_1}
    \mathcal{C}_{\mathcal{A}} := \left\{\boldsymbol{\pi}\ \Big\vert\ \boldsymbol{\pi}(1) = i_0\right\}.
\end{equation}
If it holds that $\boldsymbol{\pi}_{\boldsymbol{\hat{\theta}}}\in\mathcal{C}_{\mathcal{A}}$, we will say that $\mathcal{A}$ has a successful target attack strategy by substituting $\boldsymbol{w}^*$ with $\boldsymbol{\hat{w}}$. However, \eqref{eq:constraint_set_1} always leads to the integer optimization problems which is NP-complete \cite{10.1007/978-3-642-95322-4_17}. To avoid solving such a challenging problem, we turn our attention to the latent preference score. As the latent preference score $\boldsymbol{\hat{\theta}}$ decides $\boldsymbol{\pi}(\boldsymbol{\hat{\theta}})$, we can convert $\mathcal{C}_{\mathcal{A}}$ into the constraints on the relative ranking scores as 
\begin{equation}
  \label{eq:constraint_set_2}
  \mathcal{C}_{\mathcal{A}} := \left\{\boldsymbol{\theta} \in \mathbb{R}^n\ \Big\vert\ \underset{i\in[n],\ i\neq i_0}{\textbf{max}}\theta_i \leq \theta_{i_0}\right\}. 
\end{equation}
This set is more convenient to incorporate with the optimization process of $\mathcal{R}$ and establish the targeted attack strategies. 


\vspace{0.3cm}
\noindent\textbf{The Knowledge of Adversary. }We assume four distinct target attack scenarios which are distinguished by the knowledge of $\mathcal{A}$. The knowledge of $\mathcal{A}$ is partitioned into two dimensions: the completeness of information predominated by $\mathcal{A}$ and the accuracy of the feedback given by $\mathcal{R}$. As the adversarial mechanism among $\mathcal{A}$ and $\mathcal{R}$ can be naturally modeled as a game, we introduce the following definitions from game theory.

\begin{definition}[\textbf{Complete/Incomplete Game}]
    \label{def:com_game}
    In the adversary game between ranking algorithm $\mathcal{R}$ and the attacker $\mathcal{A}$, there exists the following decomposition of $\boldsymbol{w}^*$
    \begin{equation}
        \label{eq:decomposition}
        \boldsymbol{w}^* = \boldsymbol{w}^*_k + \boldsymbol{w}^*_u,
    \end{equation}
    where $\boldsymbol{w}^*_k\in\mathbb{Z}^N_+$ represents the available part for $\mathcal{A}$, which can be modified and substituted; $\boldsymbol{w}^*_u\in\mathbb{Z}^N_+$ is the inaccessible part and it could create a disturbance for $\mathcal{A}$. 
    \begin{itemize}
      \item If $\boldsymbol{w}^*_u=\boldsymbol{0}$, it is a \textbf{Complete Information Game} and $\mathcal{A}$ has the full accessibility to manipulate the input data for $\mathcal{R}$.
      \vspace{0.25cm}
      \item If $\boldsymbol{w}^*_u\neq\boldsymbol{0}$, it is an \textbf{Incomplete Information Game} and there exists unavailable data for $\mathcal{R}$.
    \end{itemize}
\end{definition}
It is known that uncertainty is a common phenomenon in adversarial operation. The uncertainty comes from the limitation of the attack ability, the defense mechanism of the victim, the randomness of data generation process and so on. 

\begin{definition}[\textbf{Perfect/Imperfect Game}]
    \label{def:pef_game}
    Let $\boldsymbol{\bar{\theta}}$ be the relative ranking score obtained by the victim algorithm $\mathcal{R}$ with $\boldsymbol{w}^*$ and $\boldsymbol{\pi}_{\boldsymbol{\bar{\theta}}}$ be the corresponding ranking list. 
    \begin{itemize}
        \item If the attacker can get the relative ranking score $\boldsymbol{\bar{\theta}}$ obtained by $\mathcal{R}$, it is a \textbf{Perfect Information Game}.
        \vspace{0.25cm}
        \item If the attacker only knows the ranking results $\boldsymbol{\pi}_{\boldsymbol{\bar{\theta}}}$, it is an \textbf{Imperfect Information Game}.
    \end{itemize}
\end{definition}
Even though the perfect feedback $\boldsymbol{\bar{\theta}}$ is inaccessible in the majority of cases, $\mathcal{A}$ still is possible to simulate $\boldsymbol{\bar{\theta}}$ with $\boldsymbol{w}^*$ or $\boldsymbol{w}^*_u$. Such a simulation could be completed with the public algorithm of $\mathcal{R}$. Due to the existence of equivalence class, the numerical accuracy of $\boldsymbol{\bar{\theta}}$ would not strongly influence the behavior of $\mathcal{A}$. As a consequence, $\mathcal{A}$ could construct an imprecise copy of $\boldsymbol{\bar{\theta}}$ by $\boldsymbol{\pi}_{\boldsymbol{\bar{\theta}}}$ under the imperfect information situation. Of course, we can further limit the attacker’s knowledge to design the corresponding attack strategies, which would be more efficient in actual situations but more challenging. The target attack strategies with the least knowledge could be the future direction of our research.

\vspace{0.3cm}
\noindent\textbf{The Capability of Adversary. }To modify the accessible part of pairwise comparisons $\boldsymbol{w}^*_u$, $\mathcal{A}$ could inject an arbitrary pairwise comparison into $\boldsymbol{w}^*_u$, delete the existing comparison $(i,j)$ in $\boldsymbol{w}^*_u$ or just flip $(i,j)$ as $(j,i)$. We authorize the following three operations to $\mathcal{A}$.

\renewcommand\labelitemii{$\circ$}
\begin{itemize}
    \item The adding operation can be described as an assignment in the following two cases:
    \begin{itemize}
        \vspace{0.1cm}
        \item $\exists\ \boldsymbol{u}\in\boldsymbol{U},\ \exists\ i,\ j\in\boldsymbol{V},\ i\neq j,\ \text{such that}\ y^{\boldsymbol{u}}_{ij}=0$,
        \vspace{0.25cm}
        \item $\forall\ \boldsymbol{u}\notin\boldsymbol{U},\ \forall\ i,\ j\in\boldsymbol{V},\ i\neq j,\ y^{\boldsymbol{u}}_{ij}$ does not exist,
        \vspace{0.1cm}
    \end{itemize}
    then 
    \[
        y^{\boldsymbol{u}}_{ij}\gets 1\ \text{or}\ -1\ \text{arbitrarily}.
    \]
    \item The deleting operation only involves the non-zero $y^{\boldsymbol{u}}_{ij}$:
    \begin{itemize}
        \vspace{0.1cm}
        \item $\exists\ \boldsymbol{u}\in\boldsymbol{U},\ \exists\ i,\ j\in\boldsymbol{V},\ i\neq j,\ \text{such that}\ y^{\boldsymbol{u}}_{ij}\neq 0$,
        \vspace{0.1cm}      
    \end{itemize}
    then 
    \[
        y^{\boldsymbol{u}}_{ij}\gets 0.
    \]
    \item The flipping operation turns the non-zero $y^{\boldsymbol{u}}_{ij}$ into its opponent
    \begin{itemize}
        \vspace{0.1cm}
        \item $\exists\ \boldsymbol{u}\in\boldsymbol{U},\ \exists\ i,\ j\in\boldsymbol{V},\ i\neq j,\ \text{such that}\ y^{\boldsymbol{u}}_{ij}\neq 0$,
        \vspace{0.1cm}
    \end{itemize}
    then 
    \[
        y^{\boldsymbol{u}}_{ij}\gets-\big|y^{\boldsymbol{u}}_{ij}\big|.
    \]
\end{itemize}
With the complete comparison graph $\boldsymbol{G}$, all attack operations (adding, deleting and flipping) can be executed by increasing or decreasing the corresponding element in $\boldsymbol{w}^*_u$. Authorizing these operations to $\mathcal{A}$ without quantitative restrictions is a kind of favoritism. However, to fully expose the vulnerability of $\mathcal{R}$, we assume that the attacker does not need to disguise himself/herself or pay for his/her malicious behaviors. The target attack strategy with limited budget is an open question for future research.

\subsection{A Fixed Point View of the Adversarial Game}

Now we are ready to propose the concrete target attack strategies for rank aggregation approaches. First of all, we describe a fixed-point view of the purposeful attack process, where the dynamical system is the union of all behaviors, actions and purpose, the inputs are adversarial actions, and the control costs are defined by the adversary’s goals to do harm. As discussed before, we treat the adversary $\mathcal{A}$ and the victim algorithm $\mathcal{R}$ as two agents/players in a game. According to the above threat model in Sec. \ref{sec:threat_model}, $\mathcal{A}$ attempts to obtain information about what $\mathcal{R}$ will choose and take into account this information in making their own decision with some special interests. Specifically, in such an adversarial game, $\mathcal{A}$ utilizes the complete/incomplete pairwise comparisons and the perfect/imperfect feedback to counterfeit some new paired data $\boldsymbol{\hat{w}}$ and replace the observed data $\boldsymbol{w}^*$. Then the victim $\mathcal{R}$ executes the ranking algorithm to generate the misguided ranking $\boldsymbol{\pi}_{\mathcal{R}}$ which satisfies $\boldsymbol{\pi}_{\mathcal{R}}\in\mathcal{C}_{\mathcal{A}}$. 

Let $\mathcal{X}$ be a non-empty set and we define 
\begin{equation}
    \mathcal{T}_{\mathcal{A}}: \mathcal{C}_{\mathcal{A}} \rightarrow \mathbb{Z}^{N}_+
\end{equation}
to be the set-valued operator of the adversary $\mathcal{A}$ which maps every desired ranking list $\boldsymbol{\pi}$ in $\mathcal{C}_{\mathcal{A}}$ to a set of possible paired data $\mathcal{T}_{\mathcal{A}}\boldsymbol{\pi}\subset2^{\mathbb{Z}^{N}_+}$. Moreover, the ranking aggregation process belonging to the victim $\mathcal{R}$ is also a set-valued operator
\begin{equation}
    \mathcal{T}_{\mathcal{R}}: \mathbb{Z}^{N}_+ \rightarrow \mathbb{R}^{n}
\end{equation}
which maps the pairwise comparisons in $\mathbb{Z}^{N}_+$ to a set of ranking lists. Furthermore, the composition of the above two operators is 
\begin{equation}
    \mathcal{T}_{\mathcal{R}}\circ\mathcal{T}_{\mathcal{A}}:\mathcal{C}_{\mathcal{A}} \rightarrow \mathbb{R}^{n}
\end{equation}
which is a mapping function as
\begin{equation}
    \boldsymbol{\pi}\ \ \mapsto\ \ \mathcal{T}_{\mathcal{R}}(\mathcal{T}_{\mathcal{A}}\boldsymbol{\pi})\ \ \ \coloneqq\underset{\boldsymbol{w}\in\mathcal{T}_{\mathcal{A}}\boldsymbol{\pi}}{\bigcup} \mathcal{T}_{\mathcal{R}}\boldsymbol{w}.
\end{equation}
Then the complete interaction between $\mathcal{A}$ and $\mathcal{R}$ can be described by the composition of $\mathcal{T}_{\mathcal{R}}$ and $\mathcal{T}_{\mathcal{A}}$. It is worth noting that the purpose of the adversary could be represented as $\boldsymbol{\pi}_{\mathcal{R}}\in\mathcal{C}_{\mathcal{A}}$ and the target ranking $\boldsymbol{\pi}_{\mathcal{A}}$ also satisfies $\boldsymbol{\pi}_{\mathcal{A}}\in\mathcal{C}_{\mathcal{A}}$. By the definition of $\mathcal{C}_{\mathcal{A}}$ \eqref{eq:constraint_set_1} and \eqref{eq:constraint_set_2}, it is obvious that the leading candidate of $\boldsymbol{\pi}_{\mathcal{A}}$ is invariant under $\mathcal{T}_{\mathcal{R}}\circ\mathcal{T}_{\mathcal{A}}$. Accordingly, any element of $\mathcal{C}_{\mathcal{A}}$ is, to some degree, a ``fixed point'' of the adversarial game between $\mathcal{A}$ and $\mathcal{R}$. 

It is known \cite{Nash48,10.2307/1969529} that a Nash equilibrium of every finite zero-sum game is a fixed point of the game's best response correspondence. A Nash equilibrium is a solution in which each player correctly predicts what the other players will do and responds optimally, so that no player can improve their utilities by choosing different strategies. But we could not apply the classic Kakutani fixed-point theorem \cite{kakutani1941generalization} in this adversarial game. On one hand, the adversarial game between $\mathcal{A}$ and $\mathcal{R}$ is a nonzero-sum game. The purpose of $\mathcal{A}$ is allocating the target candidate $i_0$ as the leading one in the aggregation result; $\mathcal{R}$ pursuits a best aggregated ranking list by the pairwise comparisons. On the other hand, \eqref{eq:constraint_set_1} is not a convex subset of $\mathbb{R}^n$. 

But the situation will change with the choice of $\mathcal{R}$. If $\mathcal{R}$ executes the parametric-model-based ranking aggregation approaches, \textit{e.g.} \textbf{HodgeRank} in Sec \ref{sec:hodgerank} and \textbf{RankCentrality} in Sec \ref{sec:spectral}, $\mathcal{A}$ would convert her/his manipulation purpose of ranking aggregation result into the constraints on the latent preference score like \eqref{eq:constraint_set_2}. Obviously, \eqref{eq:constraint_set_2} is a convex subset of $\mathbb{R}^n$. With the convex set, we can further discuss the fixed points of operators in the adversarial game. 

\begin{definition}[Fixed Points of Operator]
    The set of fixed points of an operator $\mathcal{T}:\mathcal{X}\rightarrow\mathcal{X}$ is denoted by $\textbf{Fix}[\mathcal{T}]$, i.e.
    \begin{equation}
        \textbf{Fix}[\mathcal{T}] = \big\{\boldsymbol{x}\in\mathcal{X}\ \big\vert\ \mathcal{T}\boldsymbol{x}=\boldsymbol{x}\big\}.
    \end{equation}
\end{definition}


The following proposition shows the existence of fixed point for the continuous operators.

\begin{proposition}
     Let $(\mathcal{X}, d)$ be a complete metric space with distance (or metric) $d$ and $\mathcal{T}:\mathcal{X}\rightarrow\mathcal{X}$ be a Lipschitz continuous operator with the Lipschitz constant $L\in[0,1)$. Given $\boldsymbol{x}_0\in\mathcal{X}$, define the following sequence $\{\boldsymbol{x}_n\}$
    \begin{equation}
        \boldsymbol{x}_{n+1} \coloneqq \mathcal{T}\boldsymbol{x}_n,\ \forall\ n\in\mathbb{N}.
    \end{equation}
    Let $d_n(\boldsymbol{x})\coloneqq d(\boldsymbol{x}_{n},\boldsymbol{x})$. The sequence $\{d_n(\boldsymbol{x})\}$ is monotonically decreasing:
        \begin{equation}
            d(\boldsymbol{x}_{n+1},\ \boldsymbol{x})\leq Ld(\boldsymbol{x}_{n},\ \boldsymbol{x}),\ \forall\ n\in\mathbb{N}.
        \end{equation}
    Moreover, $\boldsymbol{x}$ is the unique fixed point of $\mathcal{T}$.
\end{proposition}

For $\boldsymbol{x},\boldsymbol{y}\in\mathcal{X}$, we denote the preference relation between $\boldsymbol{x}$ and $\boldsymbol{y}$ by $\boldsymbol{x}\succeq\boldsymbol{y}$. If $\boldsymbol{x}\succeq\boldsymbol{y}$ but $\boldsymbol{y}\nsucceq\boldsymbol{x}$, we denote this by $\boldsymbol{x}\succ\boldsymbol{y}$. If $\boldsymbol{x}\succeq\boldsymbol{y}$ and $\boldsymbol{y}\succeq\boldsymbol{x}$, we denote this by $\boldsymbol{x}\sim\boldsymbol{y}$. Since a preference relation is complete, we are able to establish a preference ordering between any two elements of $\mathcal{X}$, and transitivity prevents cycles in preferences. Notice that there exist infinitely possible choices of $\boldsymbol{\theta}_{\mathcal{A}}$ and $\boldsymbol{\theta}_{\mathcal{R}}$, we consider the adversarial games in which both $\mathcal{A}$ and $\mathcal{R}$ may have infinitely many pure strategies.

\begin{restatable}{definition}{operatorgame}
    A continuous operator game is a tuple $\big\langle\boldsymbol{P}, \boldsymbol{S}, \{\mathcal{T}_p\}\big\rangle$, where
    \begin{itemize}
        \item $\boldsymbol{P}$ is a finite (non-empty) set of players in the game,
        \vspace{0.15cm}
        \item $\boldsymbol{S}$ is the set of strategy profiles as
        \begin{equation}
            \boldsymbol{S} = \underset{p\in\boldsymbol{P}}{\bigcup}\boldsymbol{S}_p,
        \end{equation}
        where $\boldsymbol{S}_p$ denotes non-empty compact metric space of each player $p\in\boldsymbol{P}$,
        \vspace{0.15cm}
        \item 
        \begin{equation}
            \mathcal{T}_p : \boldsymbol{S}_p\rightarrow\mathbb{R}^n,\ \forall\ p\in\boldsymbol{P}
        \end{equation}
        is the continuous operator acting as the payoff function in traditional game.
    \end{itemize}
\end{restatable}
A compact metric space is a general mathematical structure for representing infinite sets that can be well approximated by large finite sets. One important fact is that, in a compact metric space, any infinite sequence has a convergent sub-sequence. We next state the analogue of Nash’s theorem for continuous operator games. For continuous utility function, this result is known as Glicksberg’s theorem \cite{10.2307/2032478}.
\begin{restatable}{theorem}{nashequilibrium}
    \label{thm:nash_equilibrium}
     Every continuous operator game has a mixed strategy Nash equilibrium. 
\end{restatable}
According to this result, we know that the desired latent preference score $\boldsymbol{\theta}_{\mathcal{A}}\in\mathcal{C}_{\mathcal{A}}$ could be a Nash equilibrium of the adversarial game. It means that the purposeful attack behavior of $\mathcal{A}$ has the opportunity of success. As discussed before, we pay attention to two famous rank aggregation methods, \textbf{HodgeRank} and \textbf{RankCentrality}, as the victim $\mathcal{R}$. Next, we will discuss how to construct the payoff function of $\mathcal{A}$ and form the corresponding adversarial game with the specific victim $\mathcal{R}$.

\section{Adversarial Game with HodgeRank}
Now we are ready to propose the payoff operator and the concrete attack strategies of $\mathcal{A}$ to execute the target attack against \textbf{HodgeRank}. Remember that \textbf{HodgeRank} aggregates the pairwise comparisons $\{\boldsymbol{G}, \boldsymbol{y}, \boldsymbol{w}^*\}$ by solving a weighted least-squared problem as \eqref{opt:Hodgerank}. When the victim $\mathcal{R}$ is designated as \textbf{HodgeRank}, the corresponding adversarial game would satisfy the following claims:
\begin{itemize}
    \item The set of players contains two members:
    \begin{equation}
        \boldsymbol{P} := \{\mathcal{A},\ \mathcal{R}\},
    \end{equation}
    where $\mathcal{R}$ represents the \textbf{HodgeRank}.
    \vspace{0.15cm}
    \item The strategy profile $\boldsymbol{S}_{\mathcal{R}}$ is 
    \begin{equation}
        \boldsymbol{S}_{\mathcal{R}} \coloneqq \mathbb{R}^n, 
    \end{equation}
    which is a non-empty compact metric space.
    \vspace{0.15cm}
    \item The continuous operator of $\mathcal{R}$ bases on the weighted least-squared problem \eqref{opt:Hodgerank}.
\end{itemize}
By introducing the comparison matrix of $\boldsymbol{C}$, we can change \textbf{HodgeRank} problem \eqref{opt:Hodgerank} into a matrix formulation:
\begin{subequations}
    \label{opt:Hodgerank_2}
    \begin{align}
        \ell_{\mathcal{R}}(\boldsymbol{\theta})\coloneqq\ \ &\frac{\ 1\ }{\ 2\ }\underset{(i,j)\in\boldsymbol{E}}{\sum}\ w^*_{ij}\big(y_{ij}-\theta_j+\theta_i\big)^2,\\
        =\ \ &\frac{\ 1\ }{\ 2\ }\ \big\|\boldsymbol{y}-\boldsymbol{C}\boldsymbol{\theta}\big\|^2_{2,\ \boldsymbol{w}^*},\\[5pt]
        =\ \ &\frac{\ 1\ }{\ 2\ }\ \big\|\boldsymbol{W}^*_{\frac{1}{2}}\boldsymbol{y}-\boldsymbol{W}^*_{\frac{1}{2}}\boldsymbol{C}\boldsymbol{\theta}\big\|^2_{2},
    \end{align}
\end{subequations}
where the comparison matrix $\boldsymbol{C}=\{c_{ml}\}_{m\in[N], l\in[n]}\in\{-1,0,1\}^{N\times n}$ indicates the edge flow of comparison graph $\boldsymbol{G}$, $\boldsymbol{W}^*_{\frac{1}{2}} = \textbf{\textit{diag}}(\sqrt{\boldsymbol{w}^*})$ and
\begin{equation}
     \sqrt{\boldsymbol{w}^*} = \left[\sqrt{\ w^*_{12{\phantom{\frac{1}{2}}}}},\sqrt{\ w^*_{13{\phantom{\frac{1}{2}}}}},\dots,\sqrt{\ w^*_{n,n-1}\ }\right]^\top.
\end{equation}
We index each directed edge $(i,j)$ by $m$ (noted as $(i,j)\sim m$) in the following way:
\begin{equation*}
    m = \left\{
    \begin{array}{ll}
        (i-1)*(n-1) + j-1, &\ \text{if}\ 1\leq i<j\leq n,\\[7.5pt]
        (i-1)*(n-1) + j,   &\ \text{if}\ 1\leq j<i\leq n.
    \end{array}
    \right.
\end{equation*}
Then the elements of $\boldsymbol{C}$ are organized as
\begin{equation}
    \label{eq:comparison_matrix}
    c_{ml}=\left\{
    \begin{array}{rl}
         1, &\text{if}\ (i,\ j)\sim m,\ l = i,\\[5pt]
        -1, &\text{if}\ (i,\ j)\sim m,\ l = j,\\[5pt]
         0, &\text{otherwise.} 
    \end{array}
    \right.
\end{equation}
With simple calculation, we will know that
\begin{equation}
    \boldsymbol{1} = [1,1,\dots,1]^\top\in\mathbb{R}^n
\end{equation}
is an element in the null spaces of $\boldsymbol{C}$ and $\boldsymbol{W}^*_{\frac{1}{2}}\boldsymbol{C}$
\begin{equation}
    \boldsymbol{C}\cdot\boldsymbol{1} = \boldsymbol{0},\ \boldsymbol{W}^*_{\frac{1}{2}}\boldsymbol{C}\cdot\boldsymbol{1} = \boldsymbol{0}.
\end{equation}
As a consequence, the classic Tikhonov regularization is introduced to generate a unique solution of \textbf{HodgeRank}
\begin{equation}
    \underset{\boldsymbol{\theta}\in\mathbb{R}^n}{\textbf{\textit{minimize}}} \ \ \ell_{\mathcal{R}}(\boldsymbol{\theta})+\big\|\boldsymbol{\Gamma}\boldsymbol{\theta}\big\|_2^2,
\end{equation}
where $\boldsymbol{\Gamma}$ is always chosen as $\lambda_0\boldsymbol{I}$. Then the corresponding optimization problem and the operator of $\mathcal{R}$ becomes
\begin{flalign}
    (\mathcal{T}_{\mathcal{R}})\hspace{5em}&\underset{\boldsymbol{\theta}\in\mathbb{R}^n}{\textbf{\textit{minimize}}}\ \ \ell_{\mathcal{R}}(\boldsymbol{\theta})+\lambda_0\big\|\boldsymbol{\theta}\big\|_2^2.& \label{opt:reg_hodge}
\end{flalign}
When the regularization parameter $\lambda_0>0$ satisfies
\begin{equation}
    \textbf{\textit{Null}}\big(\boldsymbol{C}^\top\boldsymbol{W}^*\boldsymbol{C}+\lambda_0\boldsymbol{I}\big) = \{\boldsymbol{0}\},
\end{equation}
the objective function of regularized \textbf{HodgeRank} \eqref{opt:reg_hodge} will be strongly convex. Actually, the ranker $\mathcal{R}$ always sets $\lambda_0$ to be a small constant as $\boldsymbol{C}^\top\boldsymbol{W}^*\boldsymbol{C}$ is a semi-definite positive matrix where $\boldsymbol{W}^*=(\boldsymbol{W}^*_{\frac{1}{2}})^\top\boldsymbol{W}^*_{\frac{1}{2}}$. Given the original pairwise comparisons $\{\boldsymbol{G}, \boldsymbol{y}, \boldsymbol{w}^*\}$, we note $\boldsymbol{\bar{\theta}}$ is the corresponding ranking score which is a solution of \eqref{opt:reg_hodge} and holds  
\begin{equation}
    \label{eq:optimal_condition_R}
    \boldsymbol{0} \in \partial\ \ell_{\mathcal{R}}(\boldsymbol{\bar{\theta}}) + 2\lambda_0\boldsymbol{\bar{\theta}},
\end{equation}
where $\partial\ \ell_{\mathcal{R}}(\boldsymbol{\bar{\theta}})$ is the gradient of $\ell_{\mathcal{R}}$ at $\boldsymbol{\bar{\theta}}$. By the strongly convexity, $\boldsymbol{\bar{\theta}}$ is the unique solution.

Now we try to construct the continuous operator of $\mathcal{A}$ and the corresponding strategy profile. It starts with complete information and perfect feedback, and then expands the analysis to the adversary with incomplete information or imperfect feedback.

\vspace{5pt}
\noindent\textbf{Complete and Perfect Game.} To archive the goal of manipulation, the adversary $\mathcal{A}$ replaces the pairwise comparisons $\boldsymbol{w}^*$ with $\boldsymbol{\hat{w}}$ and hopes that the regularized \textbf{HodgeRank} \eqref{opt:reg_hodge} will aggregate $\boldsymbol{\hat{w}}$ into a unique ranking score $\boldsymbol{\theta}_{\mathcal{A}}\in\mathcal{C}_{\mathcal{A}}$. With complete information and perfect feedback, there do not exist unknown pairwise comparisons $\boldsymbol{w}^*_u$ in \eqref{eq:decomposition} and the desired relative ranking score $\boldsymbol{\theta}_{\mathcal{A}}$ is a permutation of $\boldsymbol{\bar{\theta}}$ as
\begin{equation}
    \boldsymbol{\theta}_{\mathcal{A}} = p(\boldsymbol{\bar{\theta}}),
\end{equation}
where $p:\mathbb{R}^n_+\rightarrow\mathcal{C}_{\mathcal{A}}$ just redistributes the elements of $\boldsymbol{\bar{\theta}}$ without modifying their values. 
 
The optimal condition \eqref{eq:optimal_condition_R} is an enlightenment of $\mathcal{A}$. When the perturbed pairwise comparisons $\boldsymbol{\hat{w}}$ and desired relative ranking score $\boldsymbol{\theta}_{\mathcal{A}}$ correspond to the input and output of $\mathcal{R}$, it implies that 
\begin{equation}
    \boldsymbol{\theta}_{\mathcal{A}} \in \underset{\boldsymbol{\theta}\in\mathbb{R}^n}{\textbf{\textit{argmin}}}\ \ \frac{\ 1\ }{\ 2\ }\ \big\|\boldsymbol{W}_{\frac{1}{2}}\boldsymbol{y}-\boldsymbol{W}_{\frac{1}{2}}\boldsymbol{C}\boldsymbol{\theta}\big\|^2_{2}+\lambda_0\big\|\boldsymbol{\theta}\big\|^2_2,
\end{equation}
where $\boldsymbol{W}_{\frac{1}{2}} = \textbf{\textit{diag}}(\sqrt{\boldsymbol{\hat{w}}})$, $\boldsymbol{\hat{w}}=[\hat{w}_{12},\hat{w}_{13},\dots,\hat{w}_{n,n-1}]^\top$. Consequently, $\boldsymbol{\hat{w}}$ and $\boldsymbol{\theta}_{\mathcal{A}}$ also satisfy the following condition like \eqref{eq:optimal_condition_R}: 
\begin{equation}
    \label{eq:first_order_condition_1}
    \big(\boldsymbol{C}^\top\boldsymbol{W}\boldsymbol{C}+2\lambda_0\boldsymbol{I}\big)\boldsymbol{\theta}_{\mathcal{A}}=\boldsymbol{C}^\top\boldsymbol{W}\boldsymbol{y},
\end{equation}
where $\boldsymbol{W} = \boldsymbol{W}^\top_{\frac{1}{2}}\boldsymbol{W}_{\frac{1}{2}} = \textbf{\textit{diag}}(\boldsymbol{\hat{w}})$. As a necessary condition of $\boldsymbol{\theta}_{\mathcal{A}}$, \eqref{eq:first_order_condition_1} illustrates the relationship between the pairwise comparisons $\boldsymbol{\hat{w}}$ and the latent preference score considered by \textbf{HodgeRank}. On one hand, the ranker $\mathcal{R}$ deduces the solution $\boldsymbol{\theta}_{\mathcal{A}}$ from \eqref{eq:first_order_condition_1} with given data $\boldsymbol{\hat{w}}$. On the other hand, the \eqref{eq:first_order_condition_1} could instruct $\mathcal{A}$ to modify the pairwise comparisons. Once $\mathcal{A}$ has determined the target $\boldsymbol{\theta}_{\mathcal{A}}$, the corresponding attack strategy will be obtained by \eqref{eq:first_order_condition_1}. In addition, in order to cover up his/her own behavior, $\mathcal{A}$ wants to make as few changes as possible to the original data $\boldsymbol{w}^*$:
\begin{equation}
    \label{eq:mini_modf}
    \boldsymbol{\hat{w}}\in\underset{\boldsymbol{w}\in\mathbb{R}^N_+}{\textbf{\textit{arg\ min}}}\ \frac{1}{2}\big\|\boldsymbol{w}-\boldsymbol{w}^*\big\|^2_2.
\end{equation}
With some abuse of symbols, we treat $\boldsymbol{w}^*$ and $\boldsymbol{\hat{w}}$ as real number vectors in $\mathbb{R}^N_+$. Although $\boldsymbol{w}^*$ represents the numbers of pairwise comparisons for each type, we can turn the integer into the real number by dividing the total amount as
\begin{equation}
    \boldsymbol{w}^*\leftarrow\frac{\boldsymbol{w}^*}{\displaystyle \sum_{(i,j)\in\boldsymbol{E}}w^*_{ij}\ }.
\end{equation}
This operator will lead to a scale-different solution in \eqref{opt:reg_hodge}. However, these two ranking results (integer version and real number version) will be consistent as the ranking operator is scale-invariant:
\begin{equation}
     \boldsymbol{\pi}_{\boldsymbol{\theta}} = \boldsymbol{\pi}_{k\cdot\boldsymbol{\theta}}.
 \end{equation}
Furthermore, the real number $\boldsymbol{\hat{w}}$ can be converted to integer with the complete information. Avoiding the challenge integer programming and keeping the same ranking result, we adopt the real number version to construct the attack strategy. 

Combining \eqref{eq:first_order_condition_1} and \eqref{eq:mini_modf}, we construct the continuous operator of $\mathcal{A}$ as
\begin{flalign}
    (\mathcal{T}_{\mathcal{A}})&\hspace{-0.5em}
    \begin{split}
    \underset{\boldsymbol{w}\geq 0}{\textbf{\textit{minimize}}}&\ \ \frac{\ 1\ }{\ 2\ }\big\|\boldsymbol{w}-\boldsymbol{w}^*\big\|^2_2\\[5pt]
    \textbf{\textit{subject to}}&\ \ \boldsymbol{C}^\top\textbf{\textit{diag}}\big(\big(\boldsymbol{C}+2\lambda_0\boldsymbol{I}\big)\boldsymbol{\theta}_{\mathcal{A}}-\boldsymbol{y}\big)\boldsymbol{w}=\boldsymbol{0},  
    \end{split}
    & \label{opt:com_info}
\end{flalign}
where $\textbf{\textit{diag}}(\boldsymbol{x})$ is a diagonal matrix with $\boldsymbol{x} \in \mathbb{R}^N$ being the diagonal components and $\lambda_0>0$ is a small constant. We can adopt the Alternating Direction Method of Multipliers (\textbf{ADMM}) \cite{ADMM,10.1561/2200000016,doi:10.1137/110853996,doi:10.1007/s10589-016-9864-7} method to solve $\boldsymbol{\hat{w}}$ from \eqref{opt:com_info}. The detailed process is provided in the incomplete and imperfect cases. Here we give the following theorem and show that $\boldsymbol{\theta}_{\mathcal{A}}$ is a fixed point of the compositing operator $\mathcal{T}_{\mathcal{R}}\circ\mathcal{T}_{\mathcal{A}}$ which is also the Nash equilibrium of the adversarial operator game between \textbf{HodgeRank} $\mathcal{T}_{\mathcal{R}}$ \eqref{opt:reg_hodge} and $\mathcal{T}_{\mathcal{A}}$ \eqref{opt:com_info}.
\begin{restatable}{theorem}{fixpointhodge}
    \label{thm:hodge_nash_equilibrium}
    Given the desired latent preference score $\boldsymbol{\theta}_{\mathcal{A}}$, the corresponding attack strategy $\boldsymbol{\hat{w}}$ are the accumulation point of \eqref{opt:com_info}. If the solution set to the Karush–Kuhn–Tucker (\textbf{KKT}) conditions for \eqref{opt:com_info} is not empty and the penalty parameter $\rho$ of \textbf{ADMM} (Algorithm \ref{alg:attack_hodge}) satisfies $\rho\in(0,(1+\sqrt{5})/2)$, $\boldsymbol{\theta}_{\mathcal{A}}$ is a global optimal solution of \eqref{opt:reg_hodge}.
\end{restatable}
We provide the detailed proof process in the supplementary material. This theorem indicates that the adversary will be able to successfully control the ranking results of \textbf{HodgeRank} once he/she has complete information. It reveals an imminent threat to the integrity of the ranking results. What’s more serious is that even with partial information, the adversary $\mathcal{A}$ still has the chance to manipulate \textbf{HodgeRank}. Next, we will discuss the incomplete and imperfect game between \textbf{HodgeRank} and the adversary $\mathcal{A}$. 

\vspace{5pt}
\noindent\textbf{Incomplete and Imperfect Games.}
With the Definition \ref{def:com_game}, we know that there must exist some pairwise comparisons $\boldsymbol{w}^*_u$ that cannot be modified by attacker $\mathcal{A}$ when he/she plays the incomplete game with the victim $\mathcal{R}$. The obstruction $\boldsymbol{w}^*_u$ will lead $\mathcal{R}$ to the other ranking result which is unable to achieve the goal of $\mathcal{A}$. If $\mathcal{A}$ constructs the $\boldsymbol{\hat{w}}$ by \eqref{opt:com_info} which only replaces the known part of $\boldsymbol{w}^*_k$, the modified data given to \textbf{HodgeRank} will be
\begin{equation*}
    \boldsymbol{w}' = \boldsymbol{\hat{w}}+\boldsymbol{w}^*_u.
\end{equation*}
By Theorem \ref{thm:hodge_nash_equilibrium}, if the adversary $\mathcal{A}$ hopes $\boldsymbol{\theta}_{\mathcal{A}}$ to be the aggregated result, it will hods that 
\begin{equation}
    \label{eq:failure_case_1}
    \begin{aligned}
        & & &\ \ \boldsymbol{C}^\top(\boldsymbol{W}+\boldsymbol{W}^*_u)\boldsymbol{C}\boldsymbol{\theta}_{\mathcal{A}}+2\lambda_0\boldsymbol{\theta}_{\mathcal{A}}\\[5pt]
        & &=&\ \ \boldsymbol{C}^\top(\boldsymbol{W}+\boldsymbol{W}^*_u)\boldsymbol{y}
    \end{aligned}
\end{equation}
Notice that $\boldsymbol{\hat{w}}$ is the solution of \eqref{opt:com_info}, we have
\begin{equation}
    \label{eq:failure_case_2}
    \big(\boldsymbol{C}^\top\boldsymbol{W}\boldsymbol{C}+2\lambda_0\boldsymbol{I}\big)\boldsymbol{\theta}_{\mathcal{A}}=\boldsymbol{C}^\top\boldsymbol{W}\boldsymbol{y}
\end{equation}
Substituting \eqref{eq:failure_case_2} into \eqref{eq:failure_case_1}, we have
\begin{equation}
    \label{eq:failure_case_3}
    \boldsymbol{C}^\top\boldsymbol{W}^*_u\boldsymbol{C}\boldsymbol{\theta}_{\mathcal{A}}=\boldsymbol{C}^\top\boldsymbol{W}^*_u\boldsymbol{y}.
\end{equation}
\eqref{eq:failure_case_3} is an unrealistic requirement for the adversary $\mathcal{A}$. Without any prior knowledge of $\boldsymbol{w}^*_u$, $\mathcal{A}$ could not choose a $\boldsymbol{\theta}_{\mathcal{A}}\in\mathcal{C}_{\mathcal{A}}$ to satisfy \eqref{eq:failure_case_3}. To eliminate the uncertainty of incomplete information, the attacker $\mathcal{A}$ will introduce some constraints on $\boldsymbol{w}_u$ which involve the feedback $\boldsymbol{\theta}_{\mathcal{R}}$ from the victim $\mathcal{R}$. Different from the perfect case, $\mathcal{A}$ only masters the ranking result $\boldsymbol{\pi}_{\bar{\theta}}$. Let $\boldsymbol{\theta}_{\mathcal{R}}$ be the feedback of $\mathcal{R}$ utilized by $\mathcal{A}$:
\begin{equation}
    \boldsymbol{\theta}_{\mathcal{R}} = \left\{
    \begin{array}{rl}
        \boldsymbol{\bar{\theta}}, & \text{perfect case},\\[5pt]
        s(\boldsymbol{\pi}_{\boldsymbol{\bar{\theta}}}), & \text{imperfect case},
    \end{array}
    \right.
\end{equation}
where $s:\mathbb{N}^n\rightarrow\mathbb{R}^n_+$ assigns scores to each item in a ranking list according to its position. Furthermore, $\mathcal{A}$ constructs $\boldsymbol{\theta}_{\mathcal{A}}\in\mathcal{C}_{\mathcal{A}}$ as the pursued latent preference score to construct $\boldsymbol{\hat{w}}_k$ and mislead $\mathcal{R}$: $\boldsymbol{\theta}_{\mathcal{A}} = p(\boldsymbol{\theta}_{\mathcal{R}})$
where $p:\mathbb{R}^n_+\rightarrow\mathbb{R}^n_+$ is a permutation of $\boldsymbol{\theta}_{\mathcal{R}}$ to assure $\boldsymbol{\theta}_{\mathcal{A}}\in\mathcal{C}_{\mathcal{A}}$. We will conduct a profound analysis and design the corresponding algorithm based on the above analysis of $\boldsymbol{w}^*_u$ and $\boldsymbol{\theta}_{\mathcal{R}}$. Eliminating the interference of $\boldsymbol{w}^*_u$ is the most urgent issue for $\mathcal{A}$ in the incomplete game. The adversary $\mathcal{A}$ should forecast the effect of $\boldsymbol{w}^*_u$ and construct the corresponding attack strategy $\boldsymbol{\hat{w}}$. Here we split the attack strategy $\boldsymbol{\hat{w}}$ into two variables like \eqref{eq:decomposition} and each variable plays a different role in the following attack: $\boldsymbol{\hat{w}} = \boldsymbol{w}_k+\boldsymbol{w}_u.$
Generally speaking, $\boldsymbol{w}_k$ substitutes the accessible part $\boldsymbol{w}^*_k$ as the actual attack strategy and $\boldsymbol{w}_u$ imitates the behavior of $\boldsymbol{w}^*_u$ in $\mathcal{T}_{\mathcal{R}}$ \eqref{opt:reg_hodge}. Although the adversary does not know $\boldsymbol{w}^*_u$ and $\boldsymbol{\theta}_{\mathcal{R}}$ would be imperfect, $\boldsymbol{\theta}_{\mathcal{R}}$ could be the solution of an optimization problem like \eqref{opt:reg_hodge} and the first-order optimal condition holds: 
\begin{equation}
    \begin{aligned}
        & & &\ \ \boldsymbol{C}^\top(\boldsymbol{W}^*_k+\boldsymbol{W}^*_u)\boldsymbol{C}\boldsymbol{\theta}_{\mathcal{R}}+2\lambda_0\boldsymbol{\theta}_{\mathcal{R}}\\[5pt]
        & &=&\ \ \boldsymbol{C}^\top(\boldsymbol{W}^*_k+\boldsymbol{W}^*_u)\boldsymbol{y}.
    \end{aligned}    
\end{equation}

\begin{algorithm}
    \SetAlgoLined
    \SetKwInOut{Input}{Input}
    \SetKwInOut{Output}{Output}
    \Input{$\boldsymbol{C},\ \boldsymbol{y},\ \boldsymbol{w}^*_k,\ \boldsymbol{\theta}_{\mathcal{R}},\ \boldsymbol{\theta}_{\mathcal{A}},\ \rho_1>0,\ \rho_2>0, \gamma>0$.}
    
    $\mathcal{A}$ executes \textbf{ADMM} algorithm \eqref{Alg:Prox_ADMM} to solve \eqref{opt:com_info} or \eqref{opt:incom_info}:
    \[
      \boldsymbol{w}_k = \textbf{ADMM}(\boldsymbol{C},\ \boldsymbol{y},\ \boldsymbol{w}^*_k,\ \boldsymbol{\theta}_{\mathcal{R}},\ \boldsymbol{\theta}_{\mathcal{A}},\ \rho_1,\ \rho_2,\ \gamma),
    \]

    $\mathcal{A}$ transfers $\boldsymbol{w}_k$ into integer:
    \[
      \boldsymbol{\hat{w}}_k=\textbf{Integer}(\boldsymbol{w}_k),
    \]

    $\mathcal{R}$ executes \textbf{HodgeRank} algorithm \eqref{opt:reg_hodge} with the poisoned data $\boldsymbol{\hat{w}}$, 
    \[
      \boldsymbol{\hat{\theta}} = \textbf{HodgeRank}(\boldsymbol{C},\ \boldsymbol{y},\ \boldsymbol{\hat{w}})
    \]
    where
    \[\boldsymbol{\hat{w}} = \boldsymbol{\hat{w}}_k + \boldsymbol{w}^*_u.\]
  
    \Output{The misguided ranking list $\boldsymbol{\hat{\theta}}$.}
    \caption{Target Attack on HodgeRank}
    \label{alg:attack_hodge}
\end{algorithm}

If $\mathcal{A}$ hopes that $\boldsymbol{w}_u$ has the similar behavior like $\boldsymbol{w}^*_u$ in \eqref{opt:reg_hodge}, $\boldsymbol{w}_u$ will satisfy
\begin{equation}
    \begin{aligned}
        & & &\ \ \boldsymbol{C}^\top(\boldsymbol{W}^*_k+\boldsymbol{W}_u)\boldsymbol{C}\boldsymbol{\theta}_{\mathcal{R}}+2\lambda_0\boldsymbol{\theta}_{\mathcal{R}}\\[5pt]
        & &=&\ \ \boldsymbol{C}^\top(\boldsymbol{W}^*_k+\boldsymbol{W}_u)\boldsymbol{y}.
    \end{aligned}    
\end{equation}
With the help of $\boldsymbol{w}_u$, the adversary will come back to a situation like the complete game. By Theorem \ref{thm:hodge_nash_equilibrium}, $\boldsymbol{w}_k$ and $\boldsymbol{w}_k$ will play an ensemble to mislead $\mathcal{R}$ as
\begin{equation}
    \begin{aligned}
        & & &\ \ \boldsymbol{C}^\top(\boldsymbol{W}_k+\boldsymbol{W}_u)\boldsymbol{C}\boldsymbol{\theta}_{\mathcal{A}}+2\lambda_0\boldsymbol{\theta}_{\mathcal{A}}\\[5pt]
        & &=&\ \ \boldsymbol{C}^\top(\boldsymbol{W}_k+\boldsymbol{W}_u)\boldsymbol{y}.
    \end{aligned}    
\end{equation}
Adopting the minimal modification like \eqref{eq:mini_modf} as the objective function, the continuous operator of $\mathcal{A}$ in the incomplete game will be
\begin{flalign}
    (\mathcal{T}_{\mathcal{A}})&
    \begin{split}
    \underset{\boldsymbol{w}_k\geq 0,\boldsymbol{w}_u\geq 0}{\textbf{\textit{minimize}}}&\ \ \frac{\ 1\ }{\ 2\ }\big\|\boldsymbol{w}_k-\boldsymbol{w}^*_k\big\|^2_2\\[5pt]
    \textbf{\textit{subject to}}&\ \ \boldsymbol{B}_{\mathcal{R}}(\boldsymbol{w}^*_k+\boldsymbol{w}_u)=\boldsymbol{0},\\[5pt]
    &\ \ \boldsymbol{B}_{\mathcal{A}}(\boldsymbol{w}_k+\boldsymbol{w}_u)=\boldsymbol{0},
    \end{split}
    &\hspace{5em} \label{opt:incom_info}
\end{flalign}
where 
\begin{equation}
    \begin{aligned}
        &\boldsymbol{B}_{\mathcal{R}} &=&\ \ \boldsymbol{C}^\top\textbf{\textit{diag}}\big(\big(\boldsymbol{C}+2\lambda_0\boldsymbol{I}\big)\boldsymbol{\theta}_{\mathcal{R}}-\boldsymbol{y}\big)\\[5pt]
        &\boldsymbol{B}_{\mathcal{A}} &=&\ \ \boldsymbol{C}^\top\textbf{\textit{diag}}\big(\big(\boldsymbol{C}+2\lambda_0\boldsymbol{I}\big)\boldsymbol{\theta}_{\mathcal{A}}-\boldsymbol{y}\big).
    \end{aligned}
\end{equation}
Here we adopt the \textbf{ADMM} to solve the attack operation $\boldsymbol{w}_k$. Let
the augmented Lagrangian function $L(\boldsymbol{w}_k,\boldsymbol{w}_u,\boldsymbol{\mu}_1,\boldsymbol{\mu}_2)$ of optimization problem \eqref{opt:incom_info} be
\begin{equation*}
  \label{Lagrangian:3}
  \begin{aligned}
    & & &\ \ \ L(\boldsymbol{w}_k,\ \boldsymbol{w}_u,\ \boldsymbol{\mu}_1,\ \boldsymbol{\mu}_2)\\[5pt]
    & &=&\ \ \frac{\ 1\ }{\ 2\ }\Big\|\boldsymbol{w}_k-\boldsymbol{w}^*_k\Big\|^2_2+\Big\langle\boldsymbol{\mu}_1,\ \boldsymbol{B}_{\mathcal{R}}\cdot\Big(\boldsymbol{w}_u+\boldsymbol{w}_k^*\Big)\Big\rangle\\[5pt]
    & &+&\ \ \frac{\rho_1}{\ 2\ }\Big\|\boldsymbol{B}_{\mathcal{R}}\cdot\Big(\boldsymbol{w}_u+\boldsymbol{w}_k^*\Big)\Big\|^2_2+\Big\langle\boldsymbol{\mu}_2,\ \boldsymbol{B}_{\mathcal{A}}\cdot\Big(\boldsymbol{w}_u+\boldsymbol{w}_k\Big)\Big\rangle\\[5pt]
    & &+&\ \ \frac{\rho_2}{\ 2\ }\Big\|\boldsymbol{B}_{\mathcal{A}}\cdot\Big(\boldsymbol{w}_u+\boldsymbol{w}_k\Big)\Big\|^2_2,
  \end{aligned}
\end{equation*}
where $\boldsymbol{\mu}_1,\boldsymbol{\mu}_2\in\mathbb{R}^n$ are the Lagrangian multiplier variables, and $\rho_1, \rho_2>0$ are penalty parameters.

The \textbf{ADMM} consists of the following iterations:
\begin{subequations}
    \label{Alg:Prox_ADMM}
    \begin{align}
        \boldsymbol{w}_k^{t+1}&:=\underset{\boldsymbol{w}_k\geq0}{\textbf{\textit{arg min}}}\ \ L(\boldsymbol{w}_k,\ \boldsymbol{w}^t_u,\ \boldsymbol{\mu}^t_1,\ \boldsymbol{\mu}^t_2),\\[7.5pt]
        \boldsymbol{w}_u^{t+1}&:=\underset{\boldsymbol{w}_u\geq0}{\textbf{\textit{arg min}}}\ \left\{\ \boldsymbol{g}^\top_{t+1}\boldsymbol{w}_u + \frac{\ 1\ }{2\gamma}\big\|\boldsymbol{w}_u-\boldsymbol{w}^{t}_u\big\|^2_2\ \right\},\\[7.5pt]
        \boldsymbol{\mu}^{t+1}_1 &:=\boldsymbol{\mu}^{t}_1 + \rho_1\cdot\boldsymbol{B}_{\mathcal{R}}\cdot\Big(\boldsymbol{w}^{t+1}_u+\boldsymbol{w}_k^*\Big),\\[10pt]
        \boldsymbol{\mu}^{t+1}_2 &:=\boldsymbol{\mu}^{t}_2 + \rho_2\cdot\boldsymbol{B}_{\mathcal{A}}\cdot\Big(\boldsymbol{w}^{t+1}_u+\boldsymbol{w}^{t+1}_k\Big).
    \end{align}
\end{subequations}
where $\gamma>0$ is a proximal parameter and $\boldsymbol{g}_{t+1}$ is defined
\begin{equation}
    \boldsymbol{g}_{t+1} := \nabla_{\boldsymbol{w}_u}L(\boldsymbol{w}^{t+1}_k,\ \boldsymbol{w}_u,\ \boldsymbol{\mu}^t_1,\ \boldsymbol{\mu}^t_2)\ \bigg\vert_{\boldsymbol{w}_u=\boldsymbol{w}^{t}_u}.
\end{equation}
Specifically, the update rule of $\boldsymbol{w}_k$ is
\begin{equation}
    \begin{aligned}
        & \boldsymbol{w}^{t+1}_k=\ \ \mathrm{Proj}_{+}\bigg[\bigg(\boldsymbol{I}+\rho_2\boldsymbol{B}^\top_{\mathcal{R}}\boldsymbol{B}_{\mathcal{R}}\bigg)^{-1}\cdot\\
        & {\color{white}\boldsymbol{w}^{t+1}_k=\mathrm{Proj}_{+}\ \ \ \ \ \ } \bigg(\boldsymbol{w}^*_k-\boldsymbol{B}^\top_{\mathcal{R}}\boldsymbol{\mu}^t_2-\rho_2\boldsymbol{B}^\top_{\mathcal{R}}\boldsymbol{B}_{\mathcal{R}}\boldsymbol{w}^t_u\bigg)\bigg],   
    \end{aligned}
\end{equation}
and the update rule of $\boldsymbol{w}_u$ is
\begin{equation}
  \begin{aligned}
    & \boldsymbol{w}^{t+1}_u = \mathrm{Proj}_{+}\bigg[\bigg(\frac{\ 1\ }{\gamma}\boldsymbol{I}+\rho_1\boldsymbol{B}^\top_{\mathcal{R}}\boldsymbol{B}_{\mathcal{R}}+\rho_2\boldsymbol{B}^\top_{\mathcal{A}}\boldsymbol{B}_{\mathcal{A}}\bigg)^{-1}\cdot\\
    & {\color{white}\boldsymbol{w}^{t+1}_u = \mathrm{Proj}_{+}\ \ \ \ }\bigg(\boldsymbol{B}^\top_{\mathcal{R}}\boldsymbol{\mu}^t_1+\rho_1\boldsymbol{B}^\top_{\mathcal{R}}\boldsymbol{B}_{\mathcal{R}}\boldsymbol{w}^*_k+\boldsymbol{B}^\top_{\mathcal{A}}\boldsymbol{\mu}^t_2\\
    & {\color{white}\boldsymbol{w}^{t+1}_u = \mathrm{Proj}_{+}\ \ \ \ \ \ \ \ \ \ \ \ \ }+\rho_2\boldsymbol{B}^\top_{\mathcal{A}}\boldsymbol{B}_{\mathcal{A}}\boldsymbol{w}^{t+1}_k-\frac{\ 1\ }{\gamma}\boldsymbol{w}^{t}_u\bigg)\bigg],
  \end{aligned}
\end{equation}
where $\boldsymbol{I}$ is an identity matrix. According to \cite{DengYin16-ADMM}, the adopted ADMM has a linear convergence rate for solving the concerned optimization problem \eqref{opt:incom_info}. The general framework for $\mathcal{A}$ to construct the attack strategies in complete game and incomplete game is illustrated as Algorithm \ref{alg:attack_hodge}. 

\vspace{0.25cm}
\noindent\textbf{Remark. }As we introduce a variable $\boldsymbol{w}_u$ to estimate the unknown pairwise comparisons $\boldsymbol{w}^*_u$, there exist two penalty parameters $\rho_1$ and $\rho_2$ in Algorithm 1. The complete game discussed Theorem 1 (only involved single penalty $\rho$) is a special case of the incomplete game as $\boldsymbol{w}_u\equiv 0$ in \eqref{opt:incom_info}. We provide the augmented Lagrangian function in the complete case for Algorithm \ref{alg:attack_hodge} in the supplementary material. It is noteworthy that we can obtain that $\boldsymbol{\theta}_{\mathcal{A}}$ is a fixed point of $\mathcal{T}_{\mathcal{A}}\circ\mathcal{T}_{\mathcal{R}}$ when $\boldsymbol{\theta}_{\mathcal{R}}=\boldsymbol{\bar{\theta}}$ and $\mathcal{A}$ only masters incomplete information $\boldsymbol{w}^*_k$ through Theorem 1. Establishing and characterizing the fixed points of the incomplete and imperfect dynamic system $\mathcal{T}_{\mathcal{A}}\circ\mathcal{T}_{\mathcal{R}}$ are challenging but still open problems.

\section{Adversarial Game with Spectral Method}
In this section, we construct the operator of the adversary to execute the target attack against \textbf{RankCentrality}\cite{DBLP:journals/ior/NegahbanOS17}. By solving the eigenvector problem \eqref{eq:detailed_balance}, \textbf{RankCentrality} obtains the latent preference score $\boldsymbol{\bar{\theta}}$, which is the leading eigenvector of the empirical transition matrix $\boldsymbol{P}^*$. When the victim $\mathcal{R}$ is designated as \textbf{RankCentrality}, the corresponding adversarial game would satisfy the following claims:
\begin{itemize}
    \item The set of players contains two members $\boldsymbol{P} := \{\mathcal{A},\ \mathcal{R}\}$. Here $\mathcal{R}$ represents the \textbf{RankCentrality}.
    \vspace{0.15cm}
    \item The strategy profile $\boldsymbol{S}_{\mathcal{R}}$ is $\mathbb{R}^n$, 
    %
    which is a non-empty compact metric space.
    \vspace{0.15cm}
    \item The continuous operator of $\mathcal{R}$ is the eigenvector problem as \eqref{eq:transition_matrix_condition} and \eqref{eq:detailed_balance}.
\end{itemize}

Different from attacking against \textbf{HodgeRank}, the operator of \textbf{RankCentrality} is not involved with the pairwise comparisons $\{\boldsymbol{G},\ \boldsymbol{y},\ \boldsymbol{w}^*\}$ directly. As a consequence, $\mathcal{A}$ could not produce the poisoned data $\boldsymbol{\hat{w}}$ with the help of optimal conditions for the leading eigenvector problem \eqref{eq:transition_matrix_condition} and \eqref{eq:detailed_balance}. Instead of manipulating the weights of pairwise comparisons, the attacker $\mathcal{A}$ would like to counterfeit the independent and identically distributed (\textit{i.i.d.}) pairwise samples from $\boldsymbol{V}\times\boldsymbol{V}$ by $\boldsymbol{\theta}_{\mathcal{A}}=\{\hat{\theta}_1,\ \hat{\theta}_2,\ \dots,\ \hat{\theta}_n\}\in\mathcal{C}_{\mathcal{A}}$. However, such a kind of the manipulation is prohibited from the unknown distribution which gives the probability that $i\succ j$ turns out true. Moreover, verification of general linear model assumptions like \eqref{eq:BTL} is challenging under the most practical scenarios. Sampling pairwise comparisons from general linear models with $\boldsymbol{\theta}_{\mathcal{A}}\in\mathcal{C}_{\mathcal{A}}$ would be a sub-optimal solution.

To design distribution-free methods for generating paired data and conduct the tractable target attack strategies on spectral ranking algorithm, the attacker $\mathcal{A}$ will construct a new Markov chain with the target invariant distribution $\boldsymbol{\theta}_{\mathcal{A}}\in\mathcal{C}_{\mathcal{A}}$. By \eqref{eq:detailed_balance}, we know that $\boldsymbol{\theta}_{\mathcal{A}}$ must be the largest left eigenvector of an unknown transition matrix $\boldsymbol{P}_{\mathcal{A}}$:
\begin{equation}
  \label{eq:optimal_condition_spectral}
  \boldsymbol{\theta}^\top_{\mathcal{A}}=\boldsymbol{\theta}^\top_{\mathcal{A}}\cdot\boldsymbol{P}_{\mathcal{A}}.
\end{equation}
Once the attacker $\mathcal{A}$ obtains $\boldsymbol{P}_{\mathcal{A}}$, he/she could elicit the poisoned data $\boldsymbol{\hat{w}}$ from the corresponding transition matrix $\boldsymbol{P}_{\mathcal{A}}$ like \eqref{eq:empirical_matrix}. Here we adopt the Markov Chain Monte Carlo (\textbf{MCMC}) method to generate the desired Markov chain $\boldsymbol{P}_{\mathcal{A}}$ with the constraint \eqref{eq:optimal_condition_spectral}, which can be understood as the other form of ``\textit{optimal condition}''. The basic idea underlying the \textbf{MCMC} is to construct a Markov chain with the desired distribution $\boldsymbol{\theta}_{\mathcal{A}}$ as its invariant distribution. As the procedure runs long enough, the samples generated from the Markov chain serve as a good approximation to the desired samples drawn from the unknown pairwise comparison distribution. Given the state space $\boldsymbol{V}$ and the initial distribution, the desired Markov chain is characterized by its transition matrix $\boldsymbol{P}_{\mathcal{A}}$ whose leading lefty eigenvector is $\boldsymbol{\theta}_{\mathcal{A}}$. Constructing a (row) stochastic matrix with a prescribed left eigenvector $\boldsymbol{\theta}_{\mathcal{A}}$ associated with the Perron root $1$ is a special type of the so-called ``\textbf{\textit{inverse eigenvalue problem}}'' \cite{doi:10.1137/S0036144596303984,BARRETT2020276}. To be more specific, since the main concern of $\mathcal{A}$ is to maintain $\boldsymbol{\theta}_{\mathcal{A}}$ as the leading eigenvector of the matrix $\boldsymbol{P}_{\mathcal{A}}$, $\mathcal{A}$ will solve an inverse eigenvalue problem as the continuous operator. Establishing the inverse eigenvalue problem and ensuring that $\boldsymbol{\theta}_{\mathcal{A}}$ becomes a fixed point of the continuous operator game need some constraints on $\boldsymbol{P}_{\mathcal{A}}$. The obvious structural constraint is that $\boldsymbol{P}_{\mathcal{A}}$ will be row stochastic. In addition, we introduce the worst-case performance measurement as the other constraint on $\boldsymbol{P}_{\mathcal{A}}$. Considering the possible incomplete information and imperfect feedback, the worst-case performance measurement tells how good the approximation of \textbf{MCMC} to the unknown pairwise comparison distribution is. In the literature, one of the commonly employed measurements for evaluating the performance of an \textbf{MCMC} algorithm is the so-called asymptotic variance \cite{10.1214/09-AOS735} which is the focus of this paper. 

Suppose that $\mathcal{V}=\{\boldsymbol{v}_0,\ \boldsymbol{v}_1,\ \dots\}$ is a discrete time Markov chain on $\boldsymbol{V}$ with some transition probability matrix $\boldsymbol{P}_{\mathcal{A}}$ and invariant distribution $\boldsymbol{\theta}_{\mathcal{A}}$. We then use the time average 
\begin{equation}
    \label{eq:time_average}
    \frac{\ 1\ }{\ m\ }\sum_{i=0}^{m-1}f(\boldsymbol{v}_i)
\end{equation}
as an estimation for the space average
\begin{equation}
    \label{eq:space_average}
    \mathbb{E}_{\boldsymbol{\theta}_{\mathcal{A}}}[f] = \underset{v\in\boldsymbol{V}}{\sum}f(\boldsymbol{v})\cdot\hat{\theta}_{\boldsymbol{v}},
\end{equation}
where $f$ is a real-valued function defined on $\mathcal{V}$. By the strong law of large numbers, we should have
\begin{equation}
    \frac{\ 1\ }{\ m\ }\sum_{i=0}^{m-1}f(\boldsymbol{v}_i)\ \xrightarrow[]{a.s.}\ \mathbb{E}_{\boldsymbol{\theta}_{\mathcal{A}}}[f].
\end{equation}
Moreover, the asymptotic mean squared error
\begin{equation}
    \label{eq:asymptotic_variance}
    \nu(f, \boldsymbol{P}_{\mathcal{A}}) \coloneqq \underset{m\rightarrow\infty}{\textbf{\textit{lim}}}\ m\cdot\mathbb{E}_{\boldsymbol{\mu}_0}\left[\frac{\ 1\ }{\ m\ }\sum_{i=0}^{m-1}f(\boldsymbol{v}_i)-\mathbb{E}_{\boldsymbol{\theta}_{\mathcal{A}}}[f]\right]
\end{equation}
is equivalent to the asymptotic variance when the initial distribution $\boldsymbol{\mu}_0$ on $\boldsymbol{V}$ has the transition matrix $\boldsymbol{P}_{\mathcal{A}}$. Note that the mean square errors of time averages and their variance coincide up to an order $1/m$ for large times of sampling $m$. The difference is equal to the square of the bias, which is of order $1/m^2$. The bias itself is of order $1/n$. It is known in the literature \cite{iosifescu2014finite} that for any irreducible transition matrix $\boldsymbol{P}_{\mathcal{A}}$, the limit in \eqref{eq:asymptotic_variance} exists no matter what initial distribution $\mu_0$ is. Obviously, a smaller asymptotic variance indicates a better approximation. The real challenge is that in the context of \textbf{MCMC} methods, we also look for a matrix $\boldsymbol{P}_{\mathcal{A}}$ of which the corresponding worst-case asymptotic variance is minimized. Here the worst-case asymptotic variance means that the game between the attacker and the ranker is the incomplete information case: it is possible to exist some unknown comparisons for the attacker but they are known to the ranker. When the adversary $\mathcal{A}$ only accesses the incomplete information and the imperfect feedback, blindly pursuing the efficiencies of \textbf{MCMC} algorithms would deviate from the purpose of $\mathcal{A}$. As a consequence, it is of great interest to analyze the asymptotic variance over all functions $f$ in the most pessimistic view. In the present paper, we take the maximum of asymptotic variance \eqref{eq:asymptotic_variance} over functions of interest $\mathcal{F}$ as the criterion for evaluating the performance of \textbf{MCMC} algorithms. Moreover, we adopt the worst-case analysis to derive a lower bound of the asymptotic variance over general Markov chains with invariant probability $\boldsymbol{\theta}_{\mathcal{A}}$. The worst scenario of the asymptotic variance for each fixed $\boldsymbol{P}_{\mathcal{A}}$ amounts to solving
\begin{equation}
    \label{opt:asympotic_var}
    \underset{\boldsymbol{P}}{\textbf{\textit{inf}}}\ \underset{f}{\textbf{\textit{sup}}}\ \ \nu({f}, \boldsymbol{P}).
\end{equation}
Assuming the generic condition that the Perron root of $\boldsymbol{P}_{\mathcal{A}}$ is unique and identifying the function $f$ in \eqref{eq:time_average} and \eqref{eq:space_average} on $\boldsymbol{V}$ as a vector $\boldsymbol{f}\in\mathbb{R}^n$, it is known \cite{iosifescu2014finite} that the asymptotic variance $\nu({f}, \boldsymbol{P}_{\mathcal{A}})$ \eqref{eq:asymptotic_variance} can be expressed as 
\begin{equation}
    \begin{aligned}
        & \nu(\boldsymbol{f},\ \boldsymbol{P}_{\mathcal{A}}) &=&\ \ 2\cdot\bigg\langle\Big(\big(\boldsymbol{I}-\boldsymbol{P}_{\mathcal{A}}-\boldsymbol{1}\boldsymbol{\theta}^\top_{\mathcal{A}}\big)^{-1}-\boldsymbol{1}\boldsymbol{\theta}^\top_{\mathcal{A}}\Big)\ \cdot\\
        & & &\ \ \ \ \ \ \ \ \ \Big(\boldsymbol{f}-\mathbb{E}_{\boldsymbol{\theta}_{\mathcal{A}}}[\boldsymbol{f}]\Big),\ \boldsymbol{f}-\mathbb{E}_{\boldsymbol{\theta}_{\mathcal{A}}}[\boldsymbol{f}]\bigg\rangle_{\boldsymbol{\theta}_{\mathcal{A}}}\\
        & & &\ \ \ \ -\ \bigg\langle \boldsymbol{f}-\mathbb{E}_{\boldsymbol{\theta}_{\mathcal{A}}}[\boldsymbol{f}],\ \boldsymbol{f}-\mathbb{E}_{\boldsymbol{\theta}_{\mathcal{A}}}[\boldsymbol{f}]\bigg\rangle_{\boldsymbol{\theta}_{\mathcal{A}}},
    \end{aligned}
\end{equation}
where $\mathbb{E}_{\boldsymbol{\theta}_{\mathcal{A}}}[\boldsymbol{f}]$ is defined as 
\begin{equation*}
    \mathbb{E}_{\boldsymbol{\theta}_{\mathcal{A}}}[\boldsymbol{f}] = \sum_{i\in[n]}f_i\cdot\hat{\theta}_i,\ \boldsymbol{f}=[f_1,\dots,f_n]^\top,
\end{equation*}
and $\langle \cdot , \cdot \rangle_{\boldsymbol{\theta}}$ stands for the weighted inner product with respect to $\boldsymbol{\theta}$
\begin{equation}
    \big\langle\boldsymbol{x},\ \boldsymbol{y}\big\rangle_{\boldsymbol{\theta}} = \boldsymbol{y}^\top\boldsymbol{\Theta}\ \boldsymbol{x}=\big\langle \boldsymbol{\Theta}^{\frac{1}{2}}\boldsymbol{x},\ \boldsymbol{\Theta}^{\frac{1}{2}}\boldsymbol{y}\big\rangle,
\end{equation}
$\boldsymbol{\Theta}=\textbf{diag}(\boldsymbol{\theta})$ is a diagonal matrix with $\boldsymbol{\theta} \in \mathbb{R}^n$ being the diagonal components. Throughout the rest of this paper, we restrict $\boldsymbol{f}$ in the space of mean zero vector $\boldsymbol{f}\in\mathcal{F}$ 
\begin{equation}
    \mathcal{F}\coloneqq\left\{\boldsymbol{f}\in\mathbb{R}^n\ \Big\vert\ \mathbb{E}_{\boldsymbol{\theta}_{\mathcal{A}}}[\boldsymbol{f}]=0\right\},
\end{equation}
and by re-centering $\boldsymbol{f}$ if necessary, we may assume without loss of generality henceforth that $\mathbb{E}_{\boldsymbol{\theta}_{\mathcal{A}}}[\boldsymbol{f}]=0$. Then we have
\begin{equation}
    \label{eq:asymptotic_var_ref}
    \begin{aligned}
        & & &\ \ \nu(\boldsymbol{f},\ \boldsymbol{P}_{\mathcal{A}})\\[5pt]
        & &=&\ \ 2\cdot\Big\langle\big(\boldsymbol{I}-\boldsymbol{P}_{\mathcal{A}}\big)^{-1}\boldsymbol{f},\ \boldsymbol{f}\Big\rangle_{\boldsymbol{\theta}_{\mathcal{A}}}-\Big\langle \boldsymbol{f},\ \boldsymbol{f}\Big\rangle_{\boldsymbol{\theta}_{\mathcal{A}}},
    \end{aligned}
\end{equation}
where $\boldsymbol{I}-\boldsymbol{P}_{\mathcal{A}}$ is the fundamental matrix of $\boldsymbol{P}_{\mathcal{A}}$. Since $\nu(\boldsymbol{f},\ \boldsymbol{P}_{\mathcal{A}})$ in \eqref{eq:asymptotic_var_ref} is homogeneous in $\boldsymbol{f}$, it suffices to assume further that
\begin{equation}
    \label{eq:norm_1}
    \|\boldsymbol{f}\|_{\boldsymbol{\theta}_{\mathcal{A}}} \coloneqq \big\langle \boldsymbol{f},\ \boldsymbol{f}\big\rangle_{\boldsymbol{\theta}_{\mathcal{A}}}=1.
\end{equation}
Thus, for a fixed transition probability matrix $\boldsymbol{P}_{\mathcal{A}}$, the asymptotic variance $\nu(f,\ \boldsymbol{P}_{\mathcal{A}})$ depends on the underlying vector $\boldsymbol{f}$. The worst-case analysis aims to solving
\begin{equation}
    \label{opt:asympotic_var_reformulate}
    \underset{\boldsymbol{P}_{\mathcal{A}}}{\textbf{\textit{inf}}}\ \ \underset{\boldsymbol{f}\in\mathcal{F},\ \|\boldsymbol{f}\|_{\Scale[0.5]{\boldsymbol{\theta}_{\mathcal{A}}}}=1}{\textbf{\textit{sup}}}\ \ \Big\langle\big(\boldsymbol{I}-\boldsymbol{P}_{\mathcal{A}}\big)^{-1}\boldsymbol{f},\ \boldsymbol{f}\Big\rangle_{\boldsymbol{\theta}_{\mathcal{A}}}.
\end{equation}

\subsection{Reversible Stochastic Transition Matrix}

Many popular \textbf{MCMC} methods assume reversibility. One good reason is that a reversible matrix with minimum asymptotic variance is achievable in closed form. It is known that minimizing the asymptotic variance of reversible matrix is equivalent to finding the reversible $\boldsymbol{P}_{\mathcal{A}}$ whose second largest eigenvalue is minimized \cite{wu2015constructing}. Here we give the definition of the reversible stochastic transition matrix. 

\begin{definition}[Reversibility of Stochastic Transition Matrix]
    A row stochastic matrix $\boldsymbol{P}$ is said to be reversible relative to non-negative $\boldsymbol{\theta}$ if and only if it satisfies the so-called detailed balance condition,
    \begin{equation}
        \boldsymbol{\Theta}\boldsymbol{P} = \boldsymbol{P}^\top\boldsymbol{\Theta},
    \end{equation}
    where $\boldsymbol{\Theta} = \textbf{diag}(\boldsymbol{\theta})$ is a diagonal matrix with $\boldsymbol{\theta} \in \mathbb{R}^n$ being the diagonal components.
\end{definition}

The following theorem dissects the spectral structure of the reversible stochastic transition matrix $\boldsymbol{P}_{\mathcal{A}}$ with the given leading eigenvector $\boldsymbol{\theta}_{\mathcal{A}}$. This is a necessary condition of the desired matrix to archive the purpose of $\mathcal{A}$. Its proof can be found in the supplementary material which is a special case of \cite[Theorem~1]{frigessi1992optimal}.

\begin{restatable}{theorem}{reversiblematrix}
    \label{thm:reversible_matrix}
    Given a reversible stochastic transition matrix $\boldsymbol{P}_{\mathcal{A}}$ and its leading eigenvector $\boldsymbol{\theta}_{\mathcal{A}}$, we have:
    \begin{enumerate}
        \item The second largest eigenvalue $\lambda_2$ of stochastic matrix $\boldsymbol{P}_{\mathcal{A}}=\{\hat{P}_{ij}\}$, reversible w.r.t $\boldsymbol{\theta}_{\mathcal{A}}=[\hat{\theta}_1,\ \hat{\theta}_2,\ \dots,\ \hat{\theta}_n]^{\top}$ satisfies
        \begin{equation*}
            \lambda_2 \geq -\frac{\hat{\theta}_1}{1-\hat{\theta}_1},\ \text{where}\ 0<\hat{\theta}_1\leq\hat{\theta}_2\leq\dots\leq\hat{\theta}_n.
        \end{equation*}
        \item If the second largest eigenvalue of matrix $\boldsymbol{P}_{\mathcal{A}}$ attains its lower bound as
        \begin{equation}
            \lambda_2 = -\frac{\hat{\theta}_1}{1-\hat{\theta}_1},
        \end{equation}
        $\boldsymbol{P}_{\mathcal{A}}$ has the form:
            \begin{equation}
                \begingroup 
                \renewcommand\arraystretch{1}
                \setlength\arraycolsep{3.5pt}
                \mbox{\normalsize\(\boldsymbol{P}_{\mathcal{A}} = 
                \begin{pmatrix}
                    0 & \frac{\displaystyle \hat{\theta}_2}{\displaystyle 1-\hat{\theta}_1} & \cdots & \frac{\displaystyle \hat{\theta}_n}{\displaystyle 1-\hat{\theta}_1}\\
                    \frac{\displaystyle \hat{\theta}_1}{\displaystyle 1-\hat{\theta}_1} &  &  & \\
                    \vdots &  &\boldsymbol{P}_{\mathcal{A}_2}  & \\
                    \frac{\displaystyle \hat{\theta}_1}{\displaystyle 1-\hat{\theta}_1} &  &  & 
                \end{pmatrix}
                \)}
                \endgroup
        \end{equation}
        where $\boldsymbol{P}_{\mathcal{A}_2}$ is in detailed balance with the vector
        \begin{equation}
            \boldsymbol{\theta}_{\mathcal{A}}/\hat{\theta}_1=[\hat{\theta}_2,\dots,\hat{\theta}_n],
        \end{equation}
        and has constant row sums.
        \end{enumerate}
\end{restatable}

Based on the Theorem \ref{thm:reversible_matrix}, we can recursively construct the desire reversible stochastic transition matrix as following. Begin with $\boldsymbol{\theta}^{(1)}\coloneqq\boldsymbol{\theta}_{\mathcal{A}}$, for $k=1,\dots,n-1$, we define
\begin{equation*}
    \begingroup 
    \renewcommand\arraystretch{1}
    \setlength\arraycolsep{2.75pt}
    \mbox
    {
        \normalsize
        \(
            \boldsymbol{P}^{(k)} \coloneqq 
            \begin{pmatrix}
                0 & \frac{\displaystyle\ \ \theta^{(k)}_{k+1}\ \ }{\displaystyle\ \ \beta_{k+1}\ \ } & \cdots & \frac{\displaystyle\ \ \theta^{(k)}_{n}\ \ }{\displaystyle\ \ \beta_{k+1}\ \ }\\[7.5pt]
                1-\alpha_{k+1} &  &  & \\[5pt]
                \vdots &  &\alpha_{k+1}\cdot\boldsymbol{P}^{(k+1)}  & \\[5pt]
                1-\alpha_{k+1} &  &  & 
            \end{pmatrix},
        \)
    }
    \endgroup
\end{equation*}
where
\begin{subequations}
    \begin{align}
        \beta_{k+1}\ &\coloneqq\ \ 1\ -\ \theta^{(k)}_{k}\\[5pt]
        \alpha_{k+1}\ &\coloneqq\ \ 1\ -\ \theta^{(k)}_{k}\ /\ \beta_{k+1}\\[5pt]
        \boldsymbol{\theta}^{(k+1)}\ &\coloneqq\ \ \boldsymbol{\theta}^{(k)}\ /\ \beta_{k+1}.
    \end{align}
\end{subequations}

The following corollary shows that the recurrent construction of $\boldsymbol{P}_{\mathcal{A}}$ could attain the minimal value of the second largest eigenvalue. 

\begin{restatable}{corollary}{reversibilefinal}
    \label{coly:reversibile_final}
    The following $\boldsymbol{P}_{\mathcal{A}}$ is a reversible stochastic matrix: 
    \begin{equation}
        \label{eq:reversibile_transition_matrix}
        \boldsymbol{P}_{\mathcal{A}}\coloneqq\boldsymbol{P}^{(1)}+\left(1-\sum^{n-1}_{j=1}P^{(1)}_{nj}\right)\boldsymbol{e}_n\boldsymbol{e}^\top_n,
    \end{equation}
    where $P^{(1)}_{nj}$ denotes the $(n,j)$ entry of $\boldsymbol{P}^{(1)}$ and $\boldsymbol{e}_n$ is the $n$-th standard unit vector of $\mathbb{R}^n$. Furthermore, the second largest eigenvalue $\lambda_2$ of $\boldsymbol{P}_{\mathcal{A}}$ is minimal among all reversible matrices with the stationary distribution $\boldsymbol{\theta}_{\mathcal{A}}$.
\end{restatable}
We end the discussion of the reversible stochastic transition matrix case with the following theorem which establishes the lower bound of worst-case asymptotic variance of \eqref{eq:reversibile_transition_matrix}. With Corollary \ref{coly:reversibile_final}, we know that $\boldsymbol{\theta}_{\mathcal{A}}$ would be a fixed point of the complete game between the perpetrator and the spectral ranking algorithm.
\begin{restatable}{theorem}{reversiblebound}
    \label{thm:reversible_bound}
    Let $\boldsymbol{P}_{\mathcal{A}}$ be a stochastic matrix, reversible with respect to $\boldsymbol{\theta}_{\mathcal{A}}$ and $\nu(f,\boldsymbol{P}_{\mathcal{A}})$ be the asymptotic variance of
    \begin{equation}
        \frac{\ 1\ }{\ m\ }\sum_{k=0}^{m-1}f(\boldsymbol{v}_k)  
    \end{equation}
    for norm $1$ function $f:[n]\rightarrow\mathbb{R}$ like \eqref{eq:norm_1}, where $\mathcal{V}=\{\boldsymbol{v}_0,\ \boldsymbol{v}_2,\ \dots\}$ is a discrete time Markov chain on $\boldsymbol{V}$ with some transition probability matrix $\boldsymbol{P}_{\mathcal{A}}$ and invariant distribution $\boldsymbol{\theta}_{\mathcal{A}}$. Then
    \begin{equation}
        \underset{\boldsymbol{f}\in\mathcal{F},\ \|\boldsymbol{f}\|_{\Scale[0.5]{\boldsymbol{\theta}_{\mathcal{A}}}}=1}{\textbf{\textit{sup}}}\nu(f,\boldsymbol{P}_{\mathcal{A}})\geq 1 - 2\hat{\theta}_1,
    \end{equation}
    and any stochastic matrix which attains the lower bound must satisfy
    \begin{equation}
        \lambda_2 = -\frac{\hat{\theta}_1}{1-\hat{\theta}_1},
    \end{equation}
    and hence be of the form \eqref{eq:reversibile_transition_matrix}.
\end{restatable}
\subsection{Irreversible Stochastic Transition Matrix}
However, the original \textbf{RankCentrality} approach does not require the reversibility of $\boldsymbol{P}^*$ in \eqref{eq:empirical_matrix}. Next, we explore the worst-case analysis of the asymptotic variance for irreversible transition matrix. We first establish the lower bound of \eqref{eq:asymptotic_var_ref} by the following theorem which is a special case of \cite[Theorem~2.1]{huang2018optimal} . The detailed proof is provided in the supplementary material. 
\begin{restatable}{theorem}{irrevisiblematrix}
    \label{thm:irrevisible_matrix}
    Assume that $\boldsymbol{P}_{\mathcal{A}}$ is an irreducible transition matrix with invariant distribution $\boldsymbol{\theta}_{\mathcal{A}}=[\hat{\theta}_1,\dots,\hat{\theta}_n]^\top$. It holds that
    \begin{equation}
        \frac{\hat{\theta}_n}{\hat{\theta}_1+\hat{\theta}_n}\leq v(\boldsymbol{P}_{\mathcal{A}}),
    \end{equation}
    where 
    \begin{equation*}
        \begin{aligned}
            &v(\boldsymbol{P}_{\mathcal{A}})&:=&\ \ \underset{\boldsymbol{f}\in\mathcal{F},\ \|\boldsymbol{f}\|_{\Scale[0.5]{\boldsymbol{\theta}_{\mathcal{A}}}}=1}{\textbf{\textit{sup}}} \nu(f,\boldsymbol{P}_{\mathcal{A}})\\
            & &=&\ \ \underset{\boldsymbol{f}\in\mathcal{F},\ \|\boldsymbol{f}\|_{\Scale[0.5]{\boldsymbol{\theta}_{\mathcal{A}}}}=1}{\textbf{\textit{sup}}} \Big\langle\big(\boldsymbol{I}-\boldsymbol{P}_{\mathcal{A}}\big)^{-1}\boldsymbol{f},\ \boldsymbol{f}\Big\rangle_{\boldsymbol{\theta}_{\mathcal{A}}}\\[7.5pt]
            & &=&\ \ 2\big\langle\boldsymbol{Z}\boldsymbol{f},\boldsymbol{f}\big\rangle_{\boldsymbol{\theta}_{\Scale[0.45]{\mathcal{A}}}}-1,
        \end{aligned}
    \end{equation*}
\end{restatable}

Let $\boldsymbol{Z} = (\boldsymbol{I}-\boldsymbol{P}_{\mathcal{A}})^{-1}$. From the lower bound of worst-case problem \eqref{opt:asympotic_var_reformulate}, we know that
\begin{equation}
    \label{eq:lower_bound}
    v(\boldsymbol{P}_{\mathcal{A}})\ \geq\ \underset{\boldsymbol{f}\in\mathcal{F},\|\boldsymbol{f}\|_{\Scale[0.5]{\boldsymbol{\theta}_{\mathcal{A}}}}=1}{\textbf{\textit{max}}}\ \left\langle\boldsymbol{Z}\boldsymbol{f},\ \boldsymbol{f}\right\rangle_{\boldsymbol{\theta}_{\mathcal{A}}}\ \geq\ \frac{\hat{\theta}_n}{\hat{\theta}_1+\hat{\theta}_n}.
\end{equation}
Based on this result, our goal is to present an optimal transition matrix which archives this lower bound \eqref{eq:lower_bound}. For simplicity, we construct the transition matrix such that
\begin{equation}
    \left\langle\boldsymbol{Z}\boldsymbol{f},\ \boldsymbol{f}\right\rangle_{\boldsymbol{\theta}_{\mathcal{A}}}\ =\ \frac{\hat{\theta}_n}{\hat{\theta}_1+\hat{\theta}_n},\ \ \boldsymbol{f}\in\mathcal{F},\ \|\boldsymbol{f}\|_{{\boldsymbol{\theta}_{\mathcal{A}}}}=1.
\end{equation}
One possible construction of irreversible stochastic transition matrix $\boldsymbol{P}_{\mathcal{A}}=\{\hat{P}_{ij}\}$ is
\begin{equation}
    \label{eq:irreversibiletransitionmatrix}
    \hat{P}_{ij} = \left\{
    \begin{array}{rl}
        0,&1\leq j<i<n,\\[7.5pt]
        \frac{\displaystyle \hat{\theta}_i-\hat{\theta}_1}{\displaystyle \hat{\theta}_i-\hat{\theta}_n},&1\leq j=i<n,\\[7.5pt]
        \frac{\displaystyle\hat{\theta}_j\big(\hat{\theta}_1+\hat{\theta}_n\big)}{\displaystyle\hat{\theta}_n\big(\hat{\theta}_j+\hat{\theta}_n\big)}\ {\displaystyle\prod_{k=1}^{j-1}}\ \frac{\displaystyle\hat{\theta}_n-\hat{\theta}_k}{\displaystyle\hat{\theta}_k+\hat{\theta}_n},&1\leq j<n=i,\\[7.5pt]
        \frac{\displaystyle2\hat{\theta}_j\big(\hat{\theta}_1+\hat{\theta}_n\big)}{\displaystyle\big(\hat{\theta}_i+\hat{\theta}_n\big)\big(\hat{\theta}_j+\hat{\theta}_n\big)}{\displaystyle\prod_{k=i+1}^{j-1}}\frac{\displaystyle\hat{\theta}_n-\hat{\theta}_k}{\displaystyle\hat{\theta}_k+\hat{\theta}_n},&1\leq i<j<n,\\[7.5pt]
        1-{\displaystyle\sum_{k=1}^{n-1}}\hat{P}_{ik},&\text{otherwise.} 
    \end{array}
    \right.
\end{equation}
The following corollary states that $\boldsymbol{\theta}_{\mathcal{A}}$ is the leading eigenvector of irreducible $\boldsymbol{P}_{\mathcal{A}}$ constructed by \eqref{eq:irreversibiletransitionmatrix}. 
\begin{restatable}{corollary}{irreversibilefinal}
    \label{coly:irreversibile_final}
    For the irreversible stochastic transition matrix $\boldsymbol{P}_{\mathcal{A}}=\{\hat{P}_{ij}\}$ defined by \eqref{eq:irreversibiletransitionmatrix}, it holds that
    \begin{enumerate}
        \item $\boldsymbol{P}_{\mathcal{A}}$ is irreducible.
        \vspace{0.15cm}
        \item $\hat{P}_{ij}\geq 0$ for any $i,\ j\in[n]$.
        \vspace{0.15cm}
        \item $\boldsymbol{\theta}_{\mathcal{A}}$ is the largest left eigenvector of $\boldsymbol{P}_{\mathcal{A}}$.
    \end{enumerate}
\end{restatable}
At last, we show that the irreversible stochastic transition matrix $\boldsymbol{P}_{\mathcal{A}}$ \eqref{eq:irreversibiletransitionmatrix} will attain its lower bound of the worst-case asymptotic variance. It indicates that $\boldsymbol{\theta}_{\mathcal{A}}$ will be the fixed point of the continuous operator games between $\mathcal{A}$ and $\mathcal{R}$ with irreversible stochastic transition matrix \eqref{eq:irreversibiletransitionmatrix} and the minimal worst-case asymptotic variance. 
\begin{restatable}{theorem}{irreversiblelowerbound}
    \label{thm:irreversible_bound}
    Given the stochastic transition matrix $\boldsymbol{P}_{\mathcal{A}}$ with its invariant distribution $\boldsymbol{\theta}_{\mathcal{A}}=\{\hat{\theta}_1,\ \hat{\theta}_2,\ \dots\ ,\hat{\theta}_n\}^\top$ by \eqref{eq:irreversibiletransitionmatrix}, and the asymptotic variance $\nu(\boldsymbol{f}, \boldsymbol{P})$ of the sample mean with respect to vector $\boldsymbol{f}\in\mathcal{F},\ \|\boldsymbol{f}\|_{\boldsymbol{\theta}_{\Scale[0.3]{\mathcal{A}}}}=1$ from the Markov chain with $\boldsymbol{P}$ is defined as \eqref{eq:asymptotic_var_ref}. Then we have $\boldsymbol{P}_{\mathcal{A}}$ is one of the optimal irreversible transition matrices which attains the following bound:
    \begin{equation}
        \underset{\boldsymbol{P}\phantom{\boldsymbol{P}^{a}}}{\textbf{\textit{inf}}}\ \underset{\boldsymbol{f}\in\mathcal{F},\|\boldsymbol{f}\|_{\boldsymbol{\theta}_{\Scale[0.3]{\mathcal{A}}}}=1}{\textbf{\textit{sup}}}\ \nu(\boldsymbol{f}, \boldsymbol{P}) = \frac{\hat{\theta}_n-\hat{\theta}_1}{\hat{\theta}_n+\hat{\theta}_1}.
    \end{equation}
\end{restatable}

Now we discuss the actions of $\mathcal{A}$ to eliminate the unknown data $\boldsymbol{w}^*_u$ in detail. Although the above results have pursued the minimization of worst-case asymptotic variance to resist the effect of $\boldsymbol{w}^*_u$, the adversary $\mathcal{A}$ could still seek the help of $\mathcal{R}$'s feedback $\boldsymbol{\theta}_{\mathcal{R}}$. Regardless $\boldsymbol{\theta}_{\mathcal{R}}$ is perfect or not, $\mathcal{A}$ reconstructs the stochastic transition matrix $\boldsymbol{P}_{\mathcal{R}}$ by $\boldsymbol{\theta}_{\mathcal{R}}$, estimates the unknown pairwise comparisons and induces the actual attack strategy. The general framework for $\mathcal{A}$ to construct the attack strategies is listed as Algorithm \ref{alg:attack_spectral}.
\begin{algorithm}
    \SetAlgoLined
    \SetKwInOut{Input}{Input}
    \SetKwInOut{Output}{Output}
    \Input{$\boldsymbol{w}^*_k,\ \boldsymbol{\theta}_{\mathcal{R}},\ \boldsymbol{\theta}_{\mathcal{A}},\ d_1,\ d_2$.}
    
    $\mathcal{A}$ recovers the transition matrix $\boldsymbol{P}_{\mathcal{R}}$ from the feedback $\boldsymbol{\theta}_{\mathcal{R}}$ by \eqref{eq:reversibile_transition_matrix} or \eqref{eq:irreversibiletransitionmatrix}
    \[
        \boldsymbol{P}_{\mathcal{R}} = \textbf{InverseEigen}(\boldsymbol{\theta}_{\mathcal{R}}).
    \]

    $\mathcal{A}$ constructs the transition matrix $\boldsymbol{P}_{\mathcal{A}}$ with the desired score $\boldsymbol{\theta}_{\mathcal{A}}$ by \eqref{eq:reversibile_transition_matrix} or \eqref{eq:irreversibiletransitionmatrix}
    \[
        \boldsymbol{P}_{\mathcal{A}} = \textbf{InverseEigen}(\boldsymbol{\theta}_{\mathcal{A}}).
    \]

    $\mathcal{A}$ estimates the unavailable data $\boldsymbol{w}_u$ by $\boldsymbol{P}_{\mathcal{R}}$ and $\boldsymbol{w}^*_k$
    \begin{equation*}
        \begin{aligned}
            \boldsymbol{w}_u = \textbf{Reconstruction}(\boldsymbol{P}_{\mathcal{R}},\ d_1) - \boldsymbol{w}^*_k,\\
            \boldsymbol{w}_k = \textbf{Reconstruction}(\boldsymbol{P}_{\mathcal{A}},\ d_2) - \boldsymbol{w}_u,
        \end{aligned}
    \end{equation*}
    where $d_1,\ d_2$ are two scale factors which acted as maximum out-degree $d_{\text{max}}$ in \eqref{eq:empirical_matrix}, and $\textbf{Reconstruction}$ is the inverse process of \eqref{eq:empirical_matrix}.

    $\mathcal{A}$ transfers $\boldsymbol{w}_k$ into integer:
    \[
        \boldsymbol{\hat{w}}_k=\textbf{Integer}(\boldsymbol{w}_k),
    \]

    $\mathcal{R}$ executes \textbf{RankCentrality} algorithm by solving \eqref{eq:detailed_balance} with the poisoned data $\boldsymbol{\hat{w}}$, 
    \[
        \boldsymbol{\hat{\theta}} = \textbf{RankCentrality}(\boldsymbol{\hat{w}})
    \]
    where
    \[
        \boldsymbol{\hat{w}} = \boldsymbol{\hat{w}}_k + \boldsymbol{w}^*_u.
    \]
    \Output{The misguided ranking list $\boldsymbol{\hat{\theta}}$.}
    \caption{Target Attack against Spectral Method}
    \label{alg:attack_spectral}
\end{algorithm}

\section{Experiments}
In this section, three examples are exhibited with both simulated and real-world data to illustrate the validity of the proposed attack strategies on \textbf{HodgeRank} and \textbf{RankCentrality}. The first example is with simulated data while the latter two exploit real-world datasets involved crowdsourcing and election.
\subsection{General Settings}
Let $n$ be the number of the candidates involved in the aggregation process of $\mathcal{R}$. The original data is $\{\boldsymbol{G},\boldsymbol{y},\boldsymbol{w}^*\}$. With the limitation of the space, we do not test all $n!$ possible permutations of $[n]$ but construct some target ranking lists of $\mathcal{A}$ as follows. Regardless of the completeness of information and the accuracy of feedback, the adversary $\mathcal{A}$ always gets the ranking list $\boldsymbol{\pi}_{\boldsymbol{\bar{\theta}}}$ aggregated by the victim $\mathcal{R}$ with $\boldsymbol{w}^*$. $\mathcal{A}$ chooses a candidate of $\boldsymbol{\pi}_{\boldsymbol{\bar{\theta}}}$ and generates
\begin{equation}
    \label{eq:target_constrcut}
    \boldsymbol{\pi}^t_{\mathcal{A}} = [\boldsymbol{\pi}_{\boldsymbol{\bar{\theta}}}(t);\ \boldsymbol{\pi}_{\boldsymbol{\bar{\theta}}}/\boldsymbol{\pi}_{\boldsymbol{\bar{\theta}}}(t)],\ t=2,\dots,n,
\end{equation}
where $\boldsymbol{\pi}(t)$ is the $t$-th candidate of $\boldsymbol{\pi}$, and $\boldsymbol{\pi}/\boldsymbol{\pi}(t)$ denotes an ordered sequence with $n-1$ elements which deletes $t$-th candidates from $\boldsymbol{\pi}$. Due to the significance of top-$1$ candidate in rank aggregation, constructing the target ranking in this way can not only reflects the purpose of the adversary, but also illustrates the applicability of the proposed attack methods. When attacking against \textbf{HodgeRank}, the original weight $\boldsymbol{w}^*$ and the comparison matrix $\boldsymbol{C}$ are generated as \eqref{eq:weight} and \eqref{eq:comparison_matrix}. The empirical transition matrix $\boldsymbol{P}^*$ is constructed as \eqref{eq:empirical_matrix} when the victim is \textbf{RankCentrality}. 

\vspace{0.25cm}
\noindent\textbf{Confrontation Scenarios. }To validate the vulnerability of ranking aggregation approaches like \textbf{HodgeRank} and \textbf{RankCentrality}, we establish different scenarios divided by the completeness of the adversary's knowledge and the perfection of the victim's feedback. We illustrate the attack results of the proposed methods in the three following cases:
\begin{itemize}[leftmargin=*]
    \item \textbf{Complete information and perfect feedback. }The adversary obtains the whole of the original data $\boldsymbol{w}^*$ and the perfect feedback $\boldsymbol{\bar{\theta}}$ of the victim. $\boldsymbol{\bar{\theta}}$ is the output of \textbf{HodgeRank} or \textbf{RankCentrality} by solving \eqref{opt:reg_hodge} or \eqref{eq:detailed_balance} with $\boldsymbol{w}^*$ respectively.
    \vspace{0.15cm}
    \item \textbf{Incomplete information and perfect feedback. }The adversary obtains a part of the original data $\boldsymbol{w}^*$ and the perfect feedback $\boldsymbol{\bar{\theta}}$ of the victim. The original data $\boldsymbol{w}^*$ has two parts as \eqref{eq:decomposition}: $\boldsymbol{w}^*_k$ and $\boldsymbol{w}^*_u$. Only the former part is accessible to the adversary. 
    \vspace{0.15cm}
    \item \textbf{Complete information and imperfect feedback. }The adversary obtains the whole of the original data $\boldsymbol{w}^*$ and the imperfect feedback $\boldsymbol{\pi}_{\boldsymbol{\bar{\theta}}}$ of the victim. $\boldsymbol{\pi}_{\boldsymbol{\bar{\theta}}}$ is the ranking list obtained from $\boldsymbol{w}^*$ by \textbf{HodgeRank} or \textbf{RankCentrality}.
    \vspace{0.15cm}
\end{itemize}
\vspace{0.25cm}
\begin{figure*}[t!]
    \centering
    \begin{subfigure}[h]{0.16\textwidth}
        \centering
        \includegraphics[width=\textwidth]{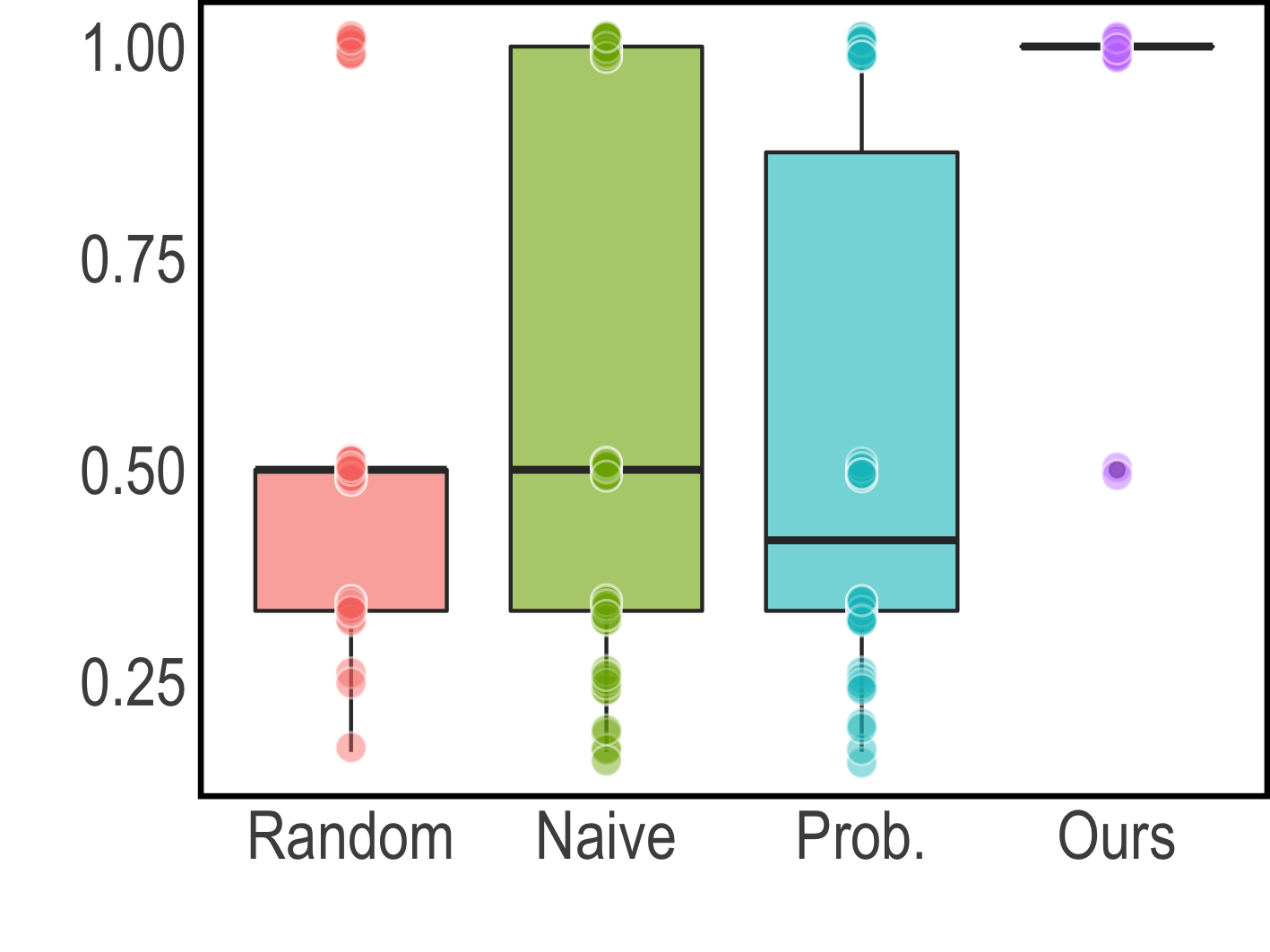}
        \vspace{-15pt}
        \caption{}
        \label{fig:hodge_rrank_cp}
    \end{subfigure}
    \hfill
    \begin{subfigure}[h]{0.16\textwidth}
        \centering
        \includegraphics[width=\textwidth]{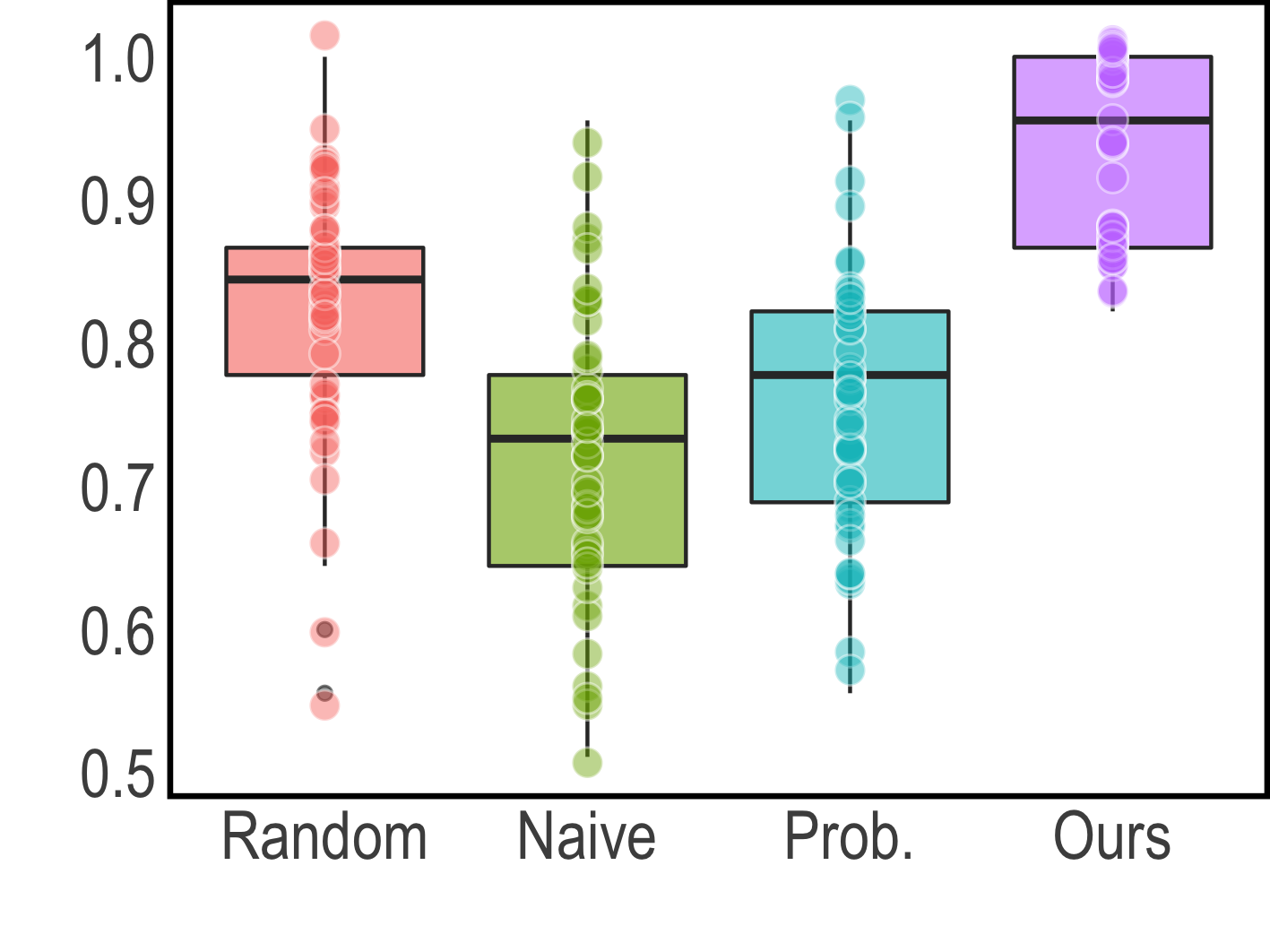}
        \vspace{-15pt}       
        \caption{}
        \label{fig:hodge_kendall_cp}
    \end{subfigure}
    \hfill
    \begin{subfigure}[h]{0.16\textwidth}
        \centering
        \includegraphics[width=\textwidth]{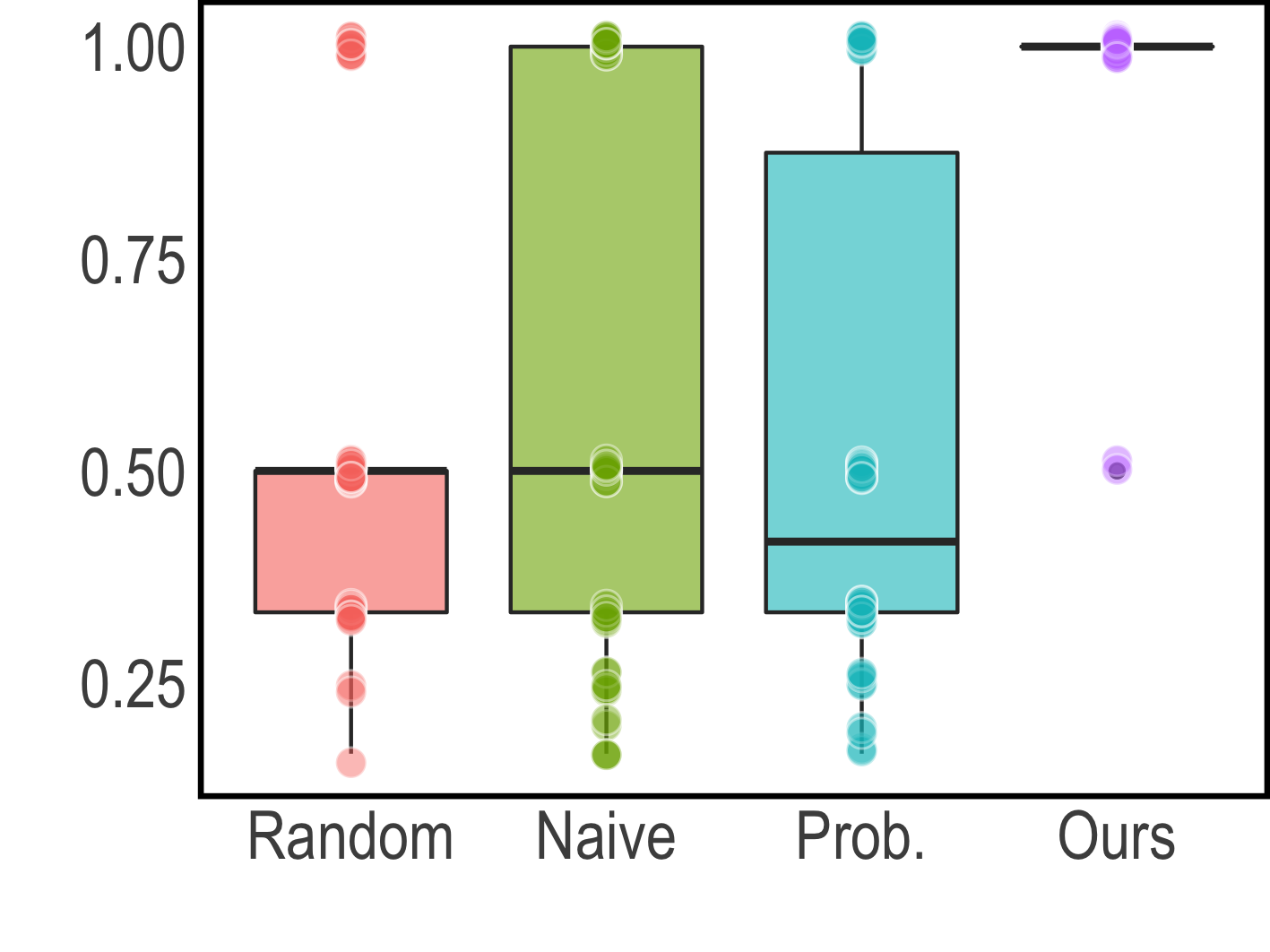}
        \vspace{-15pt}
        \caption{}
        \label{fig:hodge_rrank_ci}
    \end{subfigure}
    \hfill
    \begin{subfigure}[h]{0.16\textwidth}
        \centering
        \includegraphics[width=\textwidth]{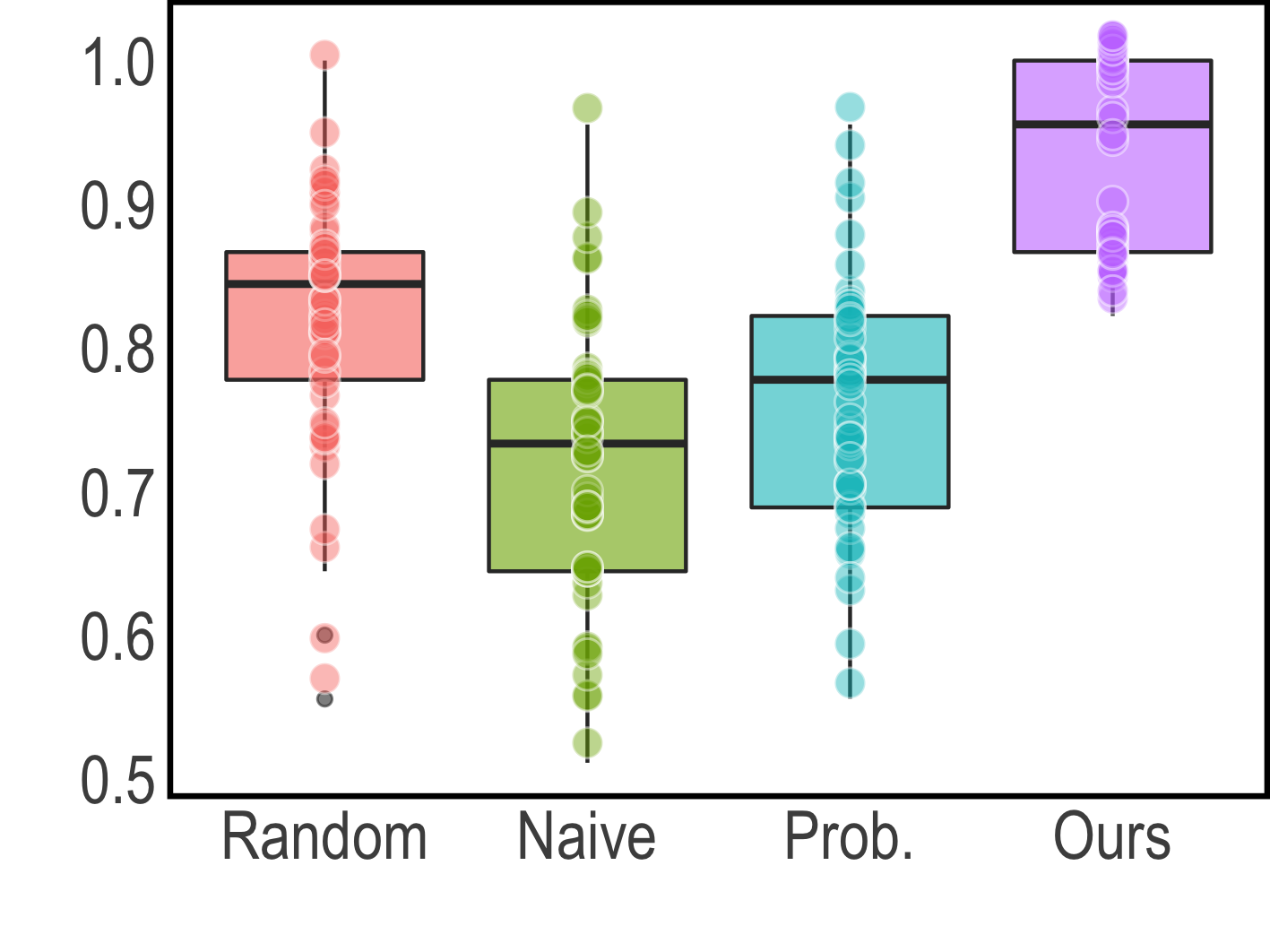}
        \vspace{-15pt}
        \caption{}
        \label{fig:hodge_kendall_ci}
    \end{subfigure}
    \hfill
    \begin{subfigure}[h]{0.16\textwidth}
        \centering
        \includegraphics[width=\textwidth]{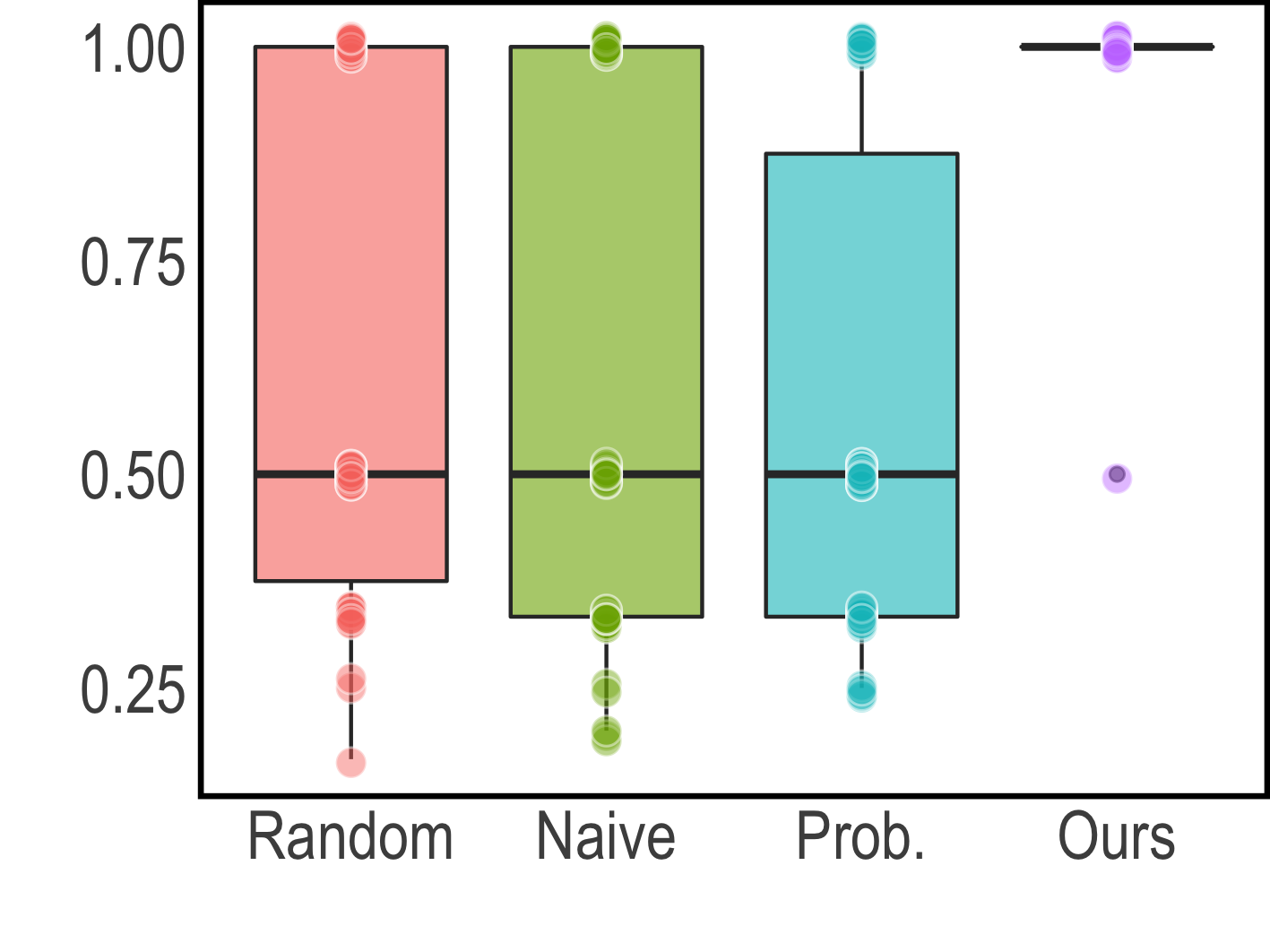}
        \vspace{-15pt}
        \caption{}
        \label{fig:hodge_rrank_ip}
    \end{subfigure}
    \hfill
    \begin{subfigure}[h]{0.16\textwidth}
        \centering
        \includegraphics[width=\textwidth]{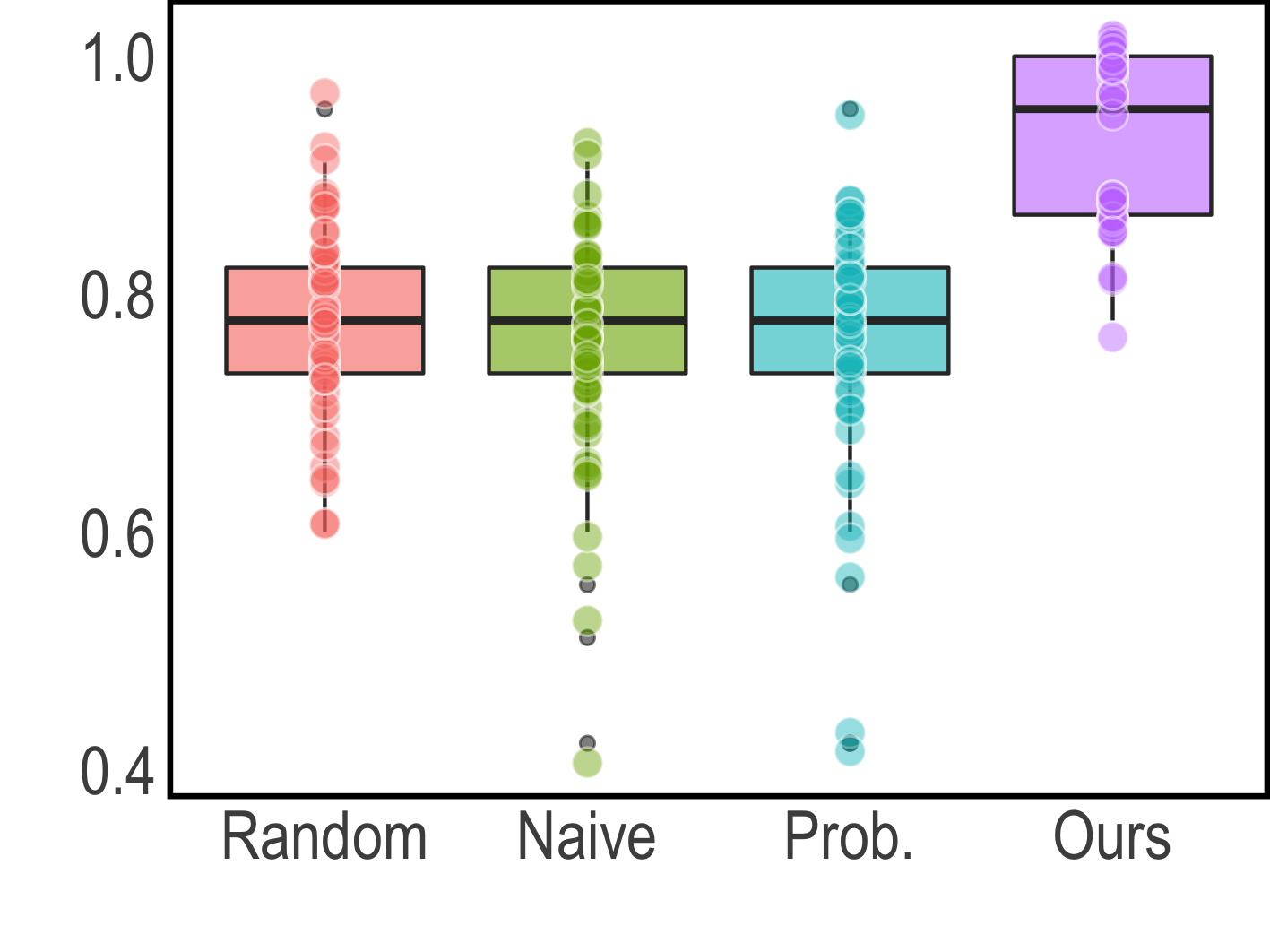}
        \vspace{-15pt}
        \caption{}
        \label{fig:hodge_kendall_ip}
    \end{subfigure}
    \begin{subfigure}[h]{0.16\textwidth}
        \centering
        \includegraphics[width=\textwidth]{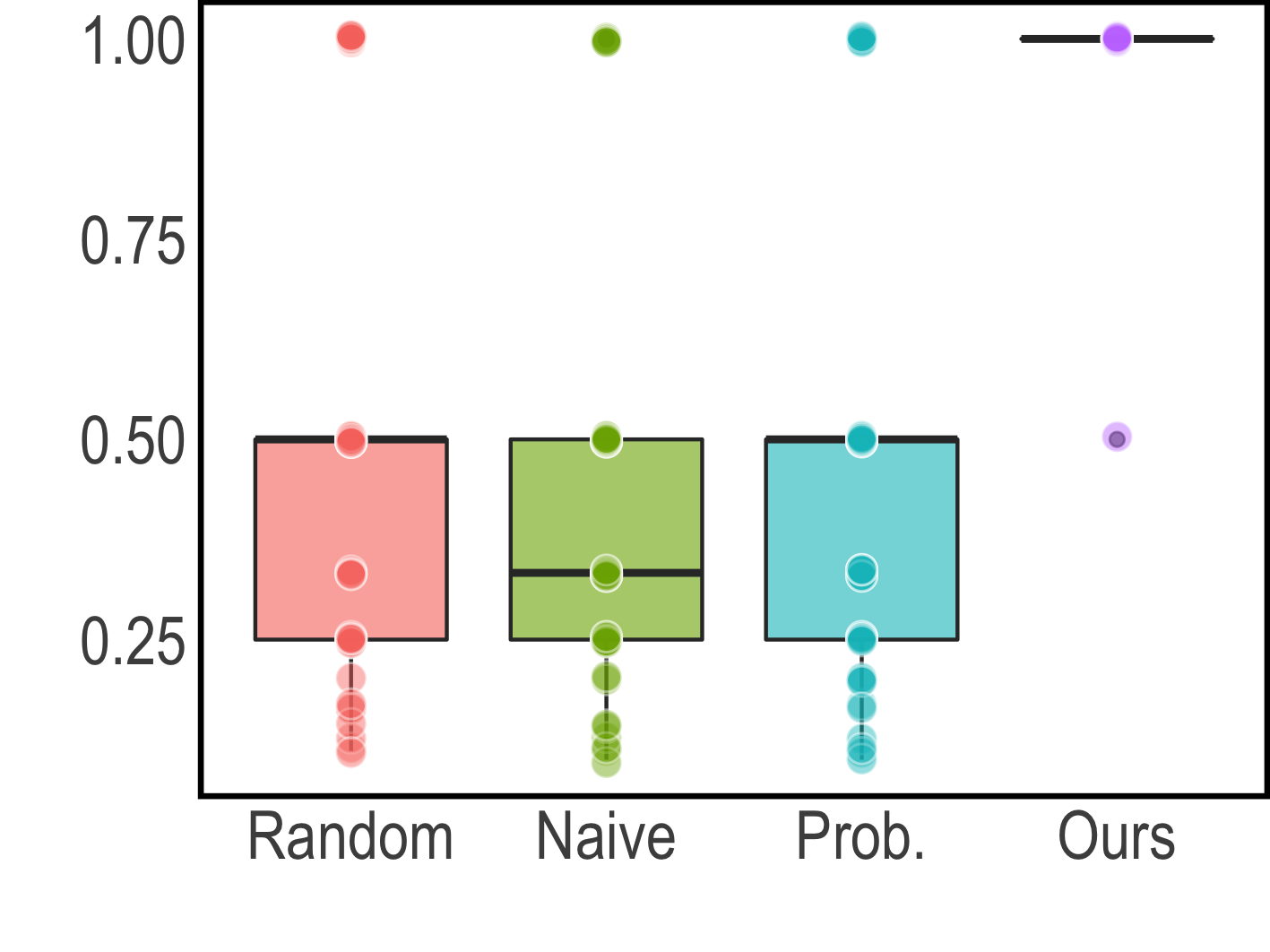}
        \vspace{-15pt}
        \caption{}
        \label{fig:spr_rrank_cp}
    \end{subfigure}
    \hfill
    \begin{subfigure}[h]{0.16\textwidth}
        \centering
        \includegraphics[width=\textwidth]{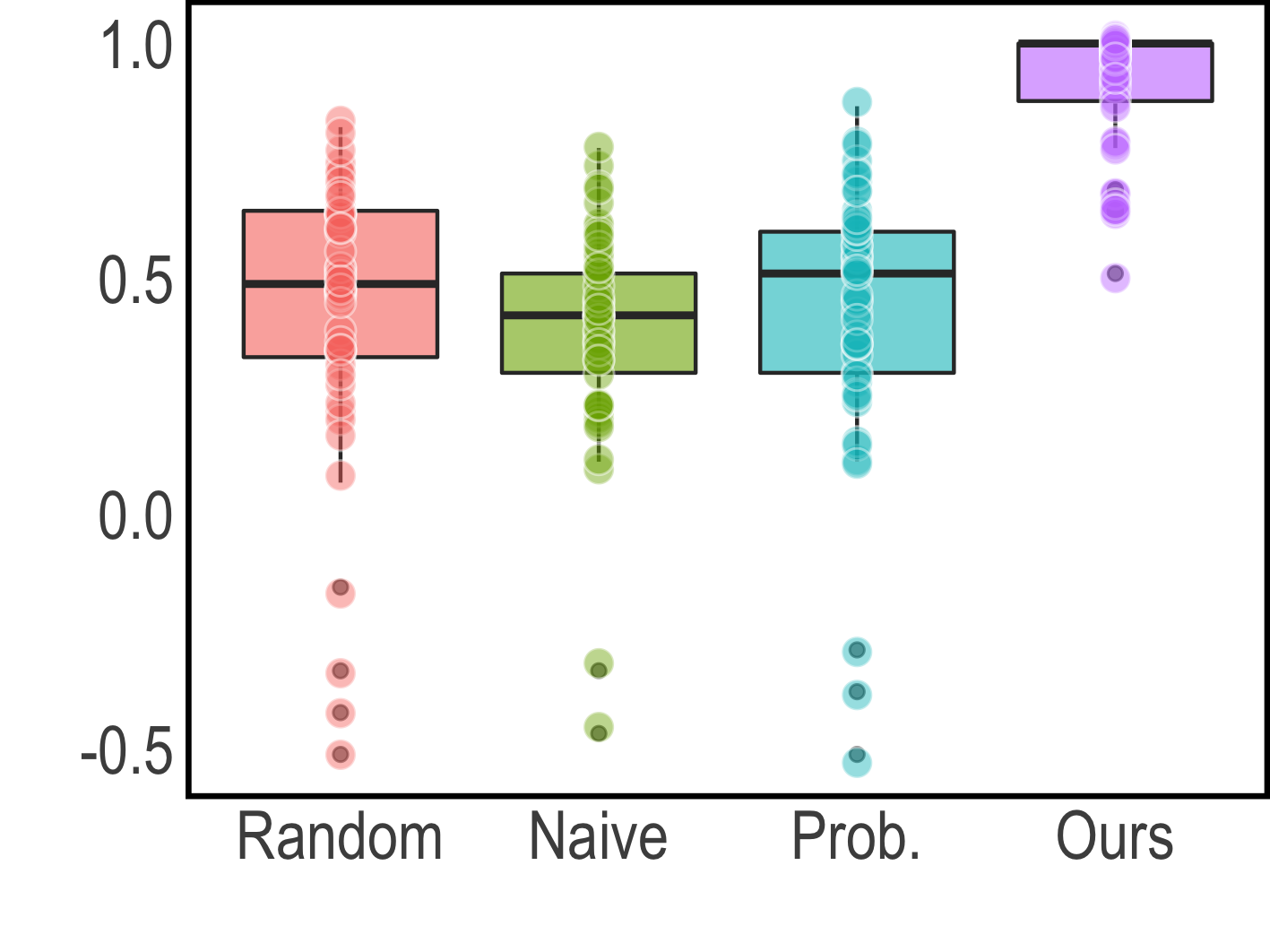}
        \vspace{-15pt}
        \caption{}
        \label{fig:spr_kendall_cp}
    \end{subfigure}
    \hfill
    \begin{subfigure}[h]{0.16\textwidth}
        \centering
        \includegraphics[width=\textwidth]{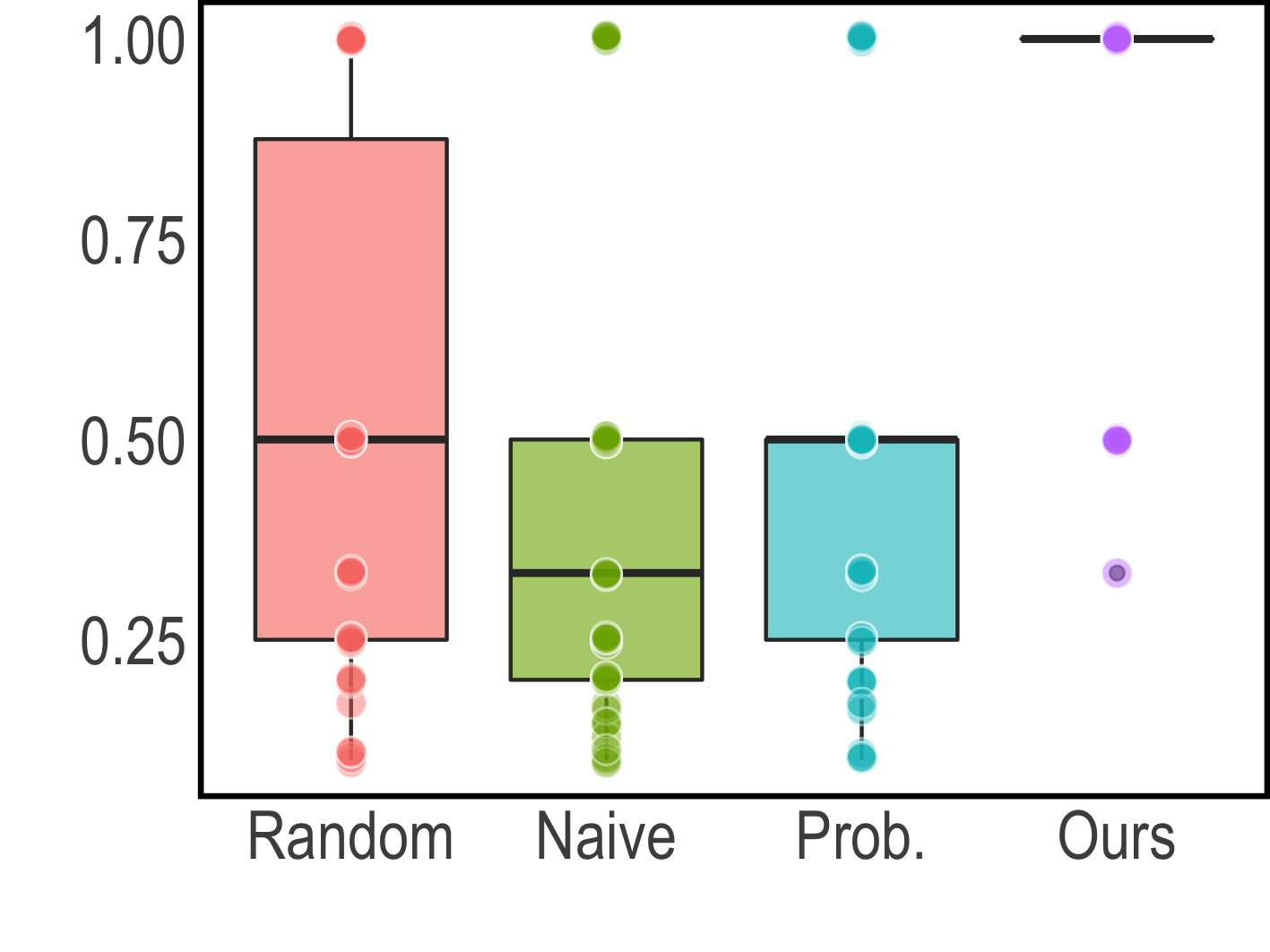}
        \vspace{-15pt}
        \caption{}
        \label{fig:spr_rrank_ci}
    \end{subfigure}
    \hfill
    \begin{subfigure}[h]{0.16\textwidth}
        \centering
        \includegraphics[width=\textwidth]{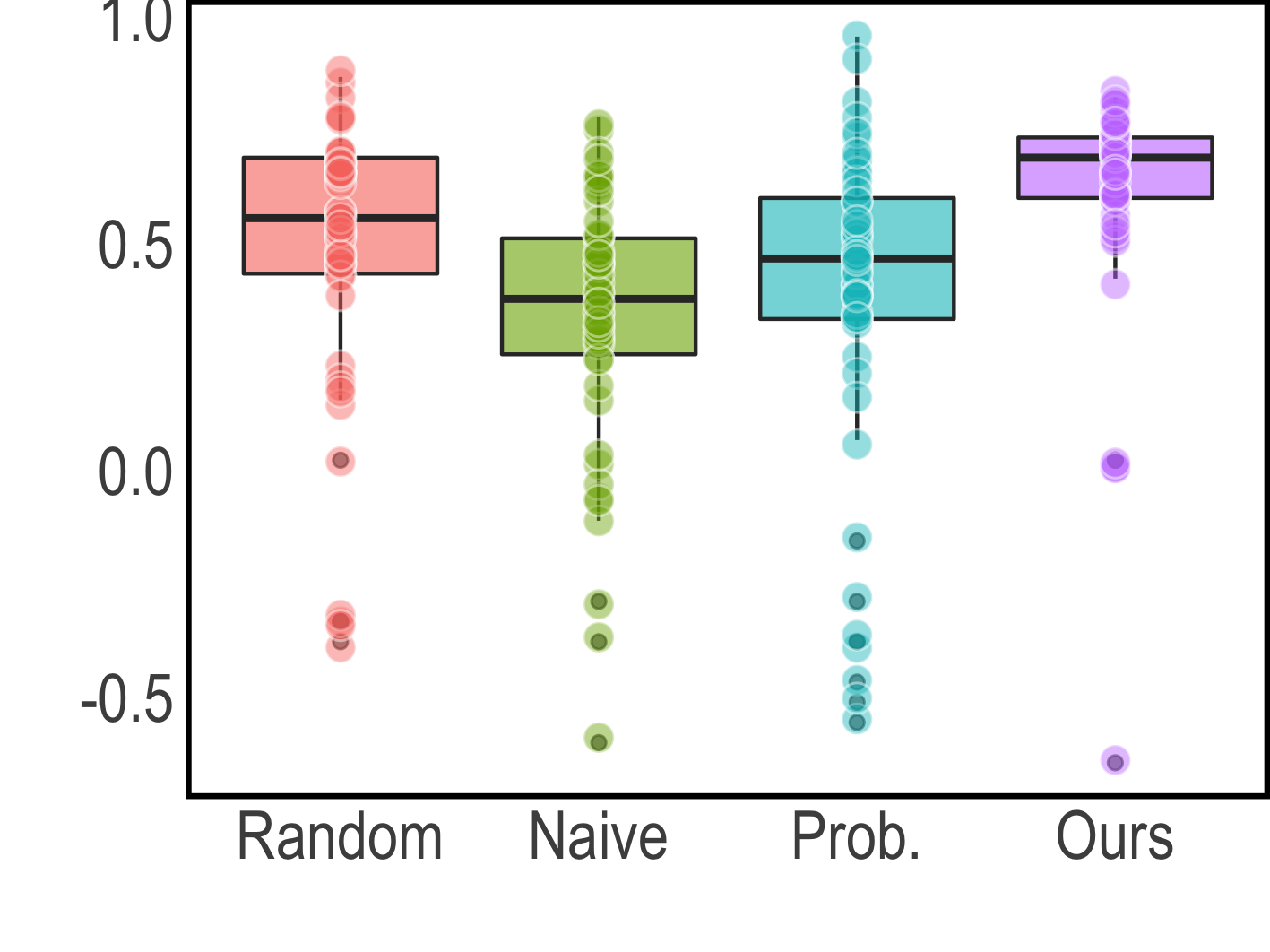}
        \vspace{-15pt}
        \caption{}
        \label{fig:spr_kendall_ci}
    \end{subfigure}
    \hfill
    \begin{subfigure}[h]{0.16\textwidth}
        \centering
        \includegraphics[width=\textwidth]{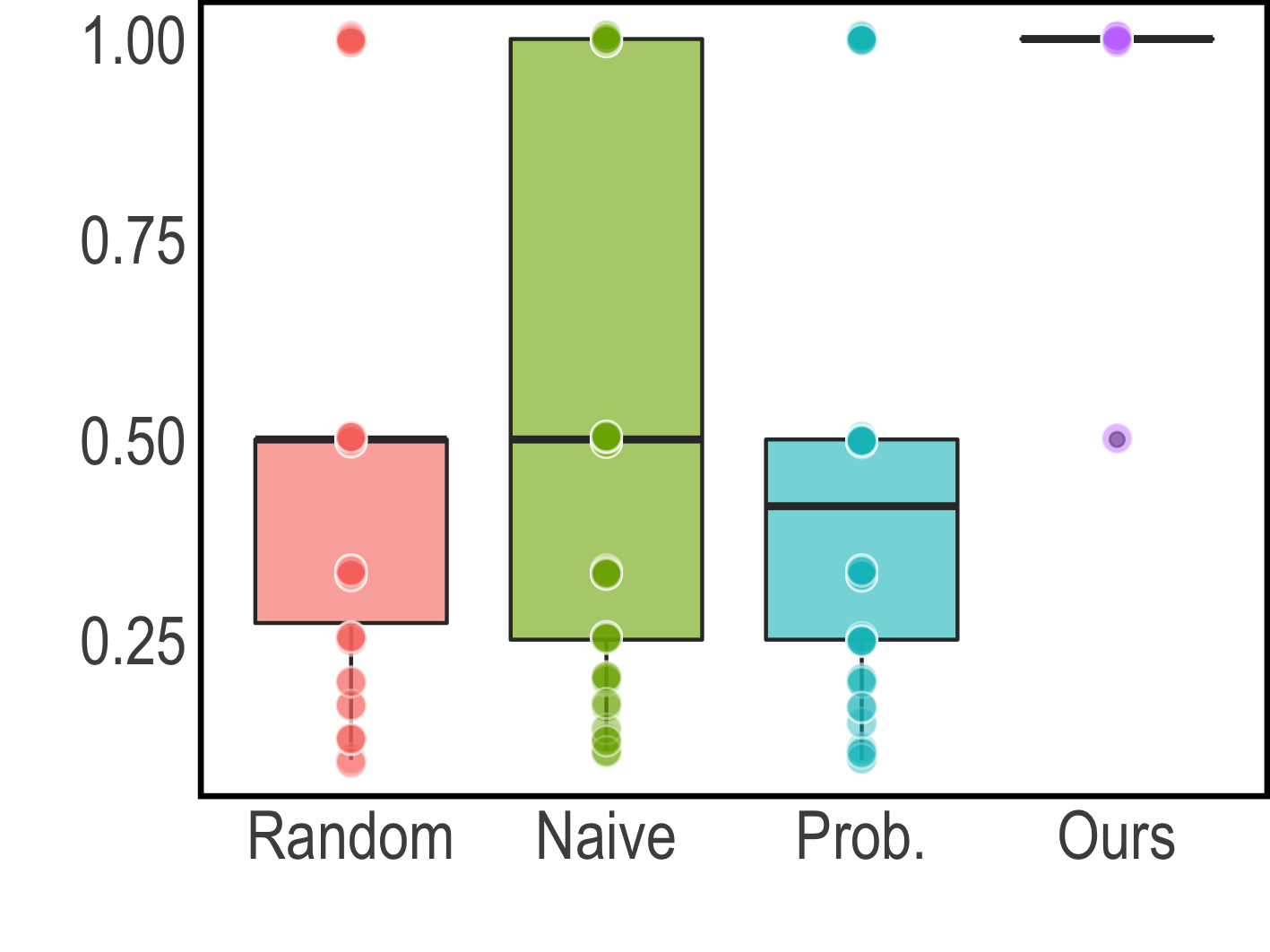}
        \vspace{-15pt}
        \caption{}
        \label{fig:spr_rrank_ip}
    \end{subfigure}
    \hfill
    \begin{subfigure}[h]{0.16\textwidth}
        \centering
        \includegraphics[width=\textwidth]{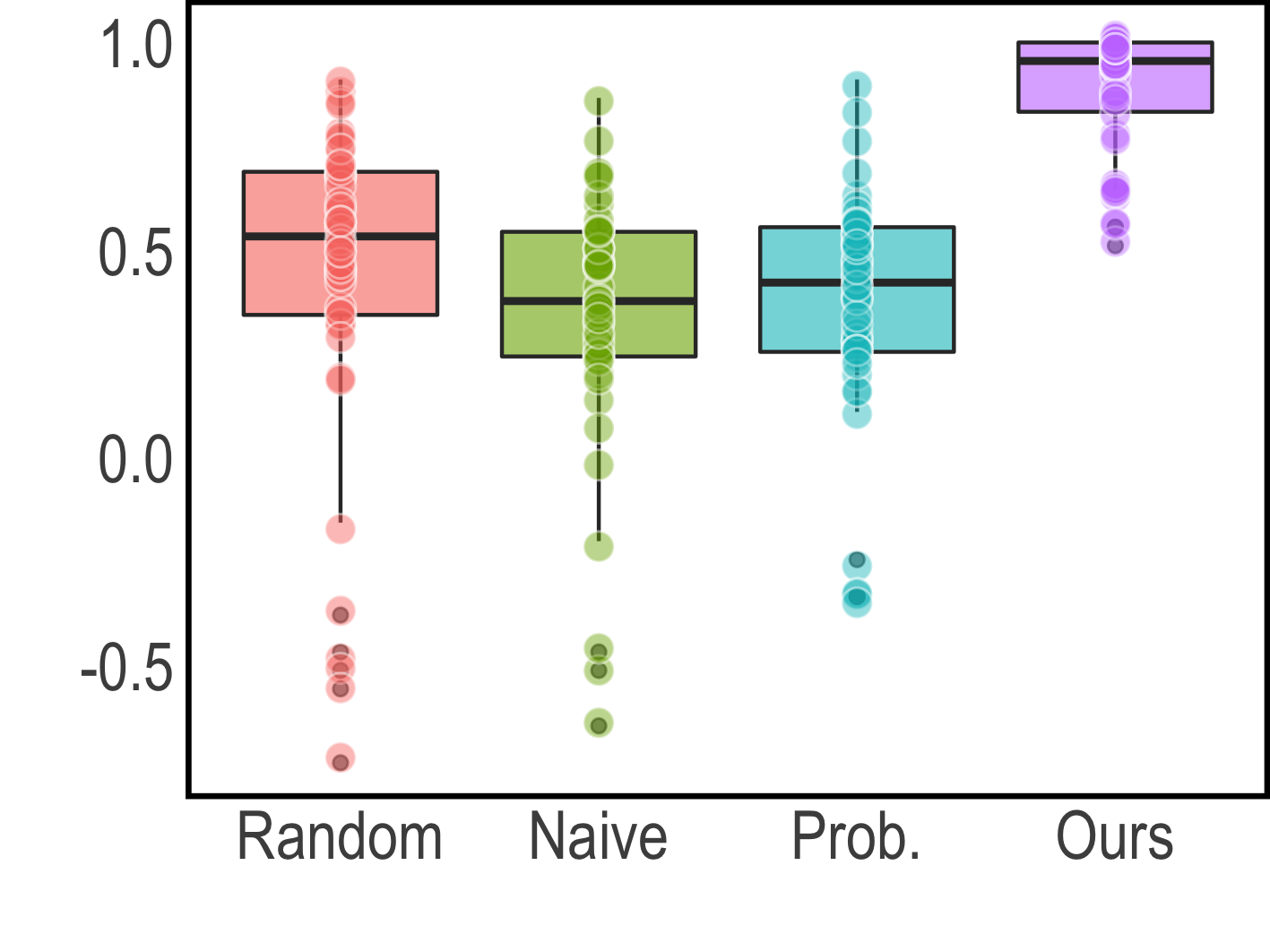}
        \vspace{-15pt}
        \caption{}
        \label{fig:spr_kendall_ip}
    \end{subfigure}
    \begin{subfigure}[h]{0.16\textwidth}
        \centering
        \includegraphics[width=\textwidth]{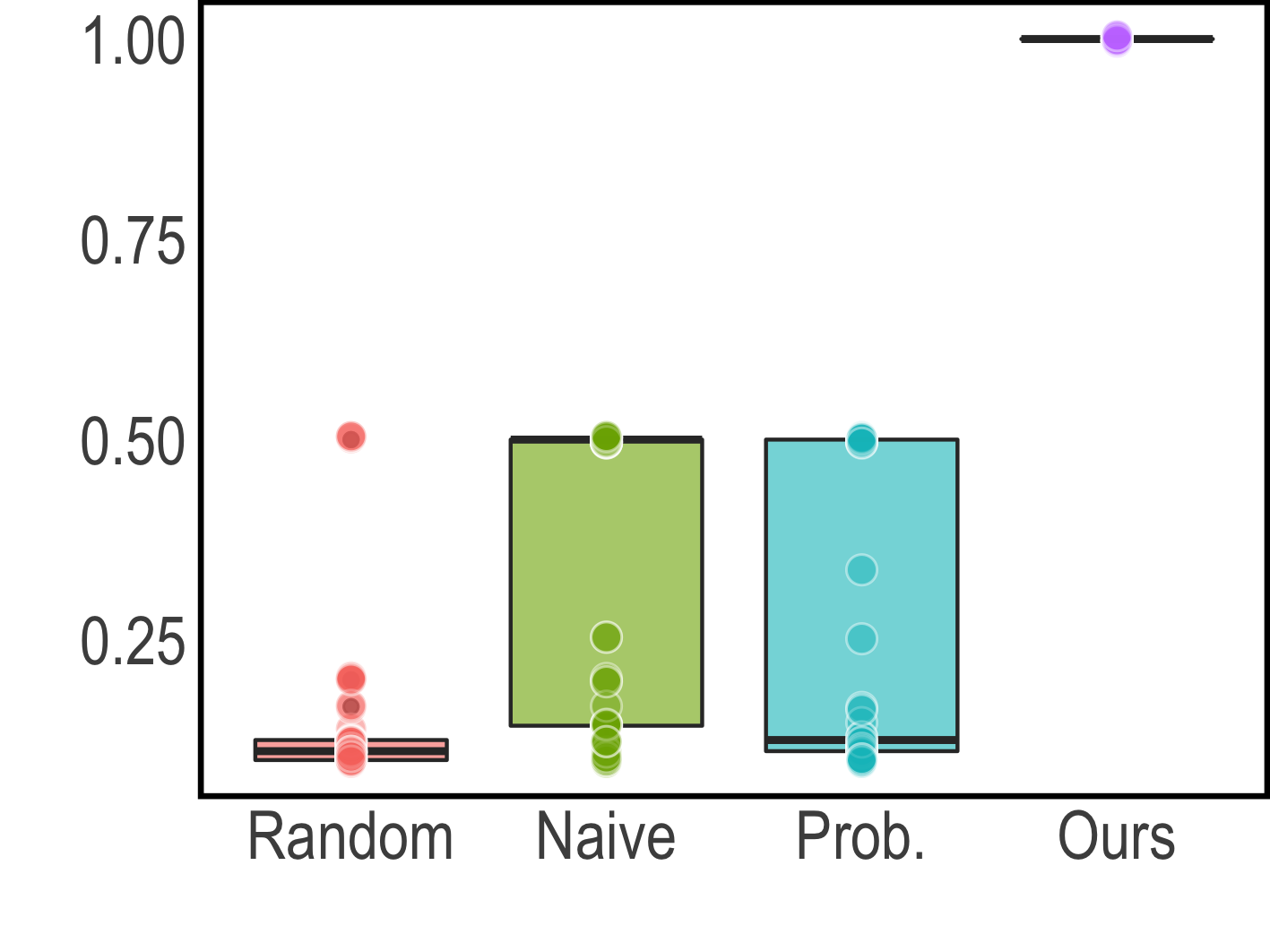}
        \vspace{-15pt}
        \caption{}
        \label{fig:spi_rrank_cp}
    \end{subfigure}
    \hfill
    \begin{subfigure}[h]{0.16\textwidth}
        \centering
        \includegraphics[width=\textwidth]{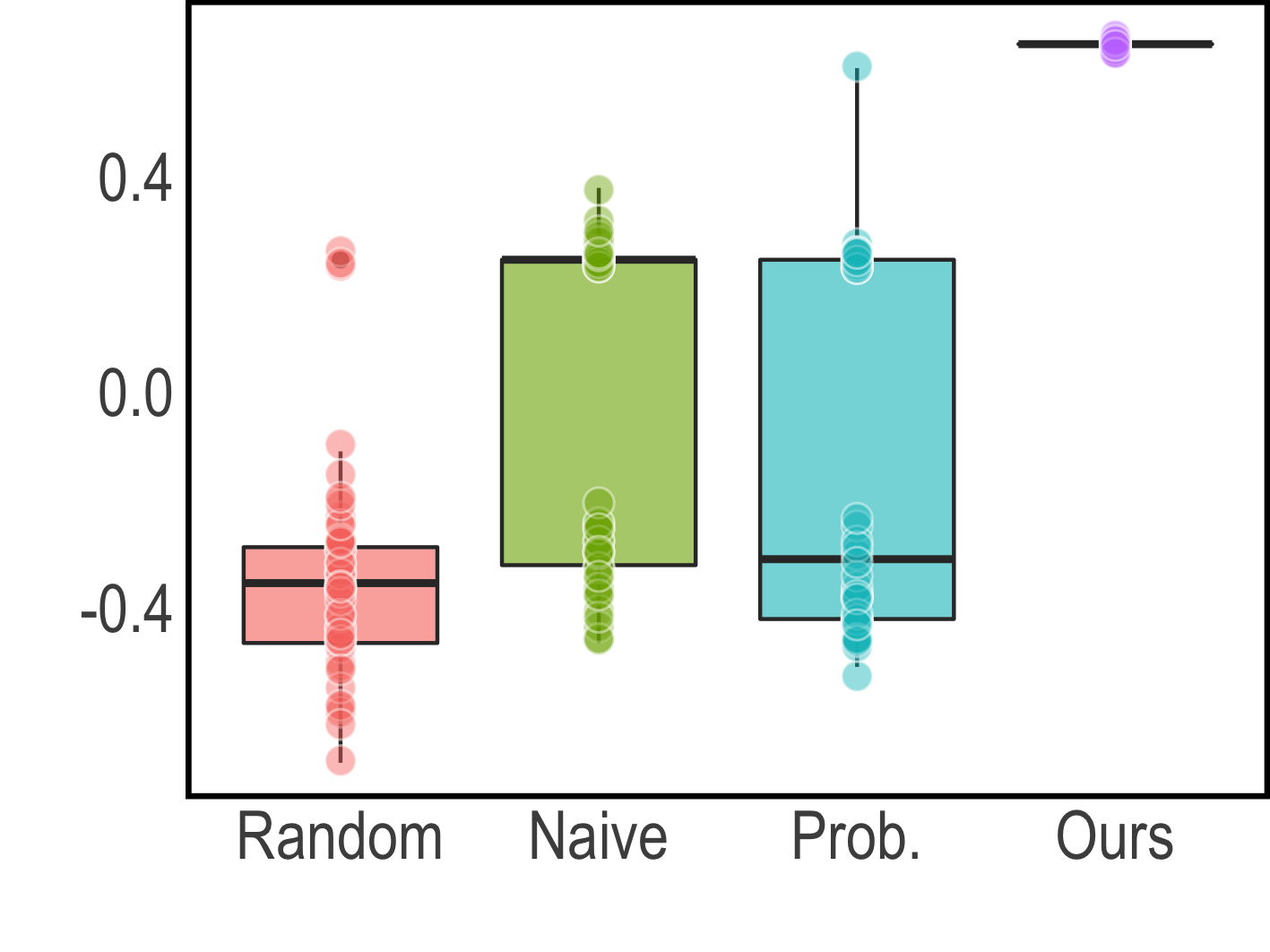}
        \vspace{-15pt}
        \caption{}
        \label{fig:spi_kendall_cp}
    \end{subfigure}
    \hfill
    \begin{subfigure}[h]{0.16\textwidth}
        \centering
        \textcolor{red}{\frame{\includegraphics[width=\textwidth]{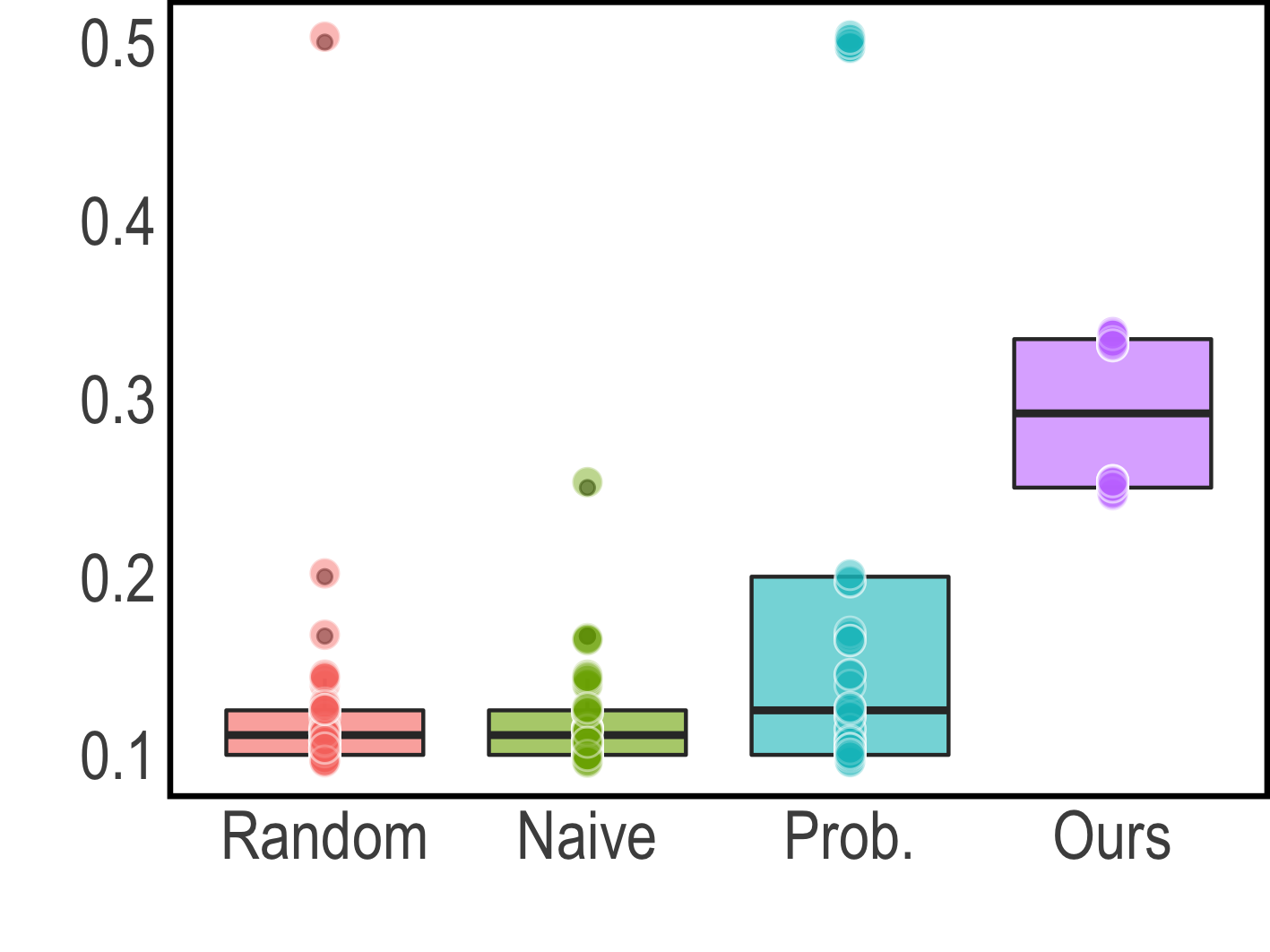}}}
        \vspace{-15pt}
        \caption{}
        \label{fig:spi_rrank_ci}
    \end{subfigure}
    \hfill
    \begin{subfigure}[h]{0.16\textwidth}
        \centering
        \textcolor{red}{\frame{\includegraphics[width=\textwidth]{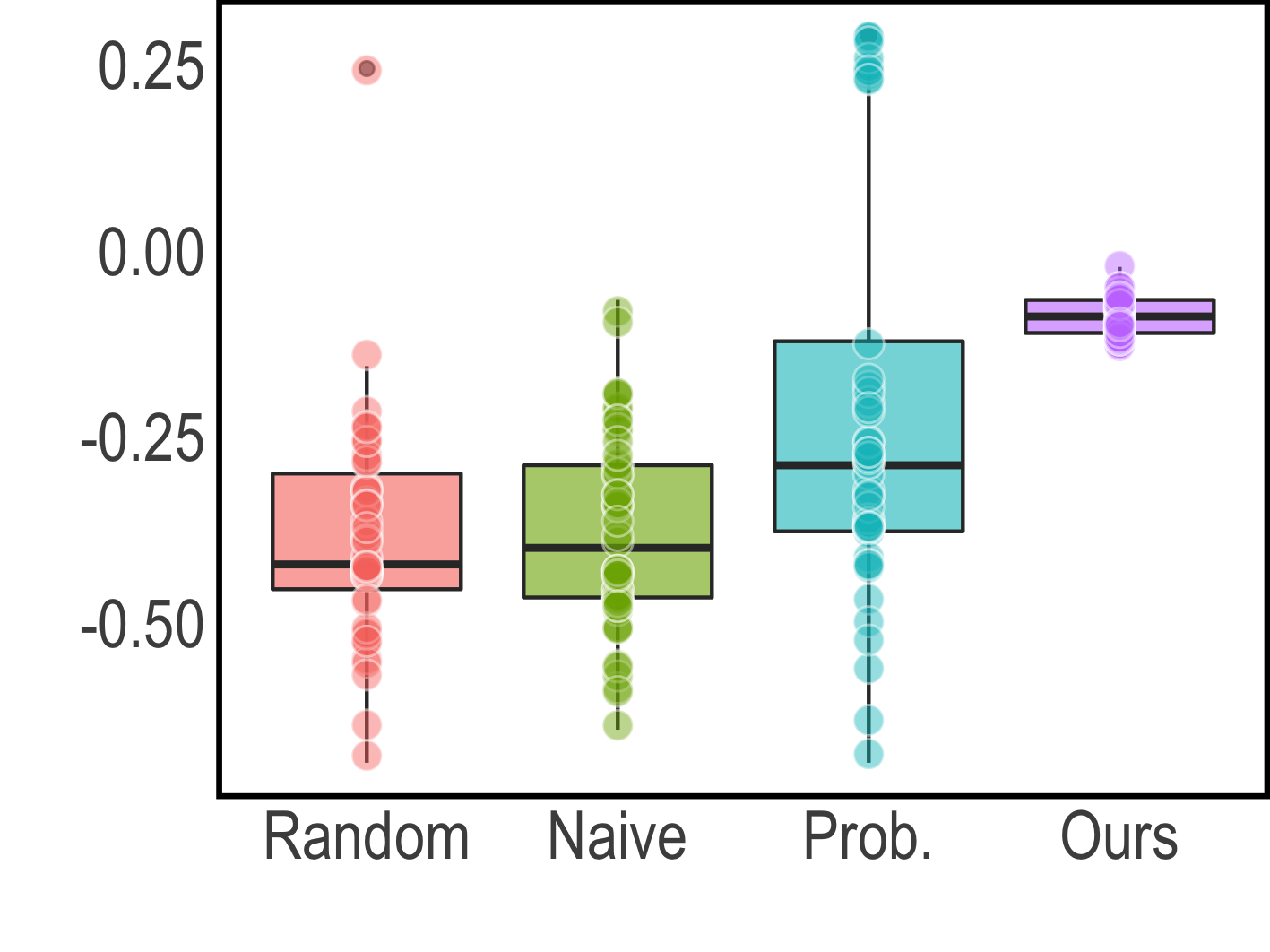}}}
        \vspace{-15pt}
        \caption{}
        \label{fig:spi_kendall_ci}
    \end{subfigure}
    \hfill
    \begin{subfigure}[h]{0.16\textwidth}
        \centering
        \includegraphics[width=\textwidth]{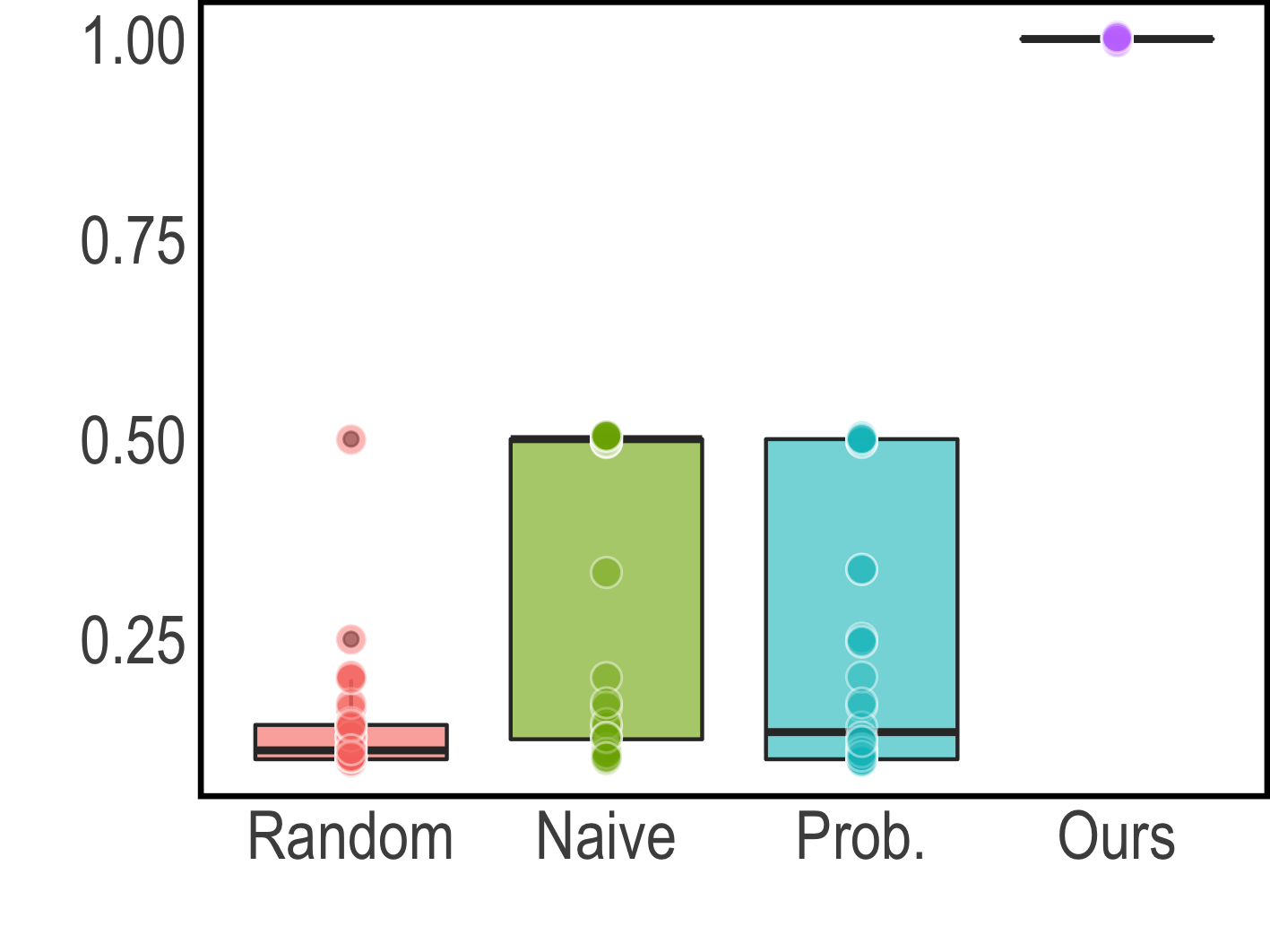}
        \vspace{-15pt}
        \caption{}
        \label{fig:spi_rrank_ip}
    \end{subfigure}
    \hfill
    \begin{subfigure}[h]{0.16\textwidth}
        \centering
        \includegraphics[width=\textwidth]{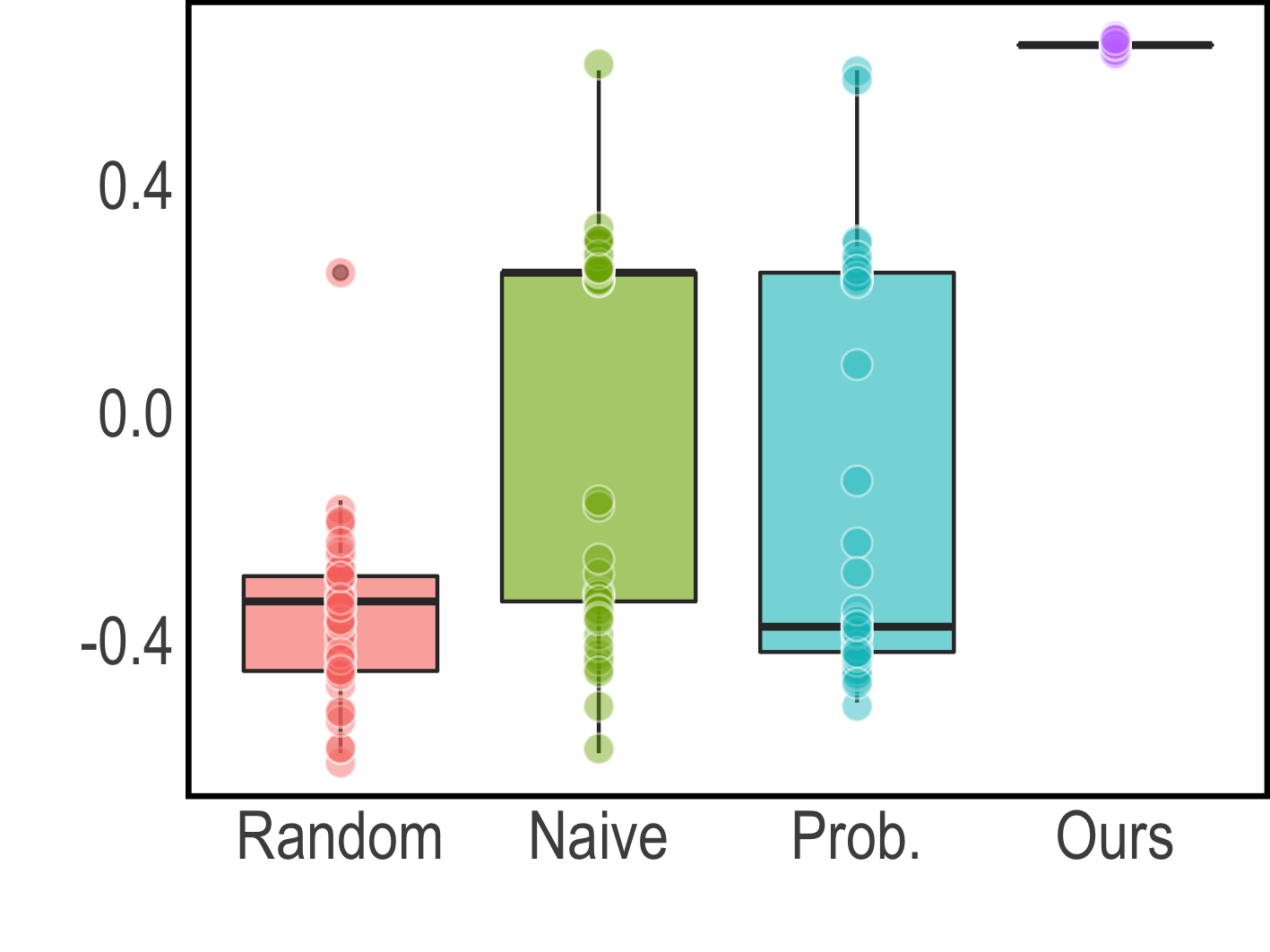}
        \vspace{-15pt}
        \caption{}
        \label{fig:spi_kendall_ip}
    \end{subfigure}
    \caption{Comparative results of different attack methods against \textbf{HodgeRank} and \textbf{RankCentrality} on simulated data. The target ranking is $\boldsymbol{\pi}^2_{\mathcal{A}}$. The box plot illustrates the results of $50$ trials with different original data. The proposed methods will get the \textbf{R-Rank} to be $1$ and \textbf{Kendall}-$\tau$ close to $1$. The first row (a-f) displays the results of \textbf{HodgeRank}. The second row (g-l) is \textbf{RankCentrality} with the \textbf{reversible} stochastic transition matrix. The last row (m-r) represents \textbf{RankCentrality} with \textbf{irreversible} stochastic transition matrix. The odd columns (column $1$, $3$, $5$) show the \textbf{R-Rank} values and the even columns (column $2$, $4$, $6$) are the results of \textbf{Kendall}-$\tau$. The first two columns exhibit the confrontation scenario with \textbf{complete information and perfect feedback}. The second two columns reveal the results of adversary with \textbf{complete information and imperfect feedback}. The third two columns indicate the \textbf{incomplete information and perfect feedback}. The red frames (o \& p) are failure cases where the proposed methods could not make $\boldsymbol{\pi}_{\boldsymbol{\hat{\theta}}}(1) = \boldsymbol{\pi}^2_{\mathcal{A}}(1) = \boldsymbol{\pi}_{\boldsymbol{\bar{\theta}}}(2)$.}\label{fig:simulate}
\end{figure*}
\noindent\textbf{Competitors. }In the interest of fairness and evaluating the performance of different target attack strategies, we define the Incremental Pairwise Comparison Oracle (\textbf{IPCO}) complexity model. At each step $s=0,\ 1,\ 2,\ \dots$, an attacker $\mathcal{A}$ picks a pairwise comparison $(i, j)\in\boldsymbol{V}\times\boldsymbol{V},\ \forall\ i,\ j\in\boldsymbol{V},\ i\neq j$, and the oracle $\mathcal{O}_{\mathcal{A}}$ returns a variable $z^s_{ij}\in\{-1, 0, 1\}$, which decides the operation of the adversary. Notice that different attack strategies correspond to different oracles. Let $\boldsymbol{\hat{w}}_s=[\hat{w}^s_{12}, \hat{w}^s_{13},\dots,\hat{w}^s_{n,n-1}]$ be the modified pairwise comparisons at $s$ step and $\boldsymbol{\hat{w}}_0=\boldsymbol{w}^*$. The strategy profile $\boldsymbol{\hat{w}}_s$ corresponds with $z^s_{ij}$ like:
\begin{equation*}
    \hat{w}^s_{ij} = \left\{
    \begin{array}{ll}
        \hat{w}^{s-1}_{ij} + 1, &\text{if}\ z^s_{ij} = 1,\\[7.5pt]
        \hat{w}^{s-1}_{ij} - 1, &\text{if}\ z^s_{ij} = -1\ \text{and}\ \hat{w}^{s-1}_{ij}\geq 1,\\[7.5pt]
        \hat{w}^{s-1}_{ij},     &\text{if}\ z^s_{ij} = -1\ \text{and}\ \hat{w}^{s-1}_{ij}= 0;\ \text{or}\ z^s_{ij} = 0.
    \end{array}
    \right.
\end{equation*}
Let $S$ be the maximum step of accessing to the oracle $\mathcal{O}_{\mathcal{A}}$, and the adversary $\mathcal{A}$ chooses one profile $\boldsymbol{\hat{w}}_s$ in $\mathcal{W}_S=\{\boldsymbol{\hat{w}}_1,\dots,\boldsymbol{\hat{w}}_S\}$ as the attack strategy which will send to the victim $\mathcal{R}$ as its input. When the attack strategy archives the purpose of the adversary, say that the perturbed latent preference score $\boldsymbol{\hat{\theta}}\in\mathcal{C}_{\mathcal{A}}$ is defined as \eqref{eq:constraint_set_2} where $i_0\equiv t$, \textbf{IPCO} complexity which is defined as the minimal number of interactions between $\mathcal{A}$ and $\mathcal{O}_{\mathcal{A}}$ indicates the efficiency of the corresponding strategy/oracle. If $\mathcal{A}_1$ has a lower \textbf{IPCO} complexity (smaller $s$) than that of $\mathcal{A}_2$ against the same victim with the same purpose, we say that $\mathcal{A}_1$ is more efficient than $\mathcal{A}_2$. To the best of our knowledge, the proposed methods are the first overture to attack strategies on pairwise ranking algorithms with specified purposes. We compare the proposed methods with the following three competitors: the random perturbation strategy (referred to as `\textbf{\textit{Rand.}}'), the naive purposeful strategy (referred to as `\textbf{\textit{Naive}}') and the probabilistic model based target strategy (referred to as `\textbf{\textit{Prob.}}'). For the interest of fairness, we compare different attack profiles with the same number of \textbf{IPCO} interactions. The interaction of our proposed methods is defined as
\begin{equation}
    \Delta = \|\boldsymbol{w}^*_k-\boldsymbol{\hat{w}}_k\|_1 = \sum_{(i,j)\in\boldsymbol{E}} |w^*_{ij}-\hat{w}_{ij}|,
\end{equation}
and the modified data with the same $\Delta$ of each attack strategy is delivered to $\mathcal{R}$.
\begin{itemize}[leftmargin=*]
    \item \textbf{\textit{Random}} perturbation involves a completely random oracle which generates a random sequence as the operation indicators $\{z^s_{ij}\}$ without any prior knowledge. This method does not rely on any information of the desired ranking $\boldsymbol{\pi}^t_\mathcal{A}$ or latent preference score $\boldsymbol{\theta}_{\mathcal{A}}$. We conjecture that the purposeless behavior of the random perturbation would not sculpt the desired results out of the poisoned data $\boldsymbol{\hat{w}}$. However, this strategy is still the evidence to prove the importance of $\boldsymbol{\pi}^t_\mathcal{A}$ and $\boldsymbol{\theta}_{\mathcal{A}}$ in the attack operations with special purpose.
    \vspace{0.15cm}
    \item \textbf{\textit{Naive}} strategy implements its oracle with the so-called ``\textbf{\textit{Matthew Principle}}'': increasing the number of the pairwise comparisons which are consistent with the desired ranking list $\boldsymbol{\pi}^t_{\mathcal{A}}$ and decreasing the number of inconsistent comparisons. The operation variables of this oracle are produced as follows: for all $0<s\leq S$, $(i, j)\in\boldsymbol{V}\times\boldsymbol{V},\ \forall\ i,\ j\in\boldsymbol{V},\ i\neq j$,
    \begin{equation}
        z^s_{ij} = \left\{
        \begin{array}{rl}
             1, & \text{if}\ i\succ_{\boldsymbol{\pi}_{\mathcal{A}}} j,\\[5pt]
            -1, & \text{otherwise,}\\
        \end{array}
        \right.
    \end{equation}
    where $i\succ_{\boldsymbol{\pi}^t_{\mathcal{A}}} j$ represents that candidate $i$ has priority over candidate $j$ in the target ranking list $\boldsymbol{\pi}^t_{\mathcal{A}}$. Obviously, this strategy/oracle only focuses the partial orders of any paired candidates and ignores the global nature of the target ranking $\boldsymbol{\pi}^t_{\mathcal{A}}$. We speculate that the order of query would impress the efficiency of this strategy seriously. 
    \vspace{0.15cm}
    \item \textbf{\textit{Probabilistic}} approach establishes its oracle with the \textbf{BTL} model \eqref{eq:BTL} and the target preference score $\boldsymbol{\theta}_{\mathcal{A}}=\{\hat{\theta}_1, \hat{\theta}_2, \dots, \hat{\theta}_n\}$. If the probability that the pairwise comparison $(i, j)$ turns out true in $\boldsymbol{\pi}^t_{\mathcal{A}}$ is larger than a given threshold $\rho$, the oracle will return a positive variable of operation. Mathematically, for all $0<s\leq T$, the operation variables $\{z^{s}_{ij}\}$ are given as 
    \begin{equation}
        z^s_{ij} = \left\{
        \begin{array}{rl}
             1, & \text{if}\ \frac{\hat{\theta}_i}{\hat{\theta}_i+\hat{\theta}_j}>\rho,\\[7.5pt]
             0, & \text{if}\ \frac{\hat{\theta}_i}{\hat{\theta}_i+\hat{\theta}_j}=\rho,\\[7.5pt]
            -1, & \text{if}\ \frac{\hat{\theta}_i}{\hat{\theta}_i+\hat{\theta}_j}<\rho,
        \end{array}
        \right.
    \end{equation}
    where the hyper-parameter $\rho\geq 0.5$ represents the belief in the \textbf{BTL} model. Moreover, the probabilistic feedbacks of this oracle reflect the global character of the desired ranking $\boldsymbol{\pi}^t_{\mathcal{A}}$. 
\end{itemize}
It is worth noting that the above three attack strategies have some common issues: they fail to predict the behaviors of the victims $\mathcal{R}$ and ignore the completeness of the accessible pairwise comparisons $\boldsymbol{w}^*_u$. On the basis of the pairwise comparisons, the aggregated result depends on the algorithmic processes and statistical characters of the victims. It is known that the same votes would translate into  different election outcomes in the social choice theory or voting theory \cite{arrow2012social}. Instead of proposing a targeted strategy toward the special victim, blindly modifying or perturbing the existing data would not archive the goal of the adversary. In addition, the completeness of observed pairwise comparisons also impacts on the success chance of the target attack operations. Once it exists, the inaccessible part $\boldsymbol{w}^*_u$ becomes the invertible obstruction.
\begin{figure*}[h!]
    \centering
    \begin{subfigure}[b]{0.33\textwidth}
        \centering
        \includegraphics[width=\textwidth]{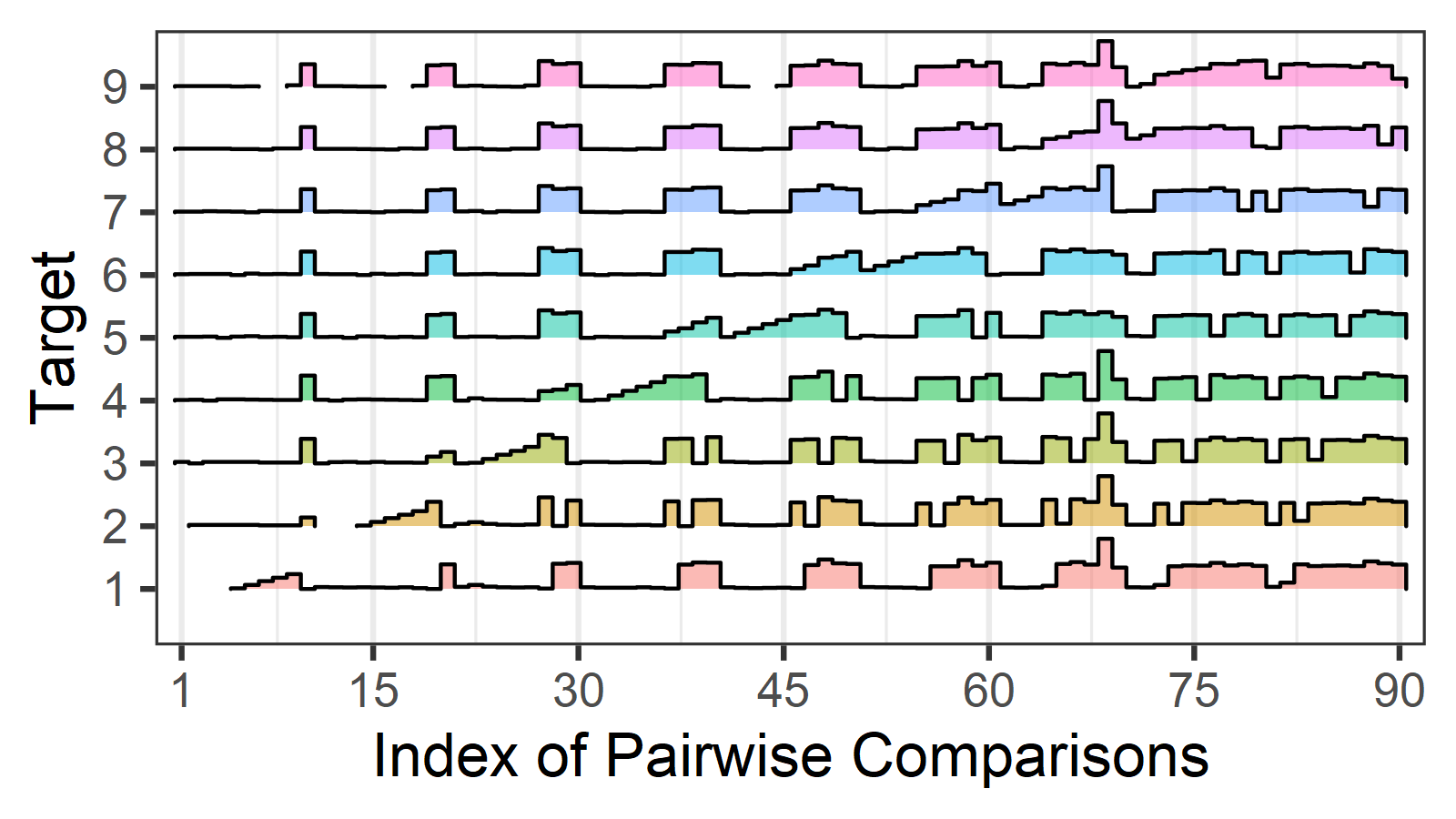}
        \caption{}
        \label{fig:hodge_data}
    \end{subfigure}
    \hfill
    \begin{subfigure}[b]{0.33\textwidth}
        \centering
        \includegraphics[width=\textwidth]{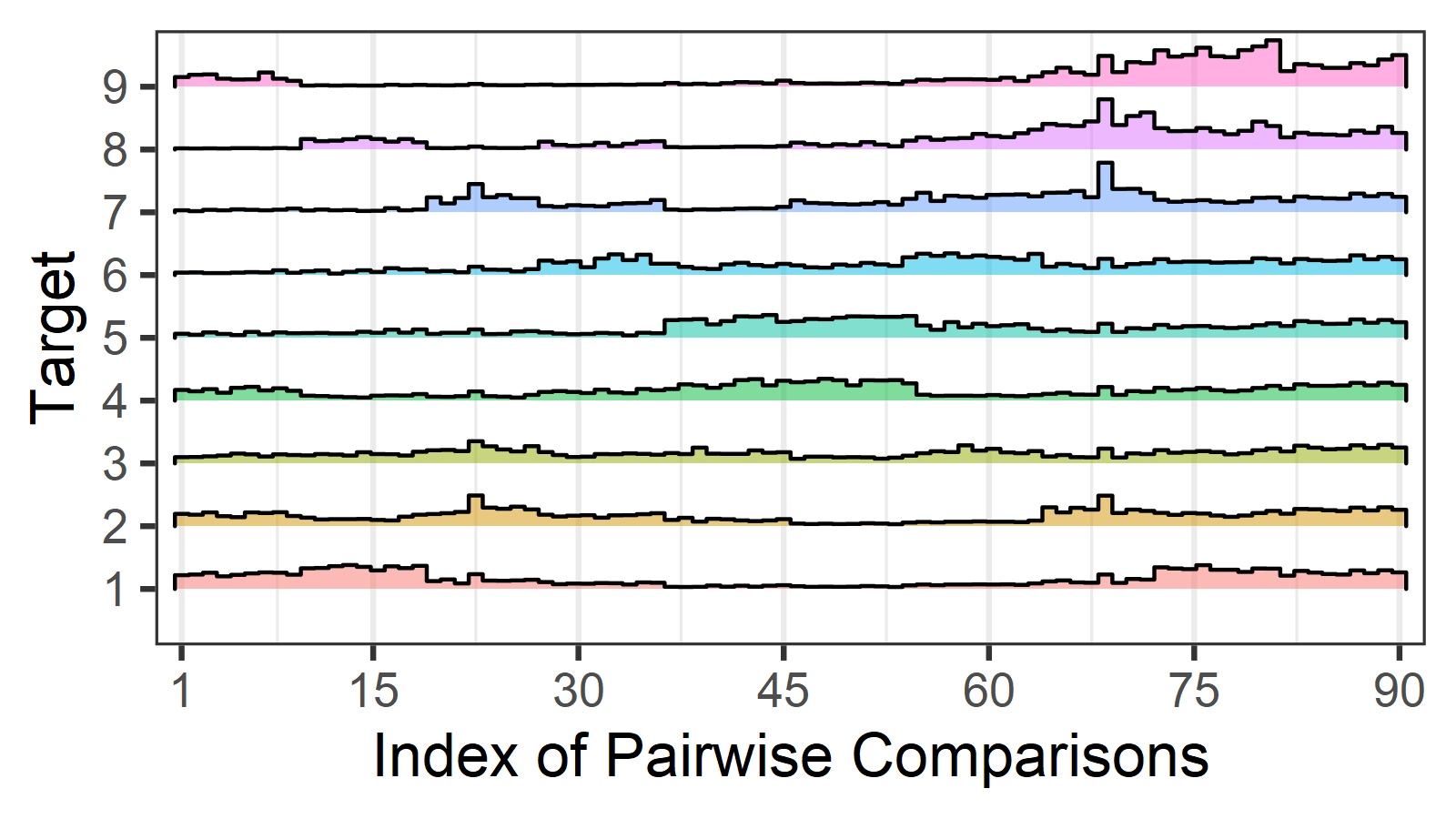}
        \caption{}
        \label{fig:spr_data}
    \end{subfigure}    
    \hfill
    \begin{subfigure}[b]{0.33\textwidth}
        \centering
        \includegraphics[width=\textwidth]{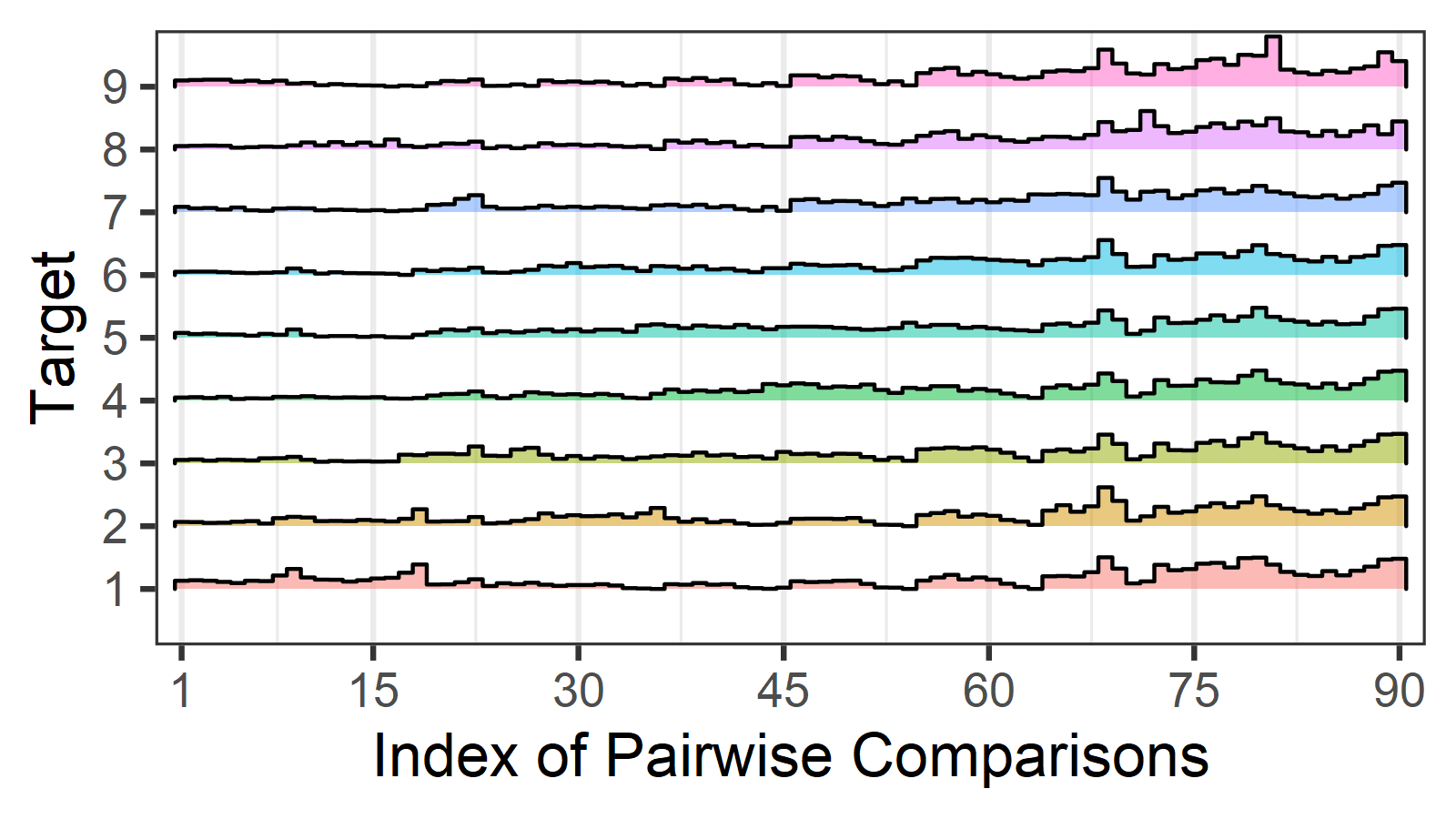}
        \caption{}
        \label{fig:spi_data}
    \end{subfigure}
    \caption{Data distribution of poisoned pairwise comparisons on simulated data with complete information and perfect feedback. The vertical axis lists all possible attack target and the horizontal axis displays all possible pairwise comparisons. All results are based on the same observed data. (a) The victim is \textbf{HodgeRank}. (b) The victim is \textbf{RankCentrality} and the adversary leverages the \textbf{reversible} stochastic transition matrix. (c) The victim is \textbf{RankCentrality} and the adversary utilizes the \textbf{irreversible} stochastic transition matrix.}
    \label{fig:data}
\end{figure*}

\noindent\textbf{Evaluation Metrics. }Based on the \textbf{IPCO} complexity model, we need some metrics to measure whether the perturbed ranking $\boldsymbol{\pi}_{\boldsymbol{\hat{\theta}}}$ achieves the attacker's goal or not. More specifically, given a threshold of \textbf{IPCO} complexity, the larger the ``similarity'' between the target ranking $\boldsymbol{\pi}^t_{\mathcal{A}}$ and the perturbed ranking $\boldsymbol{\pi}_{\boldsymbol{\hat{\theta}}}$ is, the more efficient the attack strategy is. Here we adopt \textbf{\textit{Reciprocal Rank (R-Rank)}} and \textbf{\textit{Kendall $\tau$ Coefficient (Kendall-$\tau$)}} for evaluating the correlation between two ranking lists. The first metric focuses on the top-$1$ candidates of the two ranking lists. Duo to the significance of the leading position in ranking list, this ``local'' metric is the most concerned indicator for the attackers. The second metric considers the position consistency of every candidate pair in the two ranking lists. So it is a ``global'' metric. 

\noindent\textbf{\textit{Reciprocal Rank (R-Rank).}} The reciprocal rank is a statistical measure for evaluating any process that produces an order list of possible responses to a series of queries, ordered by the probability of correctness or the ranking scores. The reciprocal rank of a ranking list is the multiplicative inverse of the leading object's position in the new order list. Let $\boldsymbol{\pi}^t_{\mathcal{A}}$ be the target ranking of the adversary and $\boldsymbol{\pi}_{\boldsymbol{\hat{\theta}}}$ be the aggregation result of the victim with the manipulated data from $\mathcal{A}$. The $\textbf{R-Rank}$ is defined as 
    \begin{equation}
        \textbf{R-Rank}(\boldsymbol{\pi}^t_{\mathcal{A}},\boldsymbol{\pi}_{\boldsymbol{\hat{\theta}}}) = \frac{1}{\ \boldsymbol{\pi}^{-1}_{\boldsymbol{\hat{\theta}}}\bigg[\boldsymbol{\pi}^t_{\mathcal{A}}(1)\bigg]\ },
    \end{equation}
    where $\boldsymbol{\pi}(i)$ refers to the item index which lies in the $i$-th position of ranking list $\boldsymbol{\pi}$, and $\boldsymbol{\pi}^{-1}[i]$ indicates the position of ranking list $\boldsymbol{\pi}$ which belongs to the item $i$. If it holds that $\boldsymbol{\pi}^t_{\mathcal{A}}(1) = \boldsymbol{\pi}_{\boldsymbol{\hat{\theta}}}(1)$,
    we have $\textbf{R-Rank}(\boldsymbol{\pi}^t_{\mathcal{A}},\boldsymbol{\pi}_{\boldsymbol{\hat{\theta}}})$ archives its maximum value $1$. The lager $\text{R-Rank}$ value indicates a better purposeful attack result. 
    
    \vspace{0.25cm}
    \noindent\textbf{\textit{Kendall $\tau$ Coefficient (Kendall-$\tau$).}} The Kendall rank correlation coefficient evaluates the degree of similarity between two ranking lists given the same objects. This coefficient depends upon the number of inversions of pairs of objects which would be needed to transform one rank order into the other. Let $\boldsymbol{V}=[n]$ be a set of $n$ candidates and $\boldsymbol{\pi}_1,\ \boldsymbol{\pi}_2$ are two total orders or permutations on $V$, the Kendall $\tau$ coefficient is defined to be 
    \begin{equation}
        d_{\tau}(\boldsymbol{\pi}_1,\ \boldsymbol{\pi}_2) = \frac{2}{n(n-1)}\cdot\vartheta,
    \end{equation}
    where
    \begin{equation}
        \vartheta = \sum_{i=1}^{n-1}\sum_{j=i+1}^n\vartheta\big(\boldsymbol{\pi}^{-1}_1[i],\ \boldsymbol{\pi}^{-1}_1[j],\ \boldsymbol{\pi}^{-1}_2[i],\ \boldsymbol{\pi}^{-1}_2[j]\big)
    \end{equation}
    is the number of different pairs between these two ordered sets $\boldsymbol{\pi}_1, \boldsymbol{\pi}_2$ as
    \begin{equation*}
        \begin{aligned}
            & & &\ \ \vartheta\big(\boldsymbol{\pi}^{-1}_1[i],\ \boldsymbol{\pi}^{-1}_1[j],\ \boldsymbol{\pi}^{-1}_2[i],\ \boldsymbol{\pi}^{-1}_2[j]\big)\\[5pt]
            & &=& 
            \left\{
            \begin{array}{rl}
                 1, & \text{if}\ \big(\boldsymbol{\pi}^{-1}_1[i]-\boldsymbol{\pi}^{-1}_1[j]\big)\big(\boldsymbol{\pi}^{-1}_2[i]-\boldsymbol{\pi}^{-1}_2[j]\big)>0,\\[7.5pt] 
                -1, & \text{if}\ \big(\boldsymbol{\pi}^{-1}_1[i]-\boldsymbol{\pi}^{-1}_1[j]\big)\big(\boldsymbol{\pi}^{-1}_2[i]-\boldsymbol{\pi}^{-1}_2[j]\big)<0,\\[7.5pt] 
                 0, & \text{otherwise}.
            \end{array}
            \right.
        \end{aligned}
    \end{equation*}
    Kendall $\tau$ coefficient counts the number of inconsistent comparisons between two rank orders. Lager Kendall-$\tau$ value indicates a better purposeful attack result. If $d_{\tau}(\boldsymbol{\pi}_1,\ \boldsymbol{\pi}_2)=1$, we have $\boldsymbol{\pi}_1=\boldsymbol{\pi}_2$.

\begin{figure}[h!]
    \centering
    \includegraphics[width=0.45\textwidth]{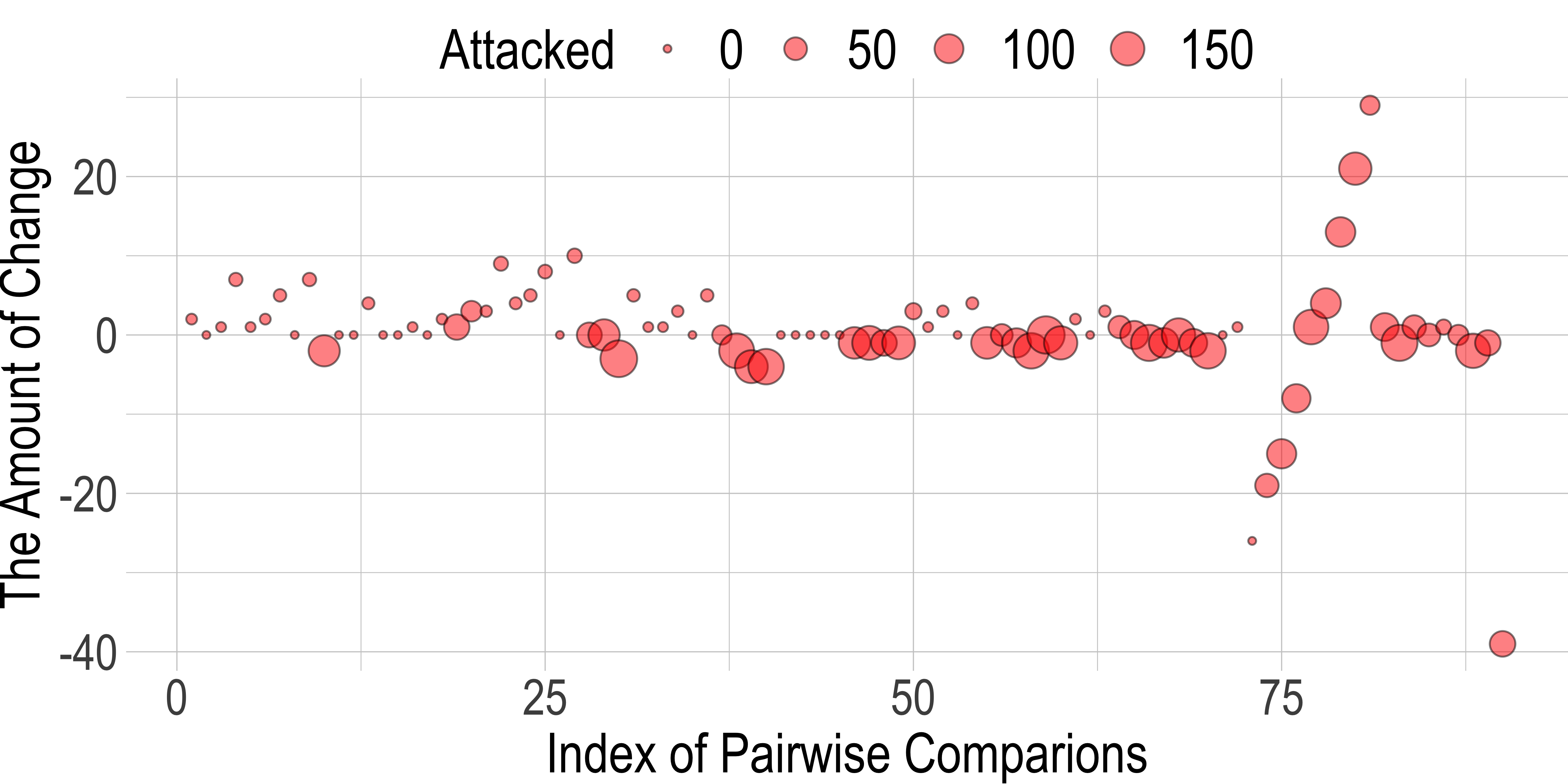}
    \caption{The modification of pairwise comparisons after attack.  Here the victim is \textbf{HodgeRank} and $\boldsymbol{\pi}^2_{\mathcal{A}}(1)=9$.the goal of the adversary is $\{9,10,8,7,6,5,4,3,2,1\}$ and the original aggregated result is $\{10,9,8,7,6,5,4,3,2,1\}$. The horizontal axis lists all possible pairwise comparisons. We index them as follows: No.$(i-1)*9+1$ to $i*9$ are the comparisons $\{(i,j)\ |\ j\in[10],j\neq i\}$. The size of red dot stands for the number of pairwise comparisons after attack. The larger the radius of the red dot, the more the number of the perturbed pairwise comparisons. The vertical axis shows the changes in each comparison. The positive/negative vertical coordinate of red dot stands for the increment/decrement.}\label{fig:hodge_data_1}
\end{figure}

\subsection{Simulated Study}
We validate the proposed target attack strategies against \textbf{Hodgerank} and \textbf{RankCentrality} on simulated data. Here the pairwise comparisons are generated as follows. First we construct a complete graph $\boldsymbol{G}=(\boldsymbol{V},\boldsymbol{E})$ where $\boldsymbol{V}=[10]$. Then the latent preference score is assigned to each candidate/vertex of $\boldsymbol{V}$ and the true ranking is obtained by these scores. Next, we randomly sample the pairwise comparisons $\boldsymbol{w}^*$ from $\boldsymbol{V}\times\boldsymbol{V}$ based on the true ranking. Notice that $\boldsymbol{w}^*$ could contain the comparisons which are inconsistent with the true ranking. Even with the same true ranking, the pairwise comparisons could be different. It means that the same $t$ can lead to induce different $\boldsymbol{\pi}_{\boldsymbol{\bar{\theta}}}(t)$ in \eqref{eq:target_constrcut}. To ensure the statistical stability, the attack procedures are conducted $50$ times for the same target $t$ with different comparisons. 

\vspace{0.25cm}
\noindent\textbf{Comparative Results. }The reciprocal rank and Kendall-$\tau$ coefficient between the perturbed ranking ${\boldsymbol{\pi}}_{\boldsymbol{\hat{\theta}}}$ and the target ranking $\boldsymbol{\pi}^{t}_{\mathcal{A}}$ are illustrated in Fig.~\ref{fig:simulate}. Here the ${\boldsymbol{\pi}}_{\boldsymbol{\hat{\theta}}}$ is obtained by sorting the latent preference scores $\boldsymbol{\hat{\theta}}$ aggregated by \textbf{HodgeRank} or \textbf{RankCentrality} (\textbf{Algorithm} \ref{alg:attack_hodge} or \ref{alg:attack_spectral}). The results indicate that all of our approaches exhibit higher success rate due to the parsimonious mechanism. Only the proposed methods stably lead $\textbf{R-Rank}(\boldsymbol{\pi}^t_{\mathcal{A}},\boldsymbol{\pi}_{\boldsymbol{\hat{\theta}}})$ to be $1$ and get $d_{\tau}(\boldsymbol{\pi}^t_{\mathcal{A}},\boldsymbol{\pi}_{\boldsymbol{\hat{\theta}}})$ close to $1$. However, the irreversible stochastic matrix attack method against \textbf{RankCentrality} would exhibit non-ideal performance with imperfect feedback as Fig. (\ref{fig:spi_rrank_ci}) and (\ref{fig:spi_kendall_ci}). We speculate that this phenomenon comes from the complicated construction of the irreversible stochastic matrix as \eqref{eq:irreversibiletransitionmatrix}. If the adversary only predominates imperfect feedback, he/she would not reconstruct the empirical stochastic matrix through the first step in Algorithm \ref{alg:attack_spectral}. This phenomenon inspires us to develop the ranking aggregation algorithms that can resist target attacks in the future.

\begin{figure}[h!]
    \centering
    \includegraphics[width=0.45\textwidth]{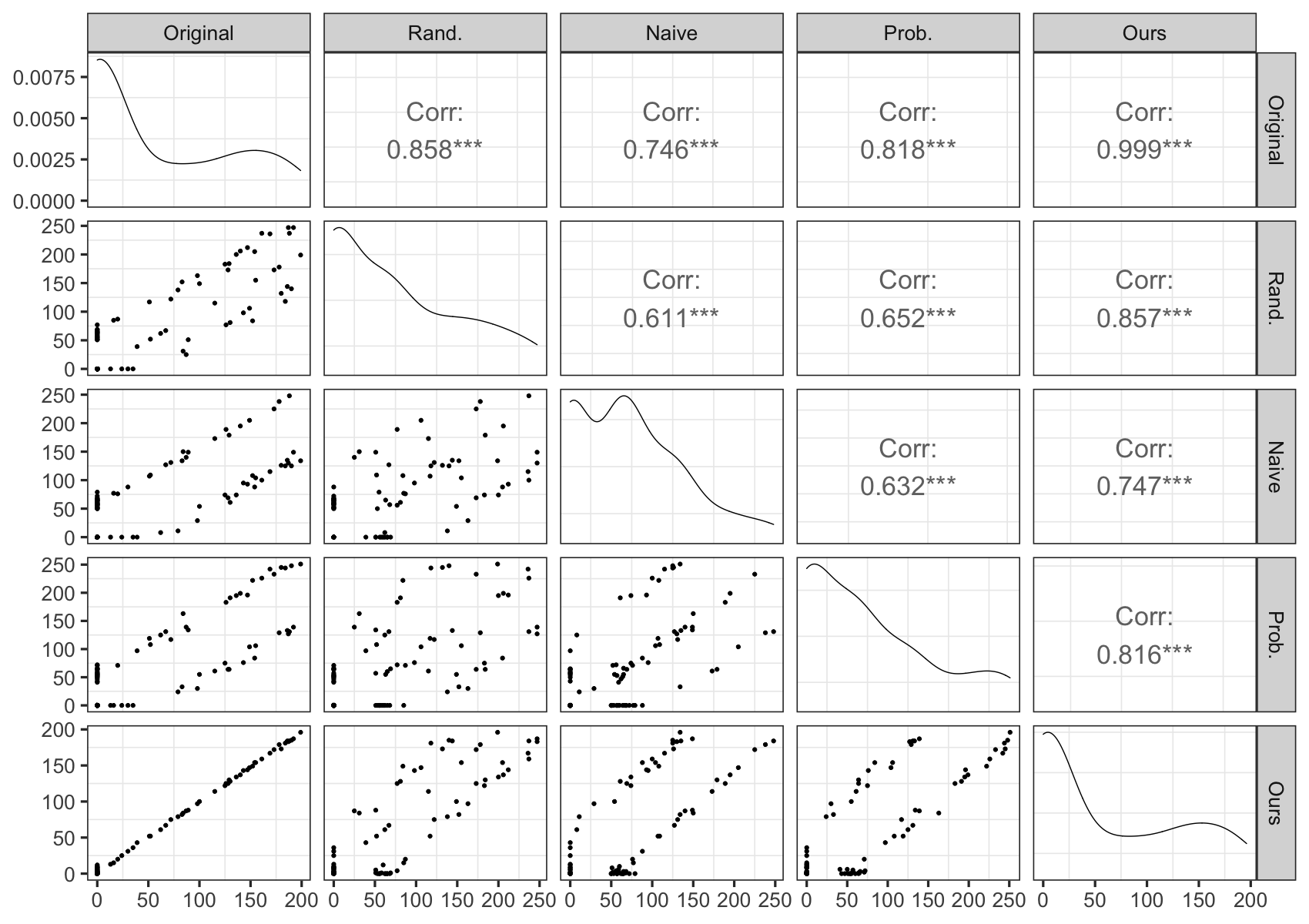}
    \caption{Correlations between the original simulated data and different modified data, where $\boldsymbol{\pi}^2_{\mathcal{A}}(1)=9$. The victim is \textbf{HodgeRank}. The scatter plots in the lower triangle part show a pair of data whose source is labeled in the “gray” regions. The upper left points of each scatter plot belong to the data labeled by the top ``gray'' regions; the lower right points come from the right gay box. In scatter plot, each point represents one type of pairwise comparison and these points are in ascending order of the weight. It is noteworthy that the horizontal axis of the scatter plot also shows the weight but not the index of pairwise comparison. The Pearson correlations of every pair of data are calculated in the upper triangle part and the ``$^{***}$'' stands for the $p$-value is smaller than $0.001$. The diagonal blocks are the data distribution normalized by the total number of pairwise comparisons.}\label{fig:hodge_data_2}
\end{figure}

\begin{table*}
    \footnotesize
    \centering
    \caption{Numeric results of different attack methods on Dublin election data. The {\color{softred}colored} results represent the failure cases in the experiments.}
    \begin{tabular}{@{}c|cccccc|cccccc|cccccc@{}}
        \hline
        \multirow{3}{*}{Target} & \multicolumn{6}{c|}{HodgeRank} & \multicolumn{6}{c|}{Spectral with Revers.} & \multicolumn{6}{c}{Spectral with Irrevers.} \\
                                & \multicolumn{2}{c}{C\&P} & \multicolumn{2}{c}{C\&I} & \multicolumn{2}{c|}{I\&P} & \multicolumn{2}{c}{C\&P} & \multicolumn{2}{c}{C\&I} & \multicolumn{2}{c|}{I\&P} & \multicolumn{2}{c}{C\&P} & \multicolumn{2}{c}{C\&I} & \multicolumn{2}{c}{I\&P} \\
                                & R.r& K.$\tau$& R.r& K.$\tau$& R.r& K.$\tau$& R.r& K.$\tau$& R.r& K.$\tau$& R.r& K.$\tau$& R.r& K.$\tau$& R.r& K.$\tau$& R.r& K.$\tau$\\ \hline
        $\scriptstyle\hat{\boldsymbol{\pi}}(n-1)$ & 1.00 & 1.00 & 1.00 & 1.00 & 1.00 & 1.00 & 1.00 & 1.00 & 1.00  & 1.00 & 1.00 & 1.00 & \color{softred}{0.33} & 0.91 & {\color{softred}0.33} & 0.34 & {\color{softred}0.25} & 0.58 \\ 
        $\scriptstyle\hat{\boldsymbol{\pi}}(n-2)$ & 1.00 & 1.00 & 1.00 & 1.00 & 1.00 & 1.00 & 1.00 & 1.00 & 1.00  & 1.00 & 1.00 & 1.00 & \color{softred}{0.25} & 0.58 & {\color{softred}0.33} & 0.27 & {\color{softred}0.25} & 0.58 \\
        $\scriptstyle\hat{\boldsymbol{\pi}}(n-3)$ & 1.00 & 1.00 & 1.00 & 1.00 & 1.00 & 1.00 & 1.00 & 1.00 & 1.00  & 1.00 & 1.00 & 1.00 & \color{softred}{0.33} & 0.56 & {\color{softred}0.33} & 0.27 & {\color{softred}0.33} & 0.56 \\
        $\scriptstyle\hat{\boldsymbol{\pi}}(n-4)$ & 1.00 & 1.00 & 1.00 & 1.00 & 1.00 & 1.00 & 1.00 & 1.00 & 1.00  & 1.00 & 1.00 & 1.00 & \color{softred}{0.33} & 0.91 & {\color{softred}0.33} & 0.31  & {\color{softred}0.33} & 0.91 \\
        \hline
    \end{tabular}
    \label{tab:duchi}
\end{table*}
\begin{table*}
    \footnotesize
    \centering
    \caption{Numeric results of different attack methods on human age data. The {\color{softred}colored} results represent the failure cases in the experiments.}
    \begin{tabular}{@{}c|cccccc|cccccc|cccccc@{}}
        \hline
        \multirow{3}{*}{Target} & \multicolumn{6}{c|}{HodgeRank} & \multicolumn{6}{c|}{Spectral with Revers.} & \multicolumn{6}{c}{Spectral with Irrevers.} \\
                                & \multicolumn{2}{c}{C\&P} & \multicolumn{2}{c}{C\&I} & \multicolumn{2}{c|}{I\&P} & \multicolumn{2}{c}{C\&P} & \multicolumn{2}{c}{C\&I} & \multicolumn{2}{c|}{I\&P} & \multicolumn{2}{c}{C\&P} & \multicolumn{2}{c}{C\&I} & \multicolumn{2}{c}{I\&P} \\
                                & R.r& K.$\tau$& R.r& K.$\tau$& R.r& K.$\tau$& R.r& K.$\tau$& R.r& K.$\tau$& R.r& K.$\tau$& R.r& K.$\tau$& R.r& K.$\tau$& R.r& K.$\tau$\\ \hline
        $\scriptstyle\hat{\boldsymbol{\pi}}(n-1)$ & 1.00 & 0.20 & 1.00 & 0.16 & 1.00 & 0.22 & 1.00 & 0.68 & 1.00  & 0.04  & 1.00 & 0.60  & 1.00  & 0.07 & {\color{softred}0.33} & 0.43 & {\color{softred}0.25} & 0.56 \\ 
        $\scriptstyle\hat{\boldsymbol{\pi}}(n-2)$ & 1.00 & 0.34 & 1.00 & 0.09 & 1.00 & 0.22 & 1.00 & 0.52 & {\color{softred}0.33}  & 0.006 & 1.00 & 0.29  & 1.00  & 0.15 & {\color{softred}0.25} & 0.23 & {\color{softred}0.33} & 0.46 \\
        $\scriptstyle\hat{\boldsymbol{\pi}}(n-3)$ & 1.00 & 0.16 & 1.00 & 0.01 & 1.00 & 0.21 & 1.00 & 0.54 & 1.00  & 0.09  & 1.00 & 0.30  & 1.00  & 0.11 & {\color{softred}0.25} & 0.31 & {\color{softred}0.25} & 0.34 \\
        $\scriptstyle\hat{\boldsymbol{\pi}}(n-4)$ & 1.00 & 0.26 & \color{softred}{0.20} & 0.12 & 1.00 & 0.06 & 1.00 & 0.56 & {\color{softred}0.50}  & -0.03 & 1.00 & 0.57  & 1.00  & 0.12 & {\color{softred}0.25} & 0.02 & {\color{softred}0.25} & 0.08 \\
        $\scriptstyle\hat{\boldsymbol{\pi}}(n-5)$ & 1.00 & 0.21 & 1.00 & 0.09 & 1.00 & 0.14 & 1.00 & 0.56 & 1.00  & 0.16  & 1.00 & 0.29  & 1.00  & 0.16 & {\color{softred}0.25} & 0.14 & {\color{softred}0.20} & 0.42 \\ \hline
    \end{tabular}
    \label{tab:age}
\end{table*}

\noindent \textbf{Visualization. }We then visualize the attacked data distribution $\boldsymbol{\hat{w}}_k$ of the proposed method with different targets in Fig. \ref{fig:data}. From the same original data $\boldsymbol{w}^*_k$, the proposed methods make efficient manipulation with specific purposes. We observe that the bins which represent the target winning the other candidates are higher after the attack procedure. Meanwhile, the number of the pairwise comparisons which indicate the target is defeated by the other candidates is decreased. Furthermore, we display these phenomenons with a `'magnifier'' in Fig. \ref{fig:hodge_data_1}. 
Here the horizontal axis lists all possible pairwise comparisons. We index them as follows: No.$(i-1)*9+1$ to $i*9$ are the comparisons $\{(i,j)\ |\ j\in[10],j\neq i\}$. The size of red dots stands for the number of pairwise comparisons after attack. The larger the radius of the red dot, the more the number of the perturbed pairwise comparisons. The vertical axis shows the changes in each comparison. The positive/negative vertical coordinate of red dot stands for the increment/decrement. The pairwise comparisons whose indices range from $73$ to $81$ are involved with the target candidate $t=9$. We can see that the numbers of No. $77$ -- $81$ comparisons are increased after the attack procedure; meanwhile, the number of No. $90$ ($10\succ 9$) is decreased by $40$. Moreover, all comparisons which stand for the original top-$1$ candidates are defeated by the other candidates (No. $9,\ 18,\ \dots,\ 72$) which are also reduced by the proposed methods. The manipulation operation on the pairwise comparisons of the adversary naturally leads to the following questions: \textit{is such a modification of the data easy to detect?} We answer this question by Fig. \ref{fig:hodge_data_2}. We plot each pair of poisoned data by different methods on the left part (lower triangle part) and calculate the Pearson correlation on the right part (upper triangle part). The scatter plots in the lower triangle part show a pair of data whose source is labeled in the “gray” regions. The upper left points of each scatter plot belongs to the data labeled by the top gay box; the lower right points come from the right gay box. In scatter plot, each point represents one type of pairwise comparison and these points are in ascending order of the weight. The data distribution normalized by the total number of pairwise comparisons is available on the diagonal. The Pearson correlation coefficient between the original data and the proposed method is $0.999$, and ``***'' represents that the $p$-value is smaller than $0.001$. This result tells us that the minor manipulation of the original data can make the ranking aggregation result arbitrarily arranged with the specified purposes.

\subsection{Human Age}
\textbf{Description. }$30$ images from human age dataset FGNET\footnote{\url{https://yanweifu.github.io/FG_NET_data/}} are annotated by a group of volunteer users on a crowdsourcing platform\footnote{\url{http://www.chinacrowds.com/}}. The ground-truth age ranking is known to us. The annotator is presented with two images and given a binary choice of which one is older. Totally, we obtain $8,017$ pairwise comparisons from $94$ annotators.
\vspace{0.25cm}

\noindent\textbf{Comparative Results. } It is worth mentioning that this real-world data has a high percentage of outliers (about $20\%$ of all comparisons conflict with the correct age ranking). The proposed methods against \textbf{HodgeRank} and \textbf{RankCentrality} with reversible stochastic transition matrix still show promise attack effect as Table \ref{tab:age}. An irreversible stochastic matrix would be failed except the complete information and perfect feedback scenario. These observations further inspires us that the spectral ranking with irreversible stochastic matrix has the potential to resist the target attack.

\subsection{Dublin Election}
\textbf{Description. }The Dublin election data set\footnote{\url{http://www.preflib.org/data/election/irish/}} contains a complete record of votes for elections held in county Meath, Dublin, Ireland in 2002. This set contains $64,081$ votes over $14$ candidates. These votes could be a complete or partial list over the candidate set. The ground-truth ranking of $14$ candidates is based on their obtained first preference votes\footnote{\url{https://electionsireland.org/result.cfm?election=2002&cons=178&sort=first}}. The five candidates who receive the most first preference votes will be the winner of the election. We are interested in the top-$5$ performance of the pairwise rank aggregation method. Then these votes are converted into the pairwise comparisons. The total number of comparisons is $652,817$. 

\noindent\textbf{Comparative Results.} It is worth noting that the election result is not obtained by pairwise ranking aggregation. However, the ordered list aggregated from induced comparisons still shows positive correlations with the actual election result. Once $\mathcal{A}$ generates the successful strategy against the \textbf{HodgeRank} and \textbf{RankCentrality}, this attack plan could be adopted to manipulate the election in the real world. Different from the manipulation or strategic voting problem in social choice \cite{brandt2016handbook}, the poisoning attack framework could break the barrier of \cite{DBLP:journals/toct/HemaspaandraHM20} or the defense mechanism for colonel game \cite{8374881} which is an abstraction involving the election of two candidates. The aggregation results of \textbf{HodgeRank} and \textbf{RankCentrality} are totally manipulated by the proposed methods, see Table \ref{tab:duchi}. Attacking the real election systems and evaluating their robustness is an interesting problem in the future.

\section{Conclusion}
We establish the first study of target attacks in the context of ranking aggregation with pairwise comparisons to the best of our knowledge. The attack problem with specified purpose is formulated as the fixed point of the continuous operator. Moreover, a game-theoretic framework is proposed to model the interaction between the adversary and the victim. Two typical algorithms are selected as the victims. The attack strategy against \textbf{HodgeRank} is modeled as the least squared problems with a constraint which derived from the optimal condition of \textbf{HodgeRank}. We prove that the target ranking would be a fixed point in the complete game. When attacking against spectral ranking algorithm like \textbf{RankCentrality}, an inverse eigenvalue problem of the stochastic transition matrix is constructed. Moreover, we adopt minimizing the worst-case asymptotic variance as a constraint of the corresponding inverse eigenvalue problem. Our empirical studies show that the proposed attack methods could achieve the attacker's goal in the sense that the leading candidate of the aggregated ranking list is the designated one by the adversary. 

There are many avenues for further investigation – providing the finite-sample and asymptotic results, characterizing the theoretical performance of the adversarial game with the minimum knowledge and extending our attacks to more general ranking system like election and cryptocurrency. Vulnerability is an interesting question which can help to develop the reliable ranking algorithms in the future.

%









\ifCLASSOPTIONcaptionsoff
  \newpage
\fi



%


\bibliographystyle{plain}
\bibliography{sample}
    \newpage
    \onecolumn

    
    
    \section*{Proof of Theorem 1}
    First we review the definition of the continuous operator games.
    \operatorgame*
    As the optimization problem of each player usually appears in the form of minimizing a objective function with constraints, we focus the case $\mathcal{T}_p:\boldsymbol{S}_p\rightarrow\mathbb{R}$.

    \nashequilibrium*
    Here we follow the treatment of \cite{10.2307/2032478,DBLP:books/daglib/0023252} and summarize the proof sketches as follows: 
    \begin{itemize}
        \item We approximate the original game with a sequence of finite games, which correspond to successively finer discretization of the original game.
        \vspace{0.15cm}
        \item We use Nash’s theorem to produce an equilibrium for each approximation.
        \vspace{0.15cm}
        \item We use the weak topology and the continuity assumptions to show that these converge to an equilibrium of the original game. 
    \end{itemize}

    \subsection*{Closeness of Two Operator Games and $\epsilon$-Equilibrium}
    Let 
    \begin{equation*}
        \mathcal{T}\ =\ \{\mathcal{T}_p\ |\ p\in\boldsymbol{P}\}
    \end{equation*}
    and
    \begin{equation*}
        \tilde{\mathcal{T}}\ =\ \{ \tilde{\mathcal{T}}_p\ |\ p\in\boldsymbol{P}\}
    \end{equation*}
    be two profiles of continuous operators defined on $\boldsymbol{S}$ such that for each $p\in\boldsymbol{P}$, the operators $\mathcal{T}_p:\boldsymbol{S}_p\rightarrow\mathbb{R}^{n}$ and $\tilde{\mathcal{T}}_p:\boldsymbol{S}_p\rightarrow\mathbb{R}^{n}$ are bounded measurable operators. The distance between the operator profiles $\mathcal{T}$ and $\tilde{\mathcal{T}}$ could be defined as
    \begin{equation}
        \underset{p\in\boldsymbol{P}}{{\vphantom{\textbf{\textit{sup}}}}\textbf{\textit{max}}}\ \underset{s\in\boldsymbol{S}}{\textbf{\textit{sup}}}\ \left|\mathcal{T}_p(\boldsymbol{s}) - \tilde{\mathcal{T}}_p(\boldsymbol{s})\right|.
    \end{equation}
    Consider two strategic form games defined by two profiles of continuous operators:
    \begin{equation}
        \label{eq:strategic}
        \begin{aligned}
            & \mathcal{G} &=&\ \ \ \langle\boldsymbol{P},\boldsymbol{S},\{\mathcal{T}_p\} \rangle,\\
            & \tilde{\mathcal{G}} &=&\ \ \ \langle\boldsymbol{P},\boldsymbol{S},\{\tilde{\mathcal{T}}_p\} \rangle.
        \end{aligned}
    \end{equation}
    Suppose that $\boldsymbol{g}$ and $\tilde{\boldsymbol{g}}$ are mixed strategy Nash equilibria of $\mathcal{G}$ and $\tilde{\mathcal{G}}$ correspondingly. It is noteworthy that $\boldsymbol{g}$ needs not be equal to $\tilde{\boldsymbol{g}}$. Even though $\mathcal{T}$ and $\tilde{\mathcal{T}}$ are close enough, $\boldsymbol{g}$ and $\tilde{\boldsymbol{g}}$ would be far apart. However, the following definition gives a case in which the equilibria of $\mathcal{G}$ are ``almost'' the equilibria of $\tilde{\mathcal{G}}$.

    \begin{definition}[$\epsilon$-equilibrium]
        Given $\epsilon\geq 0$, a mixed strategy $g$ is called an $\epsilon$-equilibrium if for all $p\in\boldsymbol{P}$ and $\boldsymbol{s}_p\in\boldsymbol{S}_p$,
        \begin{equation}
            \mathcal{T}_p(\boldsymbol{s}_p,\boldsymbol{g}_{-p})\leq\mathcal{T}_p(\boldsymbol{g}_p,\boldsymbol{g}_{-p})+\epsilon.
        \end{equation}
        If $\epsilon=0$, an $\epsilon$-equilibrium is a Nash equilibrium in the usual sense. 
    \end{definition}

    The following result shows that such $\epsilon$-equilibria have a continuity property across games.

    \begin{proposition}
        \label{pro:1}
        Let $\mathcal{G}$ be a continuous operator game. Assume that there exist two sequences $\{\boldsymbol{g}^k\}$ and $\{\epsilon_k\}$ where $k\in\mathbb{N}$ and
            \begin{equation*}
                \begin{aligned}
                    & \underset{k\rightarrow\infty}{\lim}\ \boldsymbol{g}^k &=&\ \ \ \boldsymbol{g},\\
                    & \underset{k\rightarrow\infty}{\lim}\ \epsilon^k &=&\ \ \ \epsilon.
                \end{aligned}
            \end{equation*}
        For each $k$, $\boldsymbol{g}^k$ is an $\epsilon^k$-equilibrium of $\mathcal{G}$. Then $g$ is an $\epsilon$-equilibrium of $\mathcal{G}$  
    \end{proposition}

    \begin{proof}
        For all $p\in\boldsymbol{P}$ and $\boldsymbol{s}_p\in\boldsymbol{S}_p$, we have
        \begin{equation}
            \mathcal{T}_p(\boldsymbol{s}_p,\boldsymbol{g}^k_{-p})\leq\mathcal{T}_p(\boldsymbol{g}^k_p)+\epsilon.
        \end{equation}
        Taking the limit as $k\rightarrow\infty$ in the preceding relation, and using the continuity of the operator together with the convergence of probability distributions under weak topology
        \begin{equation*}
            \underset{k\rightarrow\infty}{\lim} \int_{\boldsymbol{s}_p\in\boldsymbol{S}_p} f(\boldsymbol{s}_p)d\boldsymbol{g}^k_p(\boldsymbol{s}_p) = \int_{\boldsymbol{s}_p\in\boldsymbol{S}_p} f(\boldsymbol{s}_p)d\boldsymbol{g}_p(\boldsymbol{s}_p),
        \end{equation*}
        for every bounded continuous function $f:\boldsymbol{S}\rightarrow\mathbb{R}$, we obtain
        \begin{equation}
            \mathcal{T}_p(\boldsymbol{s}_p,\boldsymbol{g}_{-p})\leq\mathcal{T}_p(\boldsymbol{g}_p)+\epsilon
        \end{equation}
        which establishes the desired results.
    \end{proof}

    We next define formally the closeness of two strategic form games.
    \begin{definition}
        Let $\mathcal{G}$ and $\tilde{\mathcal{G}}$ be two strategic form games as \eqref{eq:strategic}. Assume that the operators $\mathcal{T}$ and $\tilde{\mathcal{T}}$ are measurable and bounded. We say that $\tilde{\mathcal{G}}$ is a $\varphi$-approximation of $\mathcal{G}$ if $\tilde{\mathcal{G}}$ satisfies
        \begin{equation}
            |\mathcal{T}_p(\boldsymbol{s})-\tilde{\mathcal{T}}_p(\boldsymbol{s})|\leq\varphi,\ \forall\ p\in\boldsymbol{P},\ \boldsymbol{s}\in\boldsymbol{S}.
        \end{equation}
    \end{definition}

    The next proposition relates $\epsilon$-equilibrium of the approximation.

    \begin{proposition}
        \label{pro:2}
        If $\tilde{\mathcal{G}}$ is a $\varphi$-approximation of $\mathcal{G}$ and $\boldsymbol{g}$ is an $\epsilon$-equilibrium of $\tilde{\mathcal{G}}$, $\boldsymbol{g}$ would be an $(\epsilon+2\varphi)$-equilibrium of $\mathcal{G}$.
    \end{proposition}

    \begin{proof}
        It holds that
        \begin{equation*}
            \begin{aligned}
                & & &\ \ \mathcal{T}_p(\boldsymbol{s}_p,\boldsymbol{g}_{-p})-\mathcal{T}_p(\boldsymbol{g})\\
                & &=&\ \ \mathcal{T}_p(\boldsymbol{s}_p,\boldsymbol{g}_{-p})-\tilde{\mathcal{T}}_p(\boldsymbol{s}_p,\boldsymbol{g}_{-p})+\tilde{\mathcal{T}}_p(\boldsymbol{s}_p,\boldsymbol{g}_{-p})-\tilde{\mathcal{T}}(\boldsymbol{g})+\tilde{\mathcal{T}}(\boldsymbol{g})-\mathcal{T}_p(\boldsymbol{g})\\
                & &\leq&\ \ \epsilon + 2\varphi.
            \end{aligned}
        \end{equation*}
    \end{proof}

    The next proposition shows that we can approximate a continuous operator game with an essentially finite game to an arbitrary degree of accuracy.

    \begin{proposition}
        \label{pro:3}
        For any continuous game $\mathcal{G}$ and any $\varphi>0$, there exists an “essentially finite” game which is a $\varphi$-approximation of $\mathcal{G}$.
    \end{proposition}

    \begin{proof}
        Since $(\boldsymbol{S}, d(\cdot,\cdot))$ is a compact metric space, $\{\mathcal{T}_p\}$ are uniformly continuous, \textit{i.e.}, for all $\varphi>0$, there exists some $\epsilon>0$ such that
        \begin{equation*}
            \mathcal{T}_p(\boldsymbol{u}) - \mathcal{T}_p(\boldsymbol{v})\leq \varphi
        \end{equation*}
        for all $\boldsymbol{u},\boldsymbol{v}\in\boldsymbol{S}$ and $d(\boldsymbol{u},\boldsymbol{v})\leq\epsilon$. Moreover,  we can cover the compact metric space $(\boldsymbol{S}_p, d(\cdot,\cdot))$ with finitely many open balls $\{\boldsymbol{U}^j_p\}$, whose radius is less than $\epsilon$. Without loss of generality, these balls are disjoint and nonempty. Choose an $\boldsymbol{s}^j_p\in\boldsymbol{U}^j_p$ for each $p$ and $j$. We define the “essentially finite” game $\tilde{\mathcal{G}}$ with the continuous operator $\tilde{\mathcal{T}}_p$ as
        \begin{equation}
            \tilde{\mathcal{T}}_p(\boldsymbol{s}) = \tilde{\mathcal{T}}_p(\boldsymbol{s}^j_1,\dots,\boldsymbol{s}^j_{|\boldsymbol{P}|}),\ \ \ \ \forall\ \ \boldsymbol{s}\in\boldsymbol{U}^j = \prod_{k=1}^{|\boldsymbol{P}|}\boldsymbol{U}^j_k.
        \end{equation}
        Then it holds that
        \begin{equation}
            \left|\tilde{\mathcal{T}}_p(\boldsymbol{s})-\mathcal{T}_p(\boldsymbol{s})\right|\leq \varphi,\ \ \ \forall\ \boldsymbol{s}\in\boldsymbol{S},\ p\in\boldsymbol{P}
        \end{equation}
        since $d(\boldsymbol{s}_p,\boldsymbol{s}^j_p)\leq \epsilon$ for all $j$. Now we have the desired results.
    \end{proof}

    \subsection*{Proof of Theorem \ref{thm:nash_equilibrium}}
    Let $\{\varphi^k\}$ be a sequence with
    \begin{equation*}
        \underset{k\rightarrow\infty}{\lim}\ \varphi^k = 0.
    \end{equation*}
    \begin{itemize}
        \item For each $\varphi^k$, there exist an ``essentially finite'' $\varphi^k$-approximation $\mathcal{G}^k$ of $\mathcal{G}$ by Proposition \ref{pro:3}.
        \item Since $\mathcal{G}^k$ is ``essentially finite'' for each $k$, it has a $0$-equilibrium $\boldsymbol{g}^k$ by the Nash's Theorem \cite{Nash48,10.2307/1969529}.
        \item By Proposition \ref{pro:2}, we know that $\boldsymbol{g}^k$ is a $2\varphi^k$-equilibrium of $\mathcal{G}$.
        \item As $\boldsymbol{S}$ is a compact space, $\{\boldsymbol{g}^k\}$ has a convergent sub-sequence. Without loss of generality, we assume that 
        \begin{equation*}
            \underset{k\rightarrow\infty}{\lim}\ \boldsymbol{g}^k =\boldsymbol{g}.
        \end{equation*}
        \item By the convergence of $\{\varphi^k\}$ and $\{\boldsymbol{g}^k\}$ with Proposition \ref{pro:1}, it holds that $\boldsymbol{g}$ is a $0$-equilibrium of $\mathcal{G}$.
    \end{itemize}

    \subsection*{Proof of Theorem \ref{thm:hodge_nash_equilibrium}}

    \fixpointhodge*

    \begin{proof}
        First we introduce the slack variable and reformulate the continuous operator of adversary $\mathcal{T}_{\mathcal{A}}$ \eqref{opt:com_info} as 
        \begin{equation}
            \label{opt:com_info_ref}
            \begin{aligned}
                & &\underset{\boldsymbol{w},\ \ \boldsymbol{x}}{\textbf{\textit{minimize}}}&\ \ \frac{1}{2}\|\boldsymbol{w}-\boldsymbol{w}^*\|^2_2 + g(\boldsymbol{x})\\[5pt]
                & &\textbf{\textit{subject to}}&\ \ \boldsymbol{B}\boldsymbol{w} = \boldsymbol{0},\ \ \boldsymbol{w} = \boldsymbol{x},
            \end{aligned}
        \end{equation}
        where 
        \begin{equation}
            \boldsymbol{B} = \boldsymbol{C}^\top\textbf{\textit{diag}}\big(\big(\boldsymbol{C}+2\lambda_0\boldsymbol{I}\big)\boldsymbol{\theta}_{\mathcal{A}}-\boldsymbol{y}\big),
        \end{equation}
        and $g$ is an indicator function as
        \begin{equation}
            g(\boldsymbol{x}) = \left\{
            \begin{array}{rl}
                0, & \text{if}\ \boldsymbol{x}\geq 0;\\[5pt]
                \infty, & \text{otherwise.}
            \end{array}
            \right.
        \end{equation}
        The augmented Lagrangian function of problem \eqref{opt:com_info_ref} is defined by, for any $(\boldsymbol{\mu}, \boldsymbol{w}, \boldsymbol{x})\in\mathbb{R}^{n+N}\times\mathbb{R}^N\times\mathbb{R}^N$,
        \begin{equation}
            L_{\rho}(\boldsymbol{w},\boldsymbol{x},\boldsymbol{\mu})=f(\boldsymbol{w})+g(\boldsymbol{x})+\big\langle\boldsymbol{\mu}, \mathcal{B}^*\boldsymbol{w}+\mathcal{I}^*\boldsymbol{x}\big\rangle + \frac{\rho}{2}\big\|\mathcal{B}^*\boldsymbol{w}+\mathcal{I}^*\boldsymbol{x}\big\|^2,
        \end{equation}
        where $f(\boldsymbol{w}) = \frac{1}{2}\|\boldsymbol{w}-\boldsymbol{w}^*\|^2_2$, $\rho>0$ is a given penalty parameter, 
        \begin{equation}
            \mathcal{B}* = \begin{pmatrix}
            \boldsymbol{B}\\          
            \boldsymbol{I}_N
            \end{pmatrix},\ \ 
            \mathcal{I}* = \begin{pmatrix}
            \boldsymbol{0}\\          
            -\boldsymbol{I}_N
            \end{pmatrix}
        \end{equation}
        and $\mathcal{B}^*$ and $\mathcal{I}^*$ are the adjoint of $\mathcal{B}$ and $\mathcal{I}$ respectively. Suppose that the \textbf{sPADMM} \cite{doi:10.1137/110853996} scheme takes the following iteration scheme for $k = 0, 1, \dots$,
        \begin{equation}
            \left\{
            \begin{array}{l}
            \boldsymbol{w}^{k+1}   = \underset{\boldsymbol{w}}{\textbf{\textit{arg min}}}\ \bigg\{L_{\rho}(\boldsymbol{w},\boldsymbol{x},\boldsymbol{\mu})+\frac{1}{2}\|\boldsymbol{w}-\boldsymbol{w}^k\|^2_{\mathcal{H}}\bigg\}, \\[10pt] 
            \boldsymbol{x}^{k+1}   = \underset{\boldsymbol{x}}{\textbf{\textit{arg min}}}\ \bigg\{L_{\rho}(\boldsymbol{w},\boldsymbol{x},\boldsymbol{\mu})+\frac{1}{2}\|\boldsymbol{x}-\boldsymbol{x}^k\|^2_{\mathcal{Q}}\bigg\}, \\[15pt]          
            \boldsymbol{\mu}^{k+1} = \boldsymbol{\mu}^{k} + \eta\rho(\mathcal{B}^*\boldsymbol{w}+\mathcal{I}^*\boldsymbol{x}),
            \end{array}
            \right.
        \end{equation}
        where $\mathcal{H}:\mathbb{R}^N\rightarrow\mathbb{R}^N$ and $\mathcal{Q}:\mathbb{R}^N\rightarrow\mathbb{R}^N$ are two self-adjoint positive semi-definite linear operators. The \textbf{sPADMM} scheme with $\mathcal{H}=\boldsymbol{0}$ and $\mathcal{H}=\boldsymbol{0}$ will be \textbf{ADMM} which is adopted as the proposed algorithm. Let $(\boldsymbol{\hat{\mu}},\boldsymbol{\hat{w}},\boldsymbol{\hat{x}})$ be an arbitrary solution to the KKT system \eqref{eq:KKT} of \eqref{opt:com_info_ref}.
        We define the following auxiliary notations
        \begin{equation}
            \label{eq:spADMM}
            \left\{
            \begin{array}{l}
            \boldsymbol{u}^{k}   := -\mathcal{B}[\boldsymbol{\mu}^k+(1-\eta)\rho(\mathcal{B}^*\boldsymbol{w}^k_e+\mathcal{I}^*\boldsymbol{x}^k_e)+\rho\mathcal{I}^*(\boldsymbol{x}^{k-1}-\boldsymbol{x}^k)]-\mathcal{H}(\boldsymbol{w}^{k}-\boldsymbol{w}^{k-1}), \\[15pt] 
            \boldsymbol{v}^{k}   := -\mathcal{I}[\boldsymbol{\mu}^k+(1-\eta)\rho(\mathcal{B}^*\boldsymbol{w}^k_e+\mathcal{I}^*\boldsymbol{x}^k_e)]-\mathcal{Q}(\boldsymbol{x}^{k}-\boldsymbol{x}^{k-1}), \\[15pt]          
            \boldsymbol{\Psi}_{k} := \frac{1}{\eta\rho}\|\boldsymbol{\mu}^k_e\|^2+\|\boldsymbol{w}^k_e\|^2_{\mathcal{H}}+\|\boldsymbol{x}^k_e\|^2_{\mathcal{Q}+\rho\mathcal{B}\mathcal{B}^*},\\[15pt] 
            \boldsymbol{\Phi}_{k} := \boldsymbol{\Psi}_{k} + \|\boldsymbol{x}^k-\boldsymbol{x}^{k-1}\|^2_{\mathcal{Q}}+\textbf{\textit{max}}(1-\eta,1-\eta^{-1})\rho\|\mathcal{B}^*\boldsymbol{w}^{k}_e+\mathcal{I}^*\boldsymbol{x}^k_e\|^2,
            \end{array}
            \right.
        \end{equation}
        where $\boldsymbol{w}^k_e = \boldsymbol{w}^k - \boldsymbol{\hat{\mu}}$. Then the \textbf{KKT} system of \eqref{opt:com_info_ref} is 
        \begin{equation}
            \label{eq:KKT}
            -\mathcal{B}^*\boldsymbol{\mu}\in\partial f(\boldsymbol{w}),\ \ -\mathcal{I}^*\boldsymbol{\mu}\in\partial g(\boldsymbol{x}),\ \ \mathcal{B}^*\boldsymbol{w}+\mathcal{I}^*\boldsymbol{x}=\boldsymbol{0}.
        \end{equation}
        Suppose that $(\boldsymbol{\mu}^{\infty},\boldsymbol{w}^{\infty},\boldsymbol{x}^{\infty})$ is an accumulation point of $\{(\boldsymbol{\mu}^{k},\boldsymbol{w}^{k},\boldsymbol{x}^{k})\}$. Let $\{(\boldsymbol{\mu}^{k_l},\boldsymbol{w}^{k_l},\boldsymbol{x}^{k_l})\}_{l\geq 0}$ be a sub-sequence of $\{(\boldsymbol{\mu}^{k},\boldsymbol{w}^{k},\boldsymbol{x}^{k})\}$ which converges to $(\boldsymbol{\mu}^{\infty},\boldsymbol{w}^{\infty},\boldsymbol{x}^{\infty})$. By \cite[Proposition 4.1]{doi:10.1007/s10589-016-9864-7}, it holds that
        \begin{equation}
            \label{eq:pro_4.1_1}
            \boldsymbol{u}^{k}\in\partial f(\boldsymbol{w}^k),\ \boldsymbol{v}^{k}\in\partial g(\boldsymbol{x}^k),
        \end{equation}
        and
        \begin{equation}
            \label{eq:pro_4.1_2}
            \begin{aligned}
                & \boldsymbol{\Psi}_{k} - \boldsymbol{\Psi}_{k+1}&\geq&\ \ 2\|\boldsymbol{w}^{k+1}_e\|^2_{\boldsymbol{\Sigma}_f}+2\|\boldsymbol{x}^{k+1}_e\|^2_{\boldsymbol{\Sigma}_g}+\|\boldsymbol{w}^{k+1}-\boldsymbol{w}^{k}\|^2_{\mathcal{H}}+\|\boldsymbol{x}^{k+1}-\boldsymbol{x}^{k}\|^2_{\mathcal{Q}}\\
                & & &\ \ +\textbf{\textit{min}}(1,1-\eta+\eta^{-1})\rho\|\mathcal{B}^*\boldsymbol{w}^{k+1}_e+\mathcal{I}^*\boldsymbol{x}^{k+1}_e\|^2\\ 
                & & &\ \ + \textbf{\textit{min}}(\eta,1+\eta-\eta^2)\rho\|\mathcal{I}^*(\boldsymbol{x}^{k+1}-\boldsymbol{x}^{k})\|^2,
            \end{aligned}
        \end{equation}
        where $\boldsymbol{\Sigma}_f$. Then we have
        \begin{equation}
            \label{eq:bounded}
            \sum_{k=0}^{\infty}\|\boldsymbol{w}^{k+1}_e\|^2_{\boldsymbol{\Sigma}_f},\ \sum_{k=0}^{\infty}\|\boldsymbol{x}^{k+1}_e\|^2_{\boldsymbol{\Sigma}_g},\ \sum_{k=0}^{\infty}\|\mathcal{B}^*\boldsymbol{w}^{k}_e+\mathcal{I}^*\boldsymbol{x}^{k}_e\|^2,\ \sum_{k=0}^{\infty}\|\mathcal{B}^*(\boldsymbol{x}^{k+1}-\boldsymbol{x}^{k})\|^2,\ \sum_{k=0}^{\infty}\|\boldsymbol{w}^{k+1}-\boldsymbol{w}^{k}\|^2_{\mathcal{H}},\ \sum_{k=0}^{\infty}\|\boldsymbol{x}^{k+1}-\boldsymbol{x}^{k}\|^2_{\mathcal{Q}}
        \end{equation}
        are all bounded.
        By taking limits in \eqref{eq:pro_4.1_1} along with $K_l$ for $l\rightarrow\infty$ and using \eqref{eq:spADMM} and \label{eq:bounded}, we have
        \begin{equation}
            \label{eq:KKT-2}
            -\mathcal{B}^*\boldsymbol{\mu}^\infty\in\partial f(\boldsymbol{w}^\infty),\ \ -\mathcal{I}^*\boldsymbol{\mu}^\infty\in\partial g(\boldsymbol{x}^\infty),\ \ \mathcal{B}^*\boldsymbol{w}^\infty+\mathcal{I}^*\boldsymbol{x}^\infty=\boldsymbol{0},
        \end{equation}
        which implies that $(\boldsymbol{\mu}^{\infty},\boldsymbol{w}^{\infty},\boldsymbol{x}^{\infty})$ is also a solution to the KKT system \eqref{eq:KKT}. Without lose of generality, let $(\boldsymbol{\hat{\mu}},\boldsymbol{\hat{w}},\boldsymbol{\hat{x}}) = (\boldsymbol{\mu}^{\infty},\boldsymbol{w}^{\infty},\boldsymbol{x}^{\infty})$. Then by \cite[Theorem 4.1 (a)]{doi:10.1007/s10589-016-9864-7}, we know that $\{\boldsymbol{\Psi}_{k}\}$ converges to zeros if $\eta\in(0,(1+\sqrt{5})/2)$ and $\{\boldsymbol{\Phi}_{k}\}$ converges to zeros if $\eta\geq(1+\sqrt{5})/2$ but 
        \begin{equation}
            \sum^{\infty}_{k=0}\|\boldsymbol{\mu}^{k+1}-\boldsymbol{\mu}^{k}\|^2<\infty.
        \end{equation}
        Thus, it holds that
        \begin{equation}
            \begin{aligned}
                & \mathcal{B}^*\boldsymbol{x}^k&\rightarrow&\ \ \mathcal{B}^*\boldsymbol{x}^{\infty}\\[5pt]
                & (\boldsymbol{\Sigma}_f+\mathcal{H})\boldsymbol{w}^{k}&\rightarrow&\ \ (\boldsymbol{\Sigma}_f+\mathcal{H})\boldsymbol{w}^{\infty}\\[5pt]
                & (\boldsymbol{\Sigma}_g+\mathcal{Q})\boldsymbol{x}^{k}&\rightarrow&\ \ (\boldsymbol{\Sigma}_g+\mathcal{Q})\boldsymbol{x}^{\infty},
            \end{aligned}
        \end{equation}
        when $k\rightarrow\infty$. Moreover, by using the facts that
        \begin{equation}
            \begin{aligned}
                &\mathcal{B}^*\boldsymbol{w}^k&=&\ \ (\mathcal{B}^*\boldsymbol{w}^k+\mathcal{I}^*\boldsymbol{x}^k)-\mathcal{I}^*\boldsymbol{x}^k,\\[5pt]
                &\mathcal{B}^*\boldsymbol{w}^k+\mathcal{I}^*\boldsymbol{x}^k&\rightarrow&\ \ \mathcal{B}^*\boldsymbol{w}^\infty+\mathcal{I}^*\boldsymbol{x}^\infty=c,
            \end{aligned}
        \end{equation}
        we get 
        \begin{equation}
            \mathcal{B}^*\boldsymbol{w}^k\rightarrow\mathcal{B}^*\boldsymbol{w}^\infty,\ k\rightarrow\infty.
        \end{equation}
        The above results tell us that any accumulation point of the sequence $\{(\boldsymbol{\mu}^{k},\boldsymbol{w}^{k},\boldsymbol{x}^{k})\}$ is a solution to the KKT system \eqref{eq:KKT}. As \eqref{eq:first_order_condition_1} is the first order optimal condition of \textbf{HodgeRank} and $\mathcal{T}_{\mathcal{R}}$ \eqref{opt:reg_hodge} is a strongly convex problem, $\boldsymbol{\theta}_{\mathcal{A}}$ will be the local optimal solution and then the global optimal solution. 

    \end{proof}

    \section*{Proof of Theorem \ref{thm:reversible_matrix} and \ref{thm:reversible_bound}}
    To prove Theorem \ref{thm:reversible_matrix}, we need the following lemma. 

    \begin{lemma}
        For any stochastic matrix $\boldsymbol{P}_{\mathcal{A}}$ which is reversible \textit{w.r.t} $\boldsymbol{\theta}_{\mathcal{A}}=\{\hat{\theta}_1,\hat{\theta}_2,\dots,\hat{\theta}_n\}\in\mathbb{R}^n$ as $0<\hat{\theta}_1\leq\hat{\theta}_2\leq\dots\leq\hat{\theta}_n$, we have
        \begin{equation}
            \left\langle\boldsymbol{P}_{\mathcal{A}}\boldsymbol{\delta}_1,\ \boldsymbol{\delta}_1\right\rangle_{\boldsymbol{\theta}_{\mathcal{A}}}\geq -(\hat{\theta}_1)^2,
        \end{equation}
        where
        \begin{equation*}
            \boldsymbol{\delta}_1 = \boldsymbol{e}_1-\|\boldsymbol{e}_1\|_{\boldsymbol{\theta}_{\mathcal{A}}}\cdot\boldsymbol{1},
        \end{equation*}
        $\boldsymbol{e}_1 = (1,0,\dots,0)^\top$, $\boldsymbol{1} = (1,1,\dots,1)^\top$, and $\|\cdot\|_{\boldsymbol{\theta}}$ means
        \begin{equation*}
             \|\boldsymbol{x}\|_{\boldsymbol{\theta}} = \langle \boldsymbol{x} , \boldsymbol{\theta} \rangle = \sum_{i=1}^{n}x_i\theta_i, 
        \end{equation*} 
        $\langle \cdot , \cdot \rangle_{\boldsymbol{\theta}}$ stands for the weighted inner product of two vectors with respect to $\boldsymbol{\theta}$
        \begin{equation*}
            \left\langle\boldsymbol{x},\ \boldsymbol{y}\right\rangle_{\boldsymbol{\theta}} = \boldsymbol{y}^\top\boldsymbol{\Theta}\ \boldsymbol{x}=\Big\langle \boldsymbol{\Theta}^{\frac{1}{2}}\boldsymbol{x},\ \boldsymbol{\Theta}^{\frac{1}{2}}\boldsymbol{y}\Big\rangle = \sum_{i=1}^nx_iy_i\theta_i,
        \end{equation*}
        $\boldsymbol{\Theta}=\textbf{diag}(\boldsymbol{\theta})$ is a diagonal matrix with $\boldsymbol{\theta} \in \mathbb{R}^n$ being the diagonal components. Moreover, the equality holds if and only if the element in first row and first column is $0$ as $\boldsymbol{P}_{\mathcal{A}}=\{\hat{P}_{ij}\}$
        \begin{equation}
            \hat{P}_{11}=0.
        \end{equation}
    \end{lemma}

    \begin{proof}
        By the definition of reversibility and row stochasticity, we have
        \begin{equation}
            \hat{\theta}_i\hat{P}_{ij} = \hat{\theta}_j\hat{P}_{ji},\ \forall\ i,\ j \in[n],\ i\neq j
        \end{equation}
        and
        \begin{equation}
            \sum_{j=1}^n \hat{P}_{ij} = 1.
        \end{equation}
        Then 
        \begin{equation}
            \begin{aligned}
                & \boldsymbol{P}_{\mathcal{A}}\boldsymbol{\delta}_1 &=&\ \ \boldsymbol{P}_{\mathcal{A}}\big[1-\hat{\theta}_1,\ -\hat{\theta}_1,\ \dots,\ -\hat{\theta}_1\big]^\top\\
                & &=&\ \ \left[\hat{P}_{11}-\hat{\theta}_1,\ \frac{\hat{\theta}_1}{\hat{\theta}_2}\hat{P}_{12}-\hat{\theta}_1,\ \dots,\ \frac{\hat{\theta}_1}{\hat{\theta}_n}\hat{P}_{1n}-\hat{\theta}_1\right]^\top.
            \end{aligned}
        \end{equation}
        Furthermore,
        \begin{equation}
            \begin{aligned}
                & \left\langle\boldsymbol{P}_{\mathcal{A}}\boldsymbol{\delta}_1,\ \boldsymbol{\delta}_1\right\rangle_{\boldsymbol{\theta}_{\mathcal{A}}} &=&\ \ \hat{\theta}_1\left(\hat{P}_{11}-\hat{\theta}_1\right)\left(1-\hat{\theta}_1\right)-\hat{\theta}_1\sum_{j=2}^n\left(\frac{\hat{\theta}_1}{\hat{\theta}_j}\hat{P}_{1j}-\hat{\theta}_1\right)\hat{\theta}_j\\
                & &=&\ \  \hat{\theta}_1\left(\hat{P}_{11}-\hat{\theta}_1\right)\\[3.5pt]
                & &\geq& -(\hat{\theta}_1)^2,
            \end{aligned}
        \end{equation}
        where the equality holds if and only if $(\boldsymbol{P}_{\mathcal{A}})_{11}=0$. 
    \end{proof}

    \reversiblematrix*

    \begin{proof}
        Suppose $\lambda_2$ is the second largest eigenvalue of $\boldsymbol{P}_{\mathcal{A}}$, it satisfies
        \begin{equation}
            \label{eq:eigenvalue_sup}
            \lambda_2 = \underset{\boldsymbol{x}\neq\boldsymbol{0},\ \|\boldsymbol{x}\|_{\boldsymbol{\theta}_{\mathcal{A}}}=0}{\text{sup}}\ \frac{\ \left\langle \boldsymbol{P}_{\mathcal{A}}\boldsymbol{x},\ \boldsymbol{x}\right\rangle_{\boldsymbol{\theta}_{\mathcal{A}}}}{\langle\boldsymbol{x},\ \boldsymbol{x}\rangle_{\boldsymbol{\theta}_{\mathcal{A}}}}.
        \end{equation}
        With the above lemma and the fact 
        \begin{equation}
            \left\langle\boldsymbol{\delta}_1,\ \boldsymbol{\delta}_1\right\rangle_{\boldsymbol{\theta}_{\mathcal{A}}} = \hat{\theta}_1(1-\hat{\theta}_1), 
        \end{equation}
        it follows that
        \begin{equation}
            \begin{aligned}
                & \lambda_2 &=&\ \ \underset{\boldsymbol{x}\neq\boldsymbol{0},\ \|\boldsymbol{x}\|_{\boldsymbol{\theta}_{\mathcal{A}}}=0}{\text{sup}}\ \frac{\ \left\langle \boldsymbol{P}_{\mathcal{A}}\boldsymbol{x},\ \boldsymbol{x}\right\rangle_{\boldsymbol{\theta}_{\mathcal{A}}}}{\langle\boldsymbol{x},\ \boldsymbol{x}\rangle_{\boldsymbol{\theta}_{\mathcal{A}}}}\\
                & &\geq&\ \ \frac{\ \left\langle \boldsymbol{P}_{\mathcal{A}}\boldsymbol{\delta}_1,\ \boldsymbol{\delta}_1\right\rangle_{\boldsymbol{\theta}_{\mathcal{A}}}}{\langle\boldsymbol{\delta}_1,\ \boldsymbol{\delta}_1\rangle_{\boldsymbol{\theta}_{\mathcal{A}}}}\\
                & &=&\ \ \frac{-(\hat{\theta}_1)^2}{\hat{\theta}_1(1-\hat{\theta}_1)} = \frac{-\hat{\theta}_1}{1-\hat{\theta}_1}.
            \end{aligned}
        \end{equation}
        Now we establish the first part of the results. 

        Now we torn to the second part. If $\lambda_2$ archives its lower bound, by the above lemma, we know that $(\boldsymbol{P}_{\mathcal{A}})_{11}=0$. Furthermore, by the self-adjoint property of $\boldsymbol{P}_{\mathcal{A}}$ \textit{w.r.t} $\langle\cdot,\ \cdot\rangle_{\boldsymbol{\theta}_{\mathcal{A}}}$, the supremum in \eqref{eq:eigenvalue_sup} can be attained by letting $\boldsymbol{x}$ be the eigenvector corresponding to $\lambda_2$. It implies that 
        \begin{equation}
            \lambda_2 = \underset{\boldsymbol{x}\neq\boldsymbol{0},\ \|\boldsymbol{x}\|_{\boldsymbol{\theta}_{\mathcal{A}}}=0}{\text{sup}}\ \frac{\ \left\langle \boldsymbol{P}_{\mathcal{A}}\boldsymbol{x},\ \boldsymbol{x}\right\rangle_{\boldsymbol{\theta}_{\mathcal{A}}}}{\langle\boldsymbol{x},\ \boldsymbol{x}\rangle_{\boldsymbol{\theta}_{\mathcal{A}}}} = \frac{\ \left\langle \boldsymbol{P}_{\mathcal{A}}\boldsymbol{\delta}_1,\ \boldsymbol{\delta}_1\right\rangle_{\boldsymbol{\theta}_{\mathcal{A}}}}{\langle\boldsymbol{\delta}_1,\ \boldsymbol{\delta}_1\rangle_{\boldsymbol{\theta}_{\mathcal{A}}}} = \frac{-\hat{\theta}_1}{1-\hat{\theta}_1}
        \end{equation}
        and $\boldsymbol{\delta}_1$ is a corresponding eigenvector of $\lambda_2$:
        \begin{equation}
            \boldsymbol{P}_{\mathcal{A}}\boldsymbol{\delta}_1 = \frac{-\hat{\theta}_1}{1-\hat{\theta}_1}\boldsymbol{\delta}_1.
        \end{equation}
        By the definition of $\boldsymbol{\delta}_1$, we have
        \begin{equation}
            \begin{pmatrix}
            0 & (\boldsymbol{P}_{\mathcal{A}})_{12} & \dots & (\boldsymbol{P}_{\mathcal{A}})_{1n}\\[5pt]
            (\boldsymbol{P}_{\mathcal{A}})_{21} & (\boldsymbol{P}_{\mathcal{A}})_{22} & \dots & (\boldsymbol{P}_{\mathcal{A}})_{2n}\\[5pt]
            \vdots & \vdots & \ddots & \vdots \\[5pt]
            (\boldsymbol{P}_{\mathcal{A}})_{n1} & (\boldsymbol{P}_{\mathcal{A}})_{n2} & \dots & (\boldsymbol{P}_{\mathcal{A}})_{nn}
            \end{pmatrix}
            \begin{pmatrix}
                1 - \hat{\theta}_1\\[5pt]
                  - \hat{\theta}_1\\[5pt]
                            \vdots\\[5pt]
                  - \hat{\theta}_1
            \end{pmatrix} 
            = -\frac{-\hat{\theta}_1}{1-\hat{\theta}_1}
            \begin{pmatrix}
                1 - \hat{\theta}_1\\[5pt]
                  - \hat{\theta}_1\\[5pt]
                            \vdots\\[5pt]
                  - \hat{\theta}_1
            \end{pmatrix}.
        \end{equation}
        With the reversibility, we knows that $(\boldsymbol{P}_{\mathcal{A}})_{1j} = \hat{\theta}_j/\hat{\theta}_1(\boldsymbol{P}_{\mathcal{A}})_{j1}$. Then we can check that
        \begin{equation}
            (\boldsymbol{P}_{\mathcal{A}})_{1j} = \frac{\hat{\theta}_j}{1-\hat{\theta}_1},\ \forall\ j>1.
        \end{equation}
        By the detailed balance condition, we also know that
        \begin{equation}
            (\boldsymbol{P}_{\mathcal{A}})_{j1} = \frac{\hat{\theta}_1}{1-\hat{\theta}_1},\ \forall\ j>1.
        \end{equation}
        It is noteworthy that we do not adopt the fact that $\hat{\theta}_1$ has the lowest probability. As a consequence, it holds that
        \begin{equation}
            \left\langle \boldsymbol{P}_{\mathcal{A}}\boldsymbol{\delta}_k,\ \boldsymbol{\delta}_k\right\rangle_{\boldsymbol{\theta}_{\mathcal{A}}}\geq -(\hat{\theta}_k)^2, 
        \end{equation}
        where $\boldsymbol{\delta}_k = \boldsymbol{e}_k-\|\boldsymbol{e}_k\|_{\boldsymbol{\theta}_{\mathcal{A}}}\cdot\boldsymbol{1}$. 
        Therefore, $\boldsymbol{P}_{\mathcal{A}}$ has this form:
        \begin{equation*}
            \begingroup 
            \renewcommand\arraystretch{1}
            \setlength\arraycolsep{3.5pt}
            \mbox{\normalsize\(\boldsymbol{P}_{\mathcal{A}} = 
            \begin{pmatrix}
                0 & \frac{\displaystyle \hat{\theta}_2}{\displaystyle 1-\hat{\theta}_1} & \cdots & \frac{\displaystyle \hat{\theta}_n}{\displaystyle 1-\hat{\theta}_1}\\
                \frac{\displaystyle \hat{\theta}_1}{\displaystyle 1-\hat{\theta}_1} &  &  & \\
                \vdots &  &\boldsymbol{P}_{\mathcal{A}_2}  & \\
                \frac{\displaystyle \hat{\theta}_1}{\displaystyle 1-\hat{\theta}_1} &  &  & 
            \end{pmatrix}
            \)},
            \endgroup
        \end{equation*}
        where $\boldsymbol{P}_{\mathcal{A}_2}$ is also a reversible \textit{w.r.t} $\{\hat{\theta}_2,\dots,\hat{\theta}_n\}$, row stochastic matrix.
    \end{proof}

    Based on the above results, we show the specific formula of $\boldsymbol{P}_{\mathcal{A}}$ can help to establish the fixed point of the corresponding adversarial game.
    
    \reversibilefinal*

    Here $\{\boldsymbol{P}^{(k)}\},\ k=n-1,\dots,1$ is recursively defined as 
    \begin{equation*}
    \begingroup 
    \renewcommand\arraystretch{1}
    \setlength\arraycolsep{2.75pt}
    \mbox
    {
        \normalsize
        \(
            \boldsymbol{P}^{(k)} \coloneqq 
            \begin{pmatrix}
                0 & \frac{\displaystyle\ \ \theta^{(k)}_{k+1}\ \ }{\displaystyle\ \ \beta_{k+1}\ \ } & \cdots & \frac{\displaystyle\ \ \theta^{(k)}_{n}\ \ }{\displaystyle\ \ \beta_{k+1}\ \ }\\[7.5pt]
                1-\alpha_{k+1} &  &  & \\[5pt]
                \vdots &  &\alpha_{k+1}\cdot\boldsymbol{P}^{(k+1)}  & \\[5pt]
                1-\alpha_{k+1} &  &  & 
            \end{pmatrix}\in\mathbb{R}^{(n-k+1)\times(n-k+1)},
        \)
    }
    \endgroup
    \end{equation*}
    and
    \begin{subequations}
        \begin{align*}
            \boldsymbol{P}^{(n)}\ \ &\coloneqq\ \ 0\\[5pt]
            \boldsymbol{\theta}^{(1)}\ \ &\coloneqq\ \ \boldsymbol{\theta}_{\mathcal{A}}\\[5pt]
            \boldsymbol{\theta}^{(k+1)}\ &\coloneqq\ \ \boldsymbol{\theta}^{(k)}\ /\ \beta_{k+1}\\[5pt]
            \beta_{k+1}\ &\coloneqq\ \ 1\ -\ \theta^{(k)}_{k}\\[5pt]
            \alpha_{k+1}\ &\coloneqq\ \ 1\ -\ \theta^{(k)}_{k}\ /\ \beta_{k+1}.
        \end{align*}
    \end{subequations}

    \begin{proof}
        With Theorem \ref{thm:reversible_matrix}, we know that $\boldsymbol{P}^{(2)}$ satisfies the detailed balance condition with the vector $\boldsymbol{\theta}_{\mathcal{A}}/\hat{\theta}_1$ as 
        \begin{equation*}
            (\boldsymbol{\theta}_{\mathcal{A}}/\hat{\theta}_1)_i P^{(2)}_{ij} = (\boldsymbol{\theta}_{\mathcal{A}}/\hat{\theta}_1)_j P^{(2)}_{ji},\ \ i,\ j\in[n-1],
        \end{equation*}
        where $\boldsymbol{\theta}_{\mathcal{A}}/\hat{\theta}_1 = \{\hat{\theta}_2,\dots,\hat{\theta}_n\}^{\top}$ and $\boldsymbol{P}^{(2)}$ has the constant row sum. If $\boldsymbol{x}_{2:n} = (x_2,\dots,x_n)^{\top}$ is the eigenvector of $\boldsymbol{P}^{(2)}$, we know that $\boldsymbol{x}_2$ is orthogonal to $\boldsymbol{1}\in\mathbb{R}^{n-1}$ as the latter one is also an eigenvector of $\boldsymbol{P}^{(2)}$. Moreover, finding $\boldsymbol{x}_{2:n}$ is equivalent to find an eigenvector of $\boldsymbol{P}_{\mathcal{A}}$ which is orthogonal to $\boldsymbol{1}\in\mathbb{R}^n$ and $\boldsymbol{\delta}_1 = \boldsymbol{e}_1-\|\boldsymbol{e}_1\|_{\boldsymbol{\theta}_{\mathcal{A}}}\cdot\boldsymbol{1}$. This requires that 
        \begin{equation*}
            \begin{aligned}
                \boldsymbol{P}_{\mathcal{A}}
                \begin{bmatrix}
                    x_1\\
                    --- \\ 
                    \boldsymbol{x}_{2:n}
                \end{bmatrix}
                = \lambda \begin{bmatrix}
                    x_1\\ 
                    --- \\ 
                    \boldsymbol{x}_{2:n}
                \end{bmatrix}
                \hspace{-3.5em}
                \phantom{\begin{bmatrix}x_1\\ -- \\ \boldsymbol{x}_{2:n}\end{bmatrix}}\hspace{-0.5em}
                \begin{tabular}{l}
                $\left.\lefteqn{\phantom{\mathcal{A}^{(k)}}}\right\}1$\\[5pt]
                $\left.\lefteqn{\phantom{\mathcal{A}^{(k)}}}\right\}n-1$
                \end{tabular}
                ,\\[3pt]
                \boldsymbol{1}^\top
                \begin{bmatrix}
                    x_1\\ 
                    --- \\ 
                    \boldsymbol{x}_{2:n}
                \end{bmatrix} = 0,\ \ 
                \boldsymbol{\delta}^{\top}_1
                \begin{bmatrix}
                    x_1\\ 
                    --- \\ 
                    \boldsymbol{x}_{2:n}
                \end{bmatrix} = 0.
            \end{aligned}
        \end{equation*}
        So $x_1$ must be $0$. Therefore, the second eigenvalue of $\boldsymbol{P}^{(2)}$ is the third eigenvalue of $\boldsymbol{P}_{\mathcal{A}}$ and the first row and column are defined by the above-mentioned iteration rules. Obviously, this procedure could be conducted further, thus obtaining at each step a new eigenvalue that fixes a new row and column in the matrix $\boldsymbol{P}_{\mathcal{A}}$. Finally, we remark that the eigenvalues of $\boldsymbol{P}_{\mathcal{A}}$ are
        \begin{equation}
            \lambda_{k+1} = -\frac{\hat{\theta}_k}{\hat{\theta}_{k+1}+\dots+\hat{\theta}_n}\cdot\left.\prod_{l=1}^{k-1}\left(1-\frac{\hat{\theta}_l}{\hat{\theta}_{l+1}+\dots+\hat{\theta}_n}\right)\right.,\ \ k=2,\dots,n.
        \end{equation}
        As a consequence, the last entry of $\boldsymbol{P}_{\mathcal{A}}=\{\hat{P}_{ij}\}$ will be
        \begin{equation}
            \hat{P}_{nn} = 1 - \sum_{j=1}^{n-1}P^{(1)}_{nj}.
        \end{equation}
    \end{proof}

    \reversiblebound*

    \begin{proof}
        First, we show that for any initial distribution
        \begin{equation}
            \begin{aligned}
                & \nu(f,\boldsymbol{P}_{\mathcal{A}}) &=&\ \ \underset{m\rightarrow\infty}{\textbf{\textit{lim}}}\ m\textbf{\textit{Var}}\left(\frac{1}{m}\sum_{k=0}^{m-1}f(\boldsymbol{v}_k)\right)\\[5pt]
                & &=&\ \ \left\langle\left(\boldsymbol{I}-\boldsymbol{P}_{\mathcal{A}}\right)^{-1}\left(\boldsymbol{I}+\boldsymbol{P}_{\mathcal{A}}\right)\left(f-\sum_{i=1}^{n}f(i)\hat{\theta}_i\cdot\boldsymbol{1}\right),\ f-\sum_{i=1}^{n}f(i)\hat{\theta}_i\cdot\boldsymbol{1}\right\rangle_{\boldsymbol{\theta}_{\mathcal{A}}}.
            \end{aligned}
        \end{equation}
        See \cite{peskun1973optimum} for the detailed proof. The above equality implies that
        \begin{equation}
            \nu(f,\boldsymbol{P}_{\mathcal{A}}) = \sum_{k=2}^{n}\frac{1+\lambda_i}{1-\lambda_i}\left\langle f, \boldsymbol{\delta}_k\right\rangle^2_{\boldsymbol{\theta}_{\mathcal{A}}},
        \end{equation}
        where $\boldsymbol{\delta}_k$ satisfies 
        \begin{equation}
            \boldsymbol{P}_{\mathcal{A}}\boldsymbol{\delta}_k = \lambda_k\boldsymbol{\delta}_k,\ \ \|\boldsymbol{\delta}_k\|_{\boldsymbol{\theta}_{\mathcal{A}}}=1.
        \end{equation}
        This formulation links the asymptotic variance $\nu(f,\boldsymbol{P}_{\mathcal{A}})$ to the eigensystem of $\boldsymbol{P}_{\mathcal{A}}$. Moreover, we know that 
        \begin{equation}
            \label{eq:tttt}
            \textbf{\textit{max}}\ \nu(f,\boldsymbol{P}_{\mathcal{A}}) = \frac{1+\lambda_2}{1-\lambda_2}.
        \end{equation}
        Substituting the value of $\lambda_2 \geq -\hat{\theta}_1/(1-\hat{\theta}_1)$ into \eqref{eq:tttt}, we obtain the minimum possible value of $\nu(f,\boldsymbol{P}_{\mathcal{A}})$.
    \end{proof}

    \section*{Proof of Theorem \ref{thm:irrevisible_matrix} and \ref{thm:irreversible_bound}}

    To prove these theorems, we start with some notations and lemmas. Let $\boldsymbol{P}_{\mathcal{A}}=\{(\hat{P}_{ij}\}$, $i,\ j\in[n],\ i\neq j$ be a transition matrix with an invariant distribution $\boldsymbol{\theta}_{\mathcal{A}}=\{\hat{\theta}_1,\dots,\hat{\theta}_n\}$, and $\boldsymbol{X}=\{X_t:t\geq0\}$ be a Markov chain generated by $\boldsymbol{P}_{\mathcal{A}}$. We define the adjoint transition matrix $\boldsymbol{P}^*_{\mathcal{A}}=\{\hat{P}^*_{ij}\}$ of $\boldsymbol{P}_{\mathcal{A}}$ as
    \begin{equation}
        \hat{P}^*_{ij} = \frac{\ \hat{\theta}_j}{\hat{\theta}_i}\hat{P}_{ji},\ \ i,\ j\in[n],\ i\neq j,
    \end{equation}
    which is the adjoint of $\boldsymbol{P}_{\mathcal{A}}$ under the inner product $\langle\cdot,\cdot\rangle_{\boldsymbol{\theta}_{\Scale[0.5]{\mathcal{A}}}}$. The dual chain of $\boldsymbol{X}$, denoted as $\boldsymbol{X}^*$, is the Markov chain generated by $\boldsymbol{P}^*_{\mathcal{A}}$. Note the first hitting time to candidate $i$ of chain $\boldsymbol{X}$ is
    \begin{equation}
        T_i = \underset{}{\inf}\left\{t\geq 0\ \Big|\ X_t=i\right\}.
    \end{equation}
    Similarly, we could define $T^*_i$ of $\boldsymbol{X}^*$. $\mathbb{P}_{\mu}$ and $\mathbb{E}_{\mu}$ denote the probability and expectation for chains started from probability distribution $\mu$. The fundamental matrix $\boldsymbol{Z}=\{Z_{ij}\}$ of $\boldsymbol{P}_{\mathcal{A}}$ is 
    \begin{equation*}
         \boldsymbol{Z} = (\boldsymbol{I}-\boldsymbol{P}_{\mathcal{A}})^{-1},
    \end{equation*}
    where $\boldsymbol{I}$ is the $n$-dimensional identity matrix.

    The following lemma relates the fundamental matrix with the hitting times.

    \begin{lemma}
        For any pair of points $i,j\in[n],\ i\neq j$, we have
        \begin{equation}
            \begin{aligned}
                & \hat{\theta}_i\cdot\mathbb{E}_{\boldsymbol{\theta}_{\Scale[0.5]{\mathcal{A}}}}[T_i] &=&\ \ Z_{ii},\\
                & \hat{\theta}_j\cdot\mathbb{E}_{i}[T_j] &=&\ \ Z_{jj} - Z_{ij},
            \end{aligned}
        \end{equation}
        where $\mathbb{E}_i$ represents the expectation of the chain $\boldsymbol{X}$ started at item (state) $i$ and time $t=0$.
    \end{lemma}

    This lemma can be found in \cite{aldous-fill-2014} as Lemma 11 and 12 of Chapter 2. With this lemma, we could establish the lower bound of $\nu(\boldsymbol{P}_{\mathcal{A}})$.

    \begin{proposition}
        \label{prop:1}
        Suppose that $\boldsymbol{P}_{\mathcal{A}}$ is an irreducible transition matrix with invariant distribution $\boldsymbol{\theta}_{\mathcal{A}}$ and $\nu(\boldsymbol{P}_{\mathcal{A}})$ is defined as 
        \begin{equation*}
            \begin{aligned}
                &v(\boldsymbol{P}_{\mathcal{A}})&:=&\ \ \underset{\boldsymbol{f}\in\mathcal{F},\ \|\boldsymbol{f}\|_{\Scale[0.5]{\boldsymbol{\theta}_{\mathcal{A}}}}=1}{\textbf{\textit{sup}}} \Big\langle\big(\boldsymbol{I}-\boldsymbol{P}_{\mathcal{A}}\big)^{-1}\boldsymbol{f},\ \boldsymbol{f}\Big\rangle_{\boldsymbol{\theta}_{\mathcal{A}}}\\
                & &=&\ \ 2\big\langle\boldsymbol{Z}\boldsymbol{f},\boldsymbol{f}\big\rangle_{\boldsymbol{\theta}_{\Scale[0.45]{\mathcal{A}}}}-1,
            \end{aligned}
        \end{equation*}
        where $\mathcal{F}$ is the space of mean zero function as
        \begin{equation*}
            \mathcal{F} :=\Big\{\boldsymbol{f}\in\mathbb{R}^n\ \Big\vert\ \mathbb{E}_{\boldsymbol{\theta}_{\mathcal{A}}}[\boldsymbol{f}] = 0\Big\},
        \end{equation*}
        we have
        \begin{equation}
            \label{eq:prop_1_result}
            v(\boldsymbol{P}_{\mathcal{A}}) \geq \underset{i,j\in[n]}{\max}\ \frac{\hat{\theta}_i\hat{\theta}_j}{\hat{\theta}_i+\hat{\theta}_j}\cdot\Big(\mathbb{E}_{i}[T_j]+\mathbb{E}_{j}[T_i]\Big).
        \end{equation}
    \end{proposition}

    \begin{proof}
        Choose an unordered pair $[i, j], i, j\in[n], i<j$ and denote 
        \begin{equation}
            \label{eq:f_ij}
            \boldsymbol{f}_{ij} = \left(0,\ \dots,0,\ \frac{1}{\hat{\theta}_i},\ 0,\ \dots,\ 0,\ -\frac{1}{\hat{\theta}_j},\ 0,\ \dots,\ 0\right)^{\top},
        \end{equation}
        by the above lemma, it holds that
        \begin{equation*}
            \begin{aligned}
                & \Big\langle \boldsymbol{Z}\boldsymbol{f}_{ij}, \boldsymbol{f}_{ij}\Big\rangle_{\boldsymbol{\theta}_{\Scale[0.5]{\mathcal{A}}}} &=&\ \ \left(\frac{1}{\hat{\theta}_i}z_{ii}-\frac{1}{\hat{\theta}_j}z_{ij}\right)-\left(\frac{1}{\hat{\theta}_i}z_{ji}-\frac{1}{\hat{\theta}_j}z_{ij}\right)\\[5pt]
                & &=&\ \ \frac{1}{\hat{\theta}_i}(z_{ii}-z_{ji})+\frac{1}{\hat{\theta}_j}(z_{jj}-z_{ij})\\[5pt]
                & &=&\ \ \mathbb{E}_j[T_i] + \mathbb{E}_i[T_j].
            \end{aligned}
        \end{equation*}
        Moreover, the normalization of $\boldsymbol{f}_{ij}$ satisfies
        \begin{equation}
            \label{eq:inter_result}
            \left\langle \boldsymbol{Z}\frac{\boldsymbol{f}_{ij}}{\|\boldsymbol{f}_{ij}\|_{\boldsymbol{\theta}_{\mathcal{A}}}}, \frac{\boldsymbol{f}_{ij}}{\|\boldsymbol{f}_{ij}\|_{\boldsymbol{\theta}_{\mathcal{A}}}}\right\rangle_{\boldsymbol{\theta}_{\Scale[0.5]{\mathcal{A}}}} = \frac{\hat{\theta}_i\hat{\theta}_j}{\hat{\theta}_i+\hat{\theta}_j}\cdot\Big(\mathbb{E}_j[T_i] + \mathbb{E}_i[T_j]\Big).
        \end{equation}
        Taking the maximum operation with respect to $i$ and $j$, we prove the lower bound of $v(\boldsymbol{P}_{\mathcal{A}})$.  
    \end{proof}

    The following lemma shows a lower bound of the expectation part in the right hand of \eqref{eq:inter_result}, which can be also found in \cite[Chapter 2, Corollary 8]{aldous-fill-2014}.

    \begin{lemma}
        \label{lemma:1}
        Let $\boldsymbol{P}_{\mathcal{A}}$ be an irreducible transition matrix with invariant distribution $\boldsymbol{\theta}_{\mathcal{A}}$. For any unordered pair $[i, j], i, j\in[n], i<j$,
        \begin{equation}
            \hat{\theta}_j\Big(\mathbb{E}_j[T_i] + \mathbb{E}_i[T_j]\Big)  = \frac{1}{\mathbb{P}_j(T_i<T_j^+)},
        \end{equation}
        where $T_j^+$ denote the first return time to $j$ of chain $\boldsymbol{X}$
        \begin{equation*}
            T_j^+ = \underset{}{\inf}\left\{t\geq 1\ \Big\vert\  X_t = j\right\},
        \end{equation*}
        and $\mathbb{P}_i(\cdot)$ stands for the probability of the chain $\boldsymbol{X}$ started at item (state) $i$ and time $t=0$.
    \end{lemma}

    Now we can prove the following results.
    \irrevisiblematrix*

    \begin{proof}
        With the above lemmas and propositions, we have
        \begin{equation}
            \begin{aligned}
                &v(\boldsymbol{P}_{\mathcal{A}})&:=&\ \ \underset{\boldsymbol{f}\in\mathcal{F},\ \|\boldsymbol{f}\|_{\Scale[0.5]{\boldsymbol{\theta}_{\mathcal{A}}}}=1}{\textbf{\textit{sup}}} \Big\langle\big(\boldsymbol{I}-\boldsymbol{P}_{\mathcal{A}}\big)^{-1}\boldsymbol{f},\ \boldsymbol{f}\Big\rangle_{\boldsymbol{\theta}_{\mathcal{A}}}\\
                & &\geq&\ \ \underset{i,j\in[n]}{\max}\ \frac{\hat{\theta}_i\hat{\theta}_j}{\hat{\theta}_i+\hat{\theta}_j}\cdot\Big(\mathbb{E}_{i}[T_j]+\mathbb{E}_{j}[T_i]\Big)\\
                & &=&\ \ \underset{i,j\in[n]}{\max}\ \frac{\hat{\theta}_i}{\hat{\theta}_i+\hat{\theta}_j}\cdot\hat{\theta}_j\Big(\mathbb{E}_{i}[T_j]+\mathbb{E}_{j}[T_i]\Big)\\
                & &\geq&\ \ \underset{i,j\in[n]}{\max}\ \frac{\hat{\theta}_i}{\hat{\theta}_i+\hat{\theta}_j}\\
                & &=&\ \ \frac{\hat{\theta}_n}{\hat{\theta}_1+\hat{\theta}_n}. 
            \end{aligned}
        \end{equation}
        The first inequality holds by the Proposition \ref{prop:1}, the second one holds by the Lemma \ref{lemma:1} and the last equality holds with the fact
        \begin{equation*}
            0<\hat{\theta}_1\leq\hat{\theta}_2\leq\dots\leq\hat{\theta}_n<1.
        \end{equation*}
    \end{proof}

    Next, we give a specific formulation of $\boldsymbol{P}_{\mathcal{A}}$ with invariant distribution $\boldsymbol{\theta}_{\mathcal{A}}$, which can help establish the fixed point of the corresponding adversarial game.

    \irreversibilefinal

    \begin{equation}
        \hat{P}_{ij} = \left\{
        \begin{array}{rl}
            0,&1\leq j<i<n,\\[7.5pt]
            \frac{\displaystyle \hat{\theta}_i-\hat{\theta}_1}{\displaystyle \hat{\theta}_i-\hat{\theta}_n},&1\leq j=i<n,\\[7.5pt]
            \frac{\displaystyle\hat{\theta}_j\big(\hat{\theta}_1+\hat{\theta}_n\big)}{\displaystyle\hat{\theta}_n\big(\hat{\theta}_j+\hat{\theta}_n\big)}\ {\displaystyle\prod_{k=1}^{j-1}}\ \frac{\displaystyle\hat{\theta}_n-\hat{\theta}_k}{\displaystyle\hat{\theta}_k+\hat{\theta}_n},&1\leq j<n=i,\\[7.5pt]
            \frac{\displaystyle2\hat{\theta}_j\big(\hat{\theta}_1+\hat{\theta}_n\big)}{\displaystyle\big(\hat{\theta}_i+\hat{\theta}_n\big)\big(\hat{\theta}_j+\hat{\theta}_n\big)}{\displaystyle\prod_{k=i+1}^{j-1}}\frac{\displaystyle\hat{\theta}_n-\hat{\theta}_k}{\displaystyle\hat{\theta}_k+\hat{\theta}_n},&1\leq i<j<n,\\[7.5pt]
            1-{\displaystyle\sum_{k=1}^{n-1}}\hat{P}_{ik},&\text{otherwise.} 
        \end{array}
        \right.
        \tag{\ref{eq:irreversibiletransitionmatrix} revisited}
    \end{equation}

    \begin{proof}
    \begin{enumerate}
        \item To prove $\boldsymbol{P}_{\mathcal{A}}$ \eqref{eq:irreversibiletransitionmatrix} is irreducible, it is obvious as 
        \begin{equation}
            \hat{P}_{n1} = \frac{\hat{\theta}_1}{\hat{\theta}_n}>0,\ \ \hat{P}_{i,i+1} = \frac{2\hat{\theta}_{i+1}\big(\hat{\theta}_1+\hat{\theta}_n\big)}{\big(\hat{\theta}_i+\hat{\theta}_n\big)\big(\hat{\theta}_(i+1)+\hat{\theta}_n\big)}>0,\ \ \forall\ i\in[n].
        \end{equation}
        \item To prove the result, it suffices to show that
        \begin{equation}
            \hat{P}_{in}\geq 0,\ \forall\ i\in[n]
        \end{equation}
        or equivalently,
        \begin{equation}
            \sum_{j=1}^{n-1}\hat{P}_{ij}\leq 1,\ \forall\ i\in[n].
        \end{equation}
        For each $i<n$, we claim that
        \begin{equation}
            \label{eq:positive_results}
            1 - \sum_{k=1}^j\hat{P}_{ik} = \frac{\hat{\theta}_1+\hat{\theta}_n}{\hat{\theta}_i+\hat{\theta}_n}\prod_{k=i+1}^{j}\frac{\hat{\theta}_n-\hat{\theta}_k}{\hat{\theta}_n+\hat{\theta}_k},\ \ \forall\ i\leq j< n,
        \end{equation}
        and
        \begin{equation}
            1 - \hat{P}_{ii} = \frac{\hat{\theta}_1+\hat{\theta}_n}{\hat{\theta}_i+\hat{\theta}_n}.
        \end{equation}
        Suppose that the statement \eqref{eq:positive_results} holds for $j = {j}'$ with $i<{j}'<n-1$. Then for $j={j}'+1$,
        \begin{equation}
            \label{eq:ini}
            \begin{aligned}
                & 1 - \underset{k\leq {j}'+1}{\sum}\hat{P}_{ik} &=&\ \ 1 - \underset{k\leq {j}'}{\sum}\hat{P}_{ik} -\hat{P}_{i({j}'+1)}\\
                & &=&\ \ \frac{\hat{\theta}_1+\hat{\theta}_n}{\hat{\theta}_i+\hat{\theta}_n}\prod_{k=i+1}^{{j}'}\frac{\hat{\theta}_n-\hat{\theta}_k}{\hat{\theta}_n+\hat{\theta}_k}-\frac{2\hat{\theta}_{{j}'+1}(\hat{\theta}_1+\hat{\theta}_n)}{(\hat{\theta}_i+\hat{\theta}_n)(\hat{\theta}_{{j}'+1}+\hat{\theta}_n)}\prod_{k=i+1}^{{j}'}\frac{\hat{\theta}_n-\hat{\theta}_k}{\hat{\theta}_n+\hat{\theta}_k}\\[5pt]
                & &=&\ \ \frac{\hat{\theta}_1+\hat{\theta}_n}{\hat{\theta}_i+\hat{\theta}_n}\prod_{k=i+1}^{{j}'+1}\frac{\hat{\theta}_n-\hat{\theta}_k}{\hat{\theta}_n+\hat{\theta}_k}.
            \end{aligned}
        \end{equation}
        Therefore, \eqref{eq:positive_results} has been proved by the induction. When $i=n$, we observe that
        \begin{equation}
            \frac{\hat{P}_{nj}}{\hat{P}_{1j}} = \frac{\hat{\theta}_1+\hat{\theta}_n}{2\hat{\theta}_n}\cdot\frac{\hat{\theta}_n-\hat{\theta}_1}{\hat{\theta}_n+\hat{\theta}_1} = \frac{\hat{\theta}_n-\hat{\theta}_1}{2\hat{\theta}_n},\ \ \forall\ 1<j<n.
        \end{equation}
        Furthermore,
        \begin{equation}
            \label{eq:in_i}
            \begin{aligned}
                & \underset{j<n}{\sum}\hat{P}_{nj} &=&\ \ \hat{P}_{n1}+\sum_{j=2}^{n-1}\hat{P}_{nj}\\
                & &=&\ \ \frac{\hat{\theta}_1}{\hat{\theta}_n} + \frac{\hat{\theta}_n-\hat{\theta}_1}{2\hat{\theta}_n}\sum_{j=2}^{n-1}\hat{P}_{1j}\\[5pt]
                & &\leq&\ \ \frac{\hat{\theta}_1}{\hat{\theta}_n} + \frac{\hat{\theta}_n-\hat{\theta}_1}{2\hat{\theta}_n}\\[5pt]
                & &=&\ \ \frac{\hat{\theta}_1+\hat{\theta}_n}{2\hat{\theta}_n}\\[5pt]
                & &\leq&\ \ 1.
            \end{aligned}
        \end{equation}
        Combine \eqref{eq:ini} and \eqref{eq:in_i}, we have proved the non-negativity of $\boldsymbol{P}_{\mathcal{A}}$.
        \item To prove $\boldsymbol{P}_{\mathcal{A}}$ \eqref{eq:irreversibiletransitionmatrix} is invariant with $\boldsymbol{\theta}_{\mathcal{A}}$, we claim that
        \begin{equation}
            \label{eq:4claim}
            \hat{\theta}_n\hat{P}_{nj}+\sum_{k=1}^{i}\hat{\theta}_k\hat{P}_{kj}=\frac{\hat{\theta}_j(\hat{\theta}_1+\hat{\theta}_n)}{\hat{\theta}_j+\hat{\theta}_n}\prod_{k=i+1}^{j-1}\frac{\hat{\theta}_n-\hat{\theta}_k}{\hat{\theta}_n+\hat{\theta}_k},\ \ \forall\ i<j<n.
        \end{equation}
        When $i=1$, \eqref{eq:irreversibiletransitionmatrix} shows that
        \begin{equation}
            \begin{aligned}
                &\hat{\theta}_n\hat{P}_{nj}+\hat{\theta}_1\hat{P}_{1j}&=&\ \ \frac{\hat{\theta}_j(\hat{\theta}_1+\hat{\theta}_n)}{\hat{\theta}_j+\hat{\theta}_n}\prod_{k=1}^{j-1}\frac{\hat{\theta}_n-\hat{\theta}_k}{\hat{\theta}_n+\hat{\theta}_k}+\frac{2\hat{\theta}_i\hat{\theta}_j(\hat{\theta}_1+\hat{\theta}_n)}{(\hat{\theta}_1+\hat{\theta}_n)(\hat{\theta}_j+\hat{\theta}_n)}\prod_{k=2}^{j-1}\frac{\hat{\theta}_n-\hat{\theta}_k}{\hat{\theta}_n+\hat{\theta}_k}\\
                & &=&\ \ \left(\frac{\hat{\theta}_n-\hat{\theta}_1}{\hat{\theta}_1+\hat{\theta}_n}+\frac{2\hat{\theta}_1}{\hat{\theta}_1+\hat{\theta}_n}\right)\ \frac{\hat{\theta}_j(\hat{\theta}_1+\hat{\theta}_n)}{\hat{\theta}_j+\hat{\theta}_n}\ \prod_{k=2}^{j-1}\frac{\hat{\theta}_n-\hat{\theta}_k}{\hat{\theta}_n+\hat{\theta}_k}\\[5pt]
                & &=&\ \ \frac{\hat{\theta}_j(\hat{\theta}_1+\hat{\theta}_n)}{\hat{\theta}_j+\hat{\theta}_n}\ \prod_{k=2}^{j-1}\frac{\hat{\theta}_n-\hat{\theta}_k}{\hat{\theta}_n+\hat{\theta}_k}.
            \end{aligned}
        \end{equation}
        Assume that the claim \eqref{eq:4claim} holds for $i={i}'<j-1$, we can get the following result for $i={i}'+1<j$:
        \begin{equation}
            \begin{aligned}
                & & &\ \ \hat{\theta}_n\hat{P}_{nj}+\sum_{k=1}^{{i}'+1}\hat{\theta}_k\hat{P}_{kj}\\
                & &=&\ \ \hat{\theta}_n\hat{P}_{nj}+\sum_{k=1}^{{i}'}\hat{\theta}_k\hat{P}_{kj}+\hat{\theta}_{{i}'+1}\hat{P}_{{i}'j}\\
                & &=&\ \ \frac{\hat{\theta}_j(\hat{\theta}_1+\hat{\theta}_n)}{\hat{\theta}_j+\hat{\theta}_n}\prod_{k={i}'+1}^{j-1}\frac{\hat{\theta}_n-\hat{\theta}_k}{\hat{\theta}_n+\hat{\theta}_k}+\frac{2\hat{\theta}_{{i}'+1}\hat{\theta}_j(\hat{\theta}_1+\hat{\theta}_n)}{(\hat{\theta}_{{i}'+1}+\hat{\theta}_n)(\hat{\theta}_j+\hat{\theta}_n)}\prod_{k={i}+2}^{j-1}\frac{\hat{\theta}_n-\hat{\theta}_k}{\hat{\theta}_n+\hat{\theta}_k}\\[5pt]
                & &=&\ \ \frac{\hat{\theta}_j(\hat{\theta}_1+\hat{\theta}_n)}{\hat{\theta}_j+\hat{\theta}_n}\prod_{k={i}'+2}^{j-1}\frac{\hat{\theta}_n-\hat{\theta}_k}{\hat{\theta}_n+\hat{\theta}_k}\left(\frac{\hat{\theta}_n-\hat{\theta}_{{i}'+1}}{\hat{\theta}_{{i}'+1}+\hat{\theta}_n}+\frac{2\hat{\theta}_{{i}'+1}}{\hat{\theta}_{{i}'+1}+\hat{\theta}_n}\right)\\[5pt]
                & &=&\ \ \frac{\hat{\theta}_j(\hat{\theta}_1+\hat{\theta}_n)}{\hat{\theta}_j+\hat{\theta}_n}\prod_{k={i}'+2}^{j-1}\frac{\hat{\theta}_n-\hat{\theta}_k}{\hat{\theta}_n+\hat{\theta}_k}.
            \end{aligned}
        \end{equation}
        Therefore the claim \eqref{eq:4claim} holds by induction. When $j<i<n$, we have $\hat{P}_{ij} = 0$ and
        \begin{equation}
            \begin{aligned}
                & \sum_{k=1}^{n}\hat{\theta}_k\hat{P}_{kj} &=&\ \ \hat{\theta}_j\hat{P}_{jj} + \hat{\theta}_n\hat{P}_{nj} + \sum_{k=1}^{j-1}\hat{\theta}_k\hat{P}_{kj}\\
                & &=&\ \ \frac{\hat{\theta}_j(\hat{\theta}_j-\hat{\theta}_1)}{\hat{\theta}_j+\hat{\theta}_n} + \frac{\hat{\theta}_j(\hat{\theta}_1-\hat{\theta}_n)}{\hat{\theta}_j+\hat{\theta}_n}\\[5pt]
                & &=&\ \ \hat{\theta}_j.
            \end{aligned}
        \end{equation}
        At last, for $j=n$, we have
        \begin{equation}
            \begin{aligned}
                & \sum_{k=1}^{n}\hat{\theta}_k\hat{P}_{kn} &=&\ \ \sum_{k=1}^{n}\hat{\theta}_k\left(1-\sum_{l=1}^{n-1}\hat{P}_{kl}\right)\\
                & &=&\ \ \sum_{k=1}^{n}\hat{\theta}_k - \sum_{l=1}^{n-1}\sum_{k=1}^{n}\hat{\theta}_k\hat{P}_{kl}\\
                & &=&\ \ 1-\sum_{l=1}^{n-1}\hat{\theta}_l = \hat{\theta}_n.
            \end{aligned}
        \end{equation}
        Now we have proved $\boldsymbol{P}_{\mathcal{A}}$ \eqref{eq:irreversibiletransitionmatrix} is invariant with $\boldsymbol{\theta}_{\mathcal{A}}$.
    \end{enumerate}
    \end{proof}

    At last, we prove the construction of \eqref{eq:irreversibiletransitionmatrix} is one of the optimal transition matrices $\boldsymbol{P}_{\mathcal{A}}$ with invariant distribution $\boldsymbol{\theta}_{\mathcal{A}}$ which will lead $\nu(\boldsymbol{P}_{\mathcal{A}})$ to attain its lower bound.
    
    \irreversiblelowerbound*

    \begin{proof}
        By Proposition \ref{prop:1}, we know that the lower bound of $\nu(\boldsymbol{f},\boldsymbol{P})$ is related to \eqref{eq:prop_1_result}. We first show that constructing $\boldsymbol{P}_{\mathcal{A}}$ from \eqref{eq:irreversibiletransitionmatrix} will obtain 
        \begin{equation}
            \label{eq:induction_result}
            \frac{\hat{\theta}_i\hat{\theta}_j}{\hat{\theta}_i+\hat{\theta}_j}\left(\mathbb{E}_i[T_j]+\mathbb{E}_j[T_i]\right)= \frac{\hat{\theta}_n}{\hat{\theta}_1+\hat{\theta}_n},\ \ \forall\ i,j\in[n],\ i\neq j.
        \end{equation}
        Consider a new process $\boldsymbol{X}^{(-k)}$ for $k>1,k\in[n]$, which represents the process $\boldsymbol{X}$ ignoring the time spent in state $k$. First, we claim that $\boldsymbol{X}^{(-k)}$ is still a Markov process with the same type of transition matrix $\boldsymbol{P}^{(-k)}_{\mathcal{A}}=\{\hat{P}^{(-k)}_{ij}\}$ defined by \eqref{eq:irreversibiletransitionmatrix}. By the strong Markov property, $\boldsymbol{X}^{(-k)}$ is still Markov. By the definition of $\boldsymbol{X}^{(-k)}$, we have
        \begin{equation}
            \hat{P}^{(-k)}_{ij} = \hat{P}_{ij} + \frac{\hat{P}_{ik}\hat{P}_{kj}}{1-\hat{P}_{kk}}.
        \end{equation}
        For $i<n$, we have $\hat{P}^{(-k)}_{ij}=\hat{P}_{ij}$ if $k<i$ or $k>j$ as $\hat{P}_{ik}=0$ or $\hat{P}_{kj}=0$. So it suffices to prove the case $i<k<j$. Furthermore, it suffices to prove the case $j<n$ since $\hat{P}_{in}$ of \eqref{eq:irreversibiletransitionmatrix} is $1-\underset{j<n}{\sum}\hat{P}_{ij}$. By \eqref{eq:irreversibiletransitionmatrix}, we have
        \begin{equation}
            \begin{aligned}
                & \frac{\hat{P}_{ik}\hat{P}_{kj}}{1-\hat{P}_{kk}} &=&\ \ \frac{\hat{\theta}_n+\hat{\theta}_k}{\hat{\theta}_n+\hat{\theta}_1}\times\frac{2\hat{\theta}_k(\hat{\theta}_n+\hat{\theta}_1)}{(\hat{\theta}_n+\hat{\theta}_i)(\hat{\theta}_n+\hat{\theta}_k)}\prod_{l=i+1}^{k-1}\frac{\hat{\theta}_n-\hat{\theta}_l}{\hat{\theta}_n+\hat{\theta}_l}\times\frac{2\hat{\theta}_j(\hat{\theta}_n+\hat{\theta}_1)}{(\hat{\theta}_n+\hat{\theta}_k)(\hat{\theta}_n+\hat{\theta}_j)}\prod_{l=k+1}^{j-1}\frac{\hat{\theta}_n-\hat{\theta}_l}{\hat{\theta}_n+\hat{\theta}_l}\\[5pt]
                & &=&\ \ \frac{2\hat{\theta}_k}{\hat{\theta}_n+\hat{\theta}_k}\times\frac{2\hat{\theta}_j(\hat{\theta}_n+\hat{\theta}_1)}{(\hat{\theta}_n+\hat{\theta}_i)(\hat{\theta}_n+\hat{\theta}_j)}\prod_{l=i+1}^{k-1}\frac{\hat{\theta}_n-\hat{\theta}_l}{\hat{\theta}_n+\hat{\theta}_l}\prod_{l=k+1}^{j-1}\frac{\hat{\theta}_n-\hat{\theta}_l}{\hat{\theta}_n+\hat{\theta}_l}.
            \end{aligned}
        \end{equation}
        Then 
        \begin{equation}
            \begin{aligned}
                & \hat{P}^{(-k)}_{ij} &=&\ \ \hat{P}_{ij}+\frac{\hat{P}_{ik}\hat{P}_{kj}}{1-\hat{P}_{kk}}\\[5pt]
                & &=&\ \ \left(\frac{\hat{\theta}_n-\hat{\theta}_k}{\hat{\theta}_n+\hat{\theta}_k}+\frac{2\hat{\theta}_k}{\hat{\theta}_n+\hat{\theta}_k}\right)\times\frac{2\hat{\theta}_j(\hat{\theta}_n+\hat{\theta}_1)}{(\hat{\theta}_n+\hat{\theta}_i)(\hat{\theta}_n+\hat{\theta}_j)}\prod_{l=i+1}^{k-1}\frac{\hat{\theta}_n-\hat{\theta}_l}{\hat{\theta}_n+\hat{\theta}_l}\prod_{l=k+1}^{j-1}\frac{\hat{\theta}_n-\hat{\theta}_l}{\hat{\theta}_n+\hat{\theta}_l}\\[5pt]
                & &=&\ \ \frac{2\hat{\theta}_j(\hat{\theta}_n+\hat{\theta}_1)}{(\hat{\theta}_n+\hat{\theta}_i)(\hat{\theta}_n+\hat{\theta}_j)}\prod_{l=i+1}^{k-1}\frac{\hat{\theta}_n-\hat{\theta}_l}{\hat{\theta}_n+\hat{\theta}_l}\prod_{l=k+1}^{j-1}\frac{\hat{\theta}_n-\hat{\theta}_l}{\hat{\theta}_n+\hat{\theta}_l}.
            \end{aligned}
        \end{equation}
        When $i=n$,
        \begin{equation}
            \begin{aligned}
                & \frac{\hat{P}_{nk}\hat{P}_{kj}}{1-\hat{P}_{kk}} &=&\ \ \frac{\hat{\theta}_n+\hat{\theta}_k}{\hat{\theta}_n+\hat{\theta}_1}\times\frac{\hat{\theta}_k(\hat{\theta}_n+\hat{\theta}_1)}{\hat{\theta}_n(\hat{\theta}_n+\hat{\theta}_k)}\prod_{l=i+1}^{k-1}\frac{\hat{\theta}_n-\hat{\theta}_l}{\hat{\theta}_n+\hat{\theta}_l}\times\frac{2\hat{\theta}_j(\hat{\theta}_n+\hat{\theta}_1)}{(\hat{\theta}_n+\hat{\theta}_k)(\hat{\theta}_n+\hat{\theta}_j)}\prod_{l=k+1}^{j-1}\frac{\hat{\theta}_n-\hat{\theta}_l}{\hat{\theta}_n+\hat{\theta}_l}\\[5pt]
                & &=&\ \ \frac{2\hat{\theta}_k}{\hat{\theta}_n+\hat{\theta}_k}\times\frac{\hat{\theta}_j(\hat{\theta}_n+\hat{\theta}_1)}{\hat{\theta}_n(\hat{\theta}_n+\hat{\theta}_j)}\prod_{l=i+1}^{k-1}\frac{\hat{\theta}_n-\hat{\theta}_l}{\hat{\theta}_n+\hat{\theta}_l}\prod_{l=k+1}^{j-1}\frac{\hat{\theta}_n-\hat{\theta}_l}{\hat{\theta}_n+\hat{\theta}_l},
            \end{aligned}
        \end{equation}
        and
        \begin{equation}
            \begin{aligned}
                & \hat{P}^{(-k)}_{nj} &=&\ \ \hat{P}_{nj}+\frac{\hat{P}_{nk}\hat{P}_{kj}}{1-\hat{P}_{kk}}\\[5pt]
                & &=&\ \ \left(\frac{\hat{\theta}_n-\hat{\theta}_k}{\hat{\theta}_n+\hat{\theta}_k}+\frac{2\hat{\theta}_k}{\hat{\theta}_n+\hat{\theta}_k}\right)\times\frac{\hat{\theta}_j(\hat{\theta}_n+\hat{\theta}_1)}{\hat{\theta}_n(\hat{\theta}_n+\hat{\theta}_j)}\prod_{l=i+1}^{k-1}\frac{\hat{\theta}_n-\hat{\theta}_l}{\hat{\theta}_n+\hat{\theta}_l}\prod_{l=k+1}^{j-1}\frac{\hat{\theta}_n-\hat{\theta}_l}{\hat{\theta}_n+\hat{\theta}_l}\\[5pt]
                & &=&\ \ \frac{\hat{\theta}_j(\hat{\theta}_n+\hat{\theta}_1)}{\hat{\theta}_n(\hat{\theta}_n+\hat{\theta}_j)}\prod_{l=i+1}^{k-1}\frac{\hat{\theta}_n-\hat{\theta}_l}{\hat{\theta}_n+\hat{\theta}_l}\prod_{l=k+1}^{j-1}\frac{\hat{\theta}_n-\hat{\theta}_l}{\hat{\theta}_n+\hat{\theta}_l}.             
            \end{aligned}
        \end{equation}
        These complete the verification of claim that the transition matrix of $\boldsymbol{X}^{(-k)}$ has the same formulation as \eqref{eq:irreversibiletransitionmatrix}.

        Next we shall prove the result by mathematical induction. Starting with the case $n=3$, \eqref{eq:irreversibiletransitionmatrix} shows that
        \begin{equation}
            \boldsymbol{P}_{\mathcal{A}} = \begin{pmatrix}
            \displaystyle 0 & \displaystyle\frac{2\hat{\theta}_2}{\hat{\theta}_2+\hat{\theta}_3} & \displaystyle\frac{\hat{\theta}_3-\hat{\theta}_2}{\hat{\theta}_3+\hat{\theta}_2}\\[15pt]
            \displaystyle 0 & \displaystyle\frac{\hat{\theta}_2-\hat{\theta}_1}{\hat{\theta}_3+\hat{\theta}_2} & \displaystyle\frac{\hat{\theta}_3+\hat{\theta}_1}{\hat{\theta}_3+\hat{\theta}_2}\\[15pt]
            \displaystyle\frac{\hat{\theta}_1}{\hat{\theta}_3} & \displaystyle\frac{\hat{\theta}_2(\hat{\theta}_3-\hat{\theta}_1)}{\hat{\theta}_2(\hat{\theta}_3+\hat{\theta}_2)} & \displaystyle\frac{\hat{\theta}_3-\hat{\theta}_1}{\hat{\theta}_3+\hat{\theta}_2}.
            \end{pmatrix}
        \end{equation}
        Recall from Lemma \ref{lemma:1}, we have 
        \begin{equation*}
            \hat{\theta}_j(\mathbb{E}_i[T_j]+\mathbb{E}_j[T_i]) = \frac{1}{\mathbb{P}_j(T_i<T_j^+)}.
        \end{equation*}
        For $j=1$ and $i=3$, $\mathbb{P}_1(T_3<T_1^+)=1$ as the chain $\boldsymbol{X}$ starting from $1$ must first go to $3$ before return $1$. Therefore \eqref{eq:induction_result} holds. For $j=2$ and $i=3$, since visiting $1$ means arriving at $3$ earlier, the chain $\boldsymbol{X}$ will visit $3$ before returning to $2$ unless it returns immediately. That is,
        \begin{equation}
            \mathbb{P}_2(T_3<T_2^+) = \hat{P}_{23} = \frac{\hat{\theta}_3+\hat{\theta}_1}{\hat{\theta}_3+\hat{\theta}_2}.
        \end{equation}
        With Lemma \ref{lemma:1}, \eqref{eq:induction_result} holds. For $j=1$ and $i=2$, the probability that  chain $\boldsymbol{X}$ returning to $1$ without visiting $2$ equals to the probability of going to $3$ from $1$ multiplied by the probability of arriving at $1$ after leaving from $3$. That is 
        \begin{equation}
            \mathbb{P}_1(T_2<T_1^+) = 1-\mathbb{P}_1(T^+_1<T_2)=1-\hat{P}_{13}\frac{\hat{P}_{31}}{\hat{P}_{31}+\hat{P}_{32}} = \frac{\hat{\theta}_2(\hat{\theta}_1+\hat{\theta}_3)}{\hat{\theta}_3(\hat{\theta}_1+\hat{\theta}_2)}.
        \end{equation}
        Consequently, we have 
        \begin{equation}
            \begin{aligned}
                & \frac{\hat{\theta}_1\hat{\theta}_2}{\hat{\theta}_1+\hat{\theta}_2}(\mathbb{E}_2[T_1]+\mathbb{E}_1[T_2]) &=&\ \ \frac{\hat{\theta}_2}{\hat{\theta}_1+\hat{\theta}_2}\cdot\frac{1}{\mathbb{P}_1(T_2<T_1^+)}\\[5pt]
                & &=&\ \ \frac{\hat{\theta}_2}{\hat{\theta}_1+\hat{\theta}_2}\cdot\frac{\hat{\theta}_3(\hat{\theta}_1+\hat{\theta}_2)}{\hat{\theta}_2(\hat{\theta}_1+\hat{\theta}_3)}\\[5pt]
                & &=&\ \ \frac{\hat{\theta}_3}{\hat{\theta}_1+\hat{\theta}_3}.
            \end{aligned}
        \end{equation}
        Assume that the claim \eqref{eq:induction_result} is true for the number of state being $n-1$ $(n >3)$. We want to prove that \eqref{eq:induction_result} still holds when the number of state being $n$. This statement is ture when $i=1$ and $j=n$ by the same argument for the case $n=3$. Then, for any unorder pair $[i,j],i\neq j$, we can choose a state $k\neq 1, i, j,n$ and it holds
        \begin{equation}
              \frac{\hat{\theta}_i^{(-k)}\hat{\theta}_j^{(-k)}}{\hat{\theta}_i^{(-k)}+\hat{\theta}_j^{(-k)}}\left(\mathbb{E}_i[T^{(-k)}_j]+\mathbb{E}_j[T^{(-k)}_i]\right)=\frac{\hat{\theta}_n^{(-k)}}{\hat{\theta}_1^{(-k)}+\hat{\theta}_n^{(-k)}},
        \end{equation}
    where the super-index $(-k)$ represents the associated statistic for the Markov chain $\boldsymbol{X}$ ignoring state $k$. On one hand, from Lemma 9 in \cite[Chapter 2]{aldous-fill-2014}, we know that the number of visiting $k$ from $i$ before time $T_j$ is $\hat{\theta}_k(\mathbb{E}_i[T_j]+\mathbb{E}_j[T_i]-\mathbb{E}_i[T_k])$. On the other hand, this number equals to $\mathbb{E}_i[T_j]-\mathbb{E}_i[T^{(-k)}_j]$. So
    \begin{equation}
        \hat{\theta}_k\left(\mathbb{E}_i[T_j]+\mathbb{E}_j[T_i]-\mathbb{E}_i[T_k]\right)=\mathbb{E}_i[T_j]-\mathbb{E}_i[T^{(-k)}_j].
    \end{equation}
    Change the role of $i$ and $j$, we have
    \begin{equation}
        \hat{\theta}_k\left(\mathbb{E}_j[T_i]+\mathbb{E}_i[T_k]-\mathbb{E}_j[T_k]\right)=\mathbb{E}_j[T_i]-\mathbb{E}_j[T^{(-k)}_i].
    \end{equation}
    Sum the above two equations and we obtain
    \begin{equation}
           (1-\hat{\theta}_k)\left(\mathbb{E}_i[T_j]+\mathbb{E}_j[T_i]\right)=\mathbb{E}_i[T^{(-k)}_j]+\mathbb{E}_j[T^{(-k)}_i].
    \end{equation}
    Then we have
    \begin{equation}
        \begin{aligned}
            & \frac{\hat{\theta}_i\hat{\theta}_j}{\hat{\theta}_i+\hat{\theta}_j}\left(\mathbb{E}_i[T_j]+\mathbb{E}_j[T_i]\right) &=&\ \ \frac{\hat{\theta}_i\hat{\theta}_j}{\hat{\theta}_i+\hat{\theta}_j}\cdot\frac{1}{1-\hat{\theta}_k}\cdot\left(\mathbb{E}_i[T^{(-k)}_j]+\mathbb{E}_j[T^{(-k)}_i]\right)\\[5pt]
            & &=&\ \ \frac{\hat{\theta}^{(-k)}_i\hat{\theta}^{(-k)}_j}{\hat{\theta}^{(-k)}_i+\hat{\theta}^{(-k)}_j}\cdot\left(\mathbb{E}_i[T^{(-k)}_j]+\mathbb{E}_j[T^{(-k)}_i]\right)\\[5pt]
            & &=&\ \ \frac{\hat{\theta}^{(-k)}_n}{\hat{\theta}^{(-k)}_1+\hat{\theta}^{(-k)}_n}\\[5pt]
            & &=&\ \ \frac{\hat{\theta}_n}{\hat{\theta}_1+\hat{\theta}_n}
        \end{aligned}
    \end{equation}
    and this complete the proof of \eqref{eq:induction_result}. 

    Next we will give the exact value of $\langle\boldsymbol{Z}\boldsymbol{f},\boldsymbol{f}\rangle$ when $\boldsymbol{f}\in\mathcal{F}$. Recall the definition of $\boldsymbol{f}_{ij}$ \eqref{eq:f_ij} and Proposition \ref{prop:1}, we have
    \begin{equation}
        \left\langle \boldsymbol{Z}\frac{\boldsymbol{f}_{ij}}{\|\boldsymbol{f}_{ij}\|_{\boldsymbol{\theta}_{\mathcal{A}}}}, \frac{\boldsymbol{f}_{ij}}{\|\boldsymbol{f}_{ij}\|_{\boldsymbol{\theta}_{\mathcal{A}}}}\right\rangle_{\boldsymbol{\theta}_{\Scale[0.5]{\mathcal{A}}}} = \frac{\hat{\theta}_i\hat{\theta}_j}{\hat{\theta}_i+\hat{\theta}_j}\cdot\Big(\mathbb{E}_j[T_i] + \mathbb{E}_i[T_j]\Big),\ \ \forall\ i\neq j. 
        \tag{\ref{eq:inter_result} revisited}
    \end{equation}
    Substitute the right hand of \eqref{eq:inter_result} with \eqref{eq:induction_result}, it becomes
    \begin{equation}
        \left\langle \boldsymbol{Z}\frac{\boldsymbol{f}_{ij}}{\|\boldsymbol{f}_{ij}\|_{\boldsymbol{\theta}_{\mathcal{A}}}}, \frac{\boldsymbol{f}_{ij}}{\|\boldsymbol{f}_{ij}\|_{\boldsymbol{\theta}_{\mathcal{A}}}}\right\rangle_{\boldsymbol{\theta}_{\Scale[0.5]{\mathcal{A}}}} =\ \ \frac{\hat{\theta}_n}{\hat{\theta}_1+\hat{\theta}_n},
    \end{equation}
    or equivalently
    \begin{equation}
        \Big\langle \boldsymbol{Z}\boldsymbol{f}_{ij}, \boldsymbol{f}_{ij}\Big\rangle_{\boldsymbol{\theta}_{\Scale[0.5]{\mathcal{A}}}} =\ \ \frac{\hat{\theta}_n}{\hat{\theta}_1+\hat{\theta}_n}\Big\langle\boldsymbol{f}_{ij}, \boldsymbol{f}_{ij}\Big\rangle_{\boldsymbol{\theta}_{\Scale[0.5]{\mathcal{A}}}} =\ \ \frac{\hat{\theta}_n}{\hat{\theta}_1+\hat{\theta}_n}\left(\frac{1}{\hat{\theta}_i}+\frac{1}{\hat{\theta}_j}\right).
    \end{equation}
    Our goal is to prove
    \begin{equation}
         \Big\langle \boldsymbol{Z}\boldsymbol{f}, \boldsymbol{f}\Big\rangle_{\boldsymbol{\theta}_{\Scale[0.5]{\mathcal{A}}}} =\ \ \frac{\hat{\theta}_n}{\hat{\theta}_1+\hat{\theta}_n}\Big\langle\boldsymbol{f}, \boldsymbol{f}\Big\rangle_{\boldsymbol{\theta}_{\Scale[0.5]{\mathcal{A}}}},\ \ \forall\ \boldsymbol{f}\in\mathcal{F}.
    \end{equation}
    Since $\{\boldsymbol{f}_{in}\},i\in[n-1]$ are linearly independent, $\{\boldsymbol{f}_{in}\},i\in[n-1]$ forms a set of basis of $\mathcal{F}$. For any $\boldsymbol{f}\in\mathcal{F}$,
    \begin{equation}
        \boldsymbol{f} = \sum_{i=1}^{n-1}\alpha_i\boldsymbol{f}_{in},
    \end{equation}
    and $\boldsymbol{f}_{ij} = \boldsymbol{f}_{in} - \boldsymbol{f}_{jn}$. Then
    \begin{equation}
        \begin{aligned}
            & \Big\langle \boldsymbol{Z}\boldsymbol{f}_{ij}, \boldsymbol{f}_{ij}\Big\rangle_{\boldsymbol{\theta}_{\Scale[0.5]{\mathcal{A}}}} &=&\ \ \Big\langle \boldsymbol{Z}(\boldsymbol{f}_{in} - \boldsymbol{f}_{jn}), \boldsymbol{f}_{in} - \boldsymbol{f}_{jn}\Big\rangle_{\boldsymbol{\theta}_{\Scale[0.5]{\mathcal{A}}}}\\[7.5pt]
            & &=&\ \ \Big\langle \boldsymbol{Z}\boldsymbol{f}_{in}, \boldsymbol{f}_{in}\Big\rangle_{\boldsymbol{\theta}_{\Scale[0.5]{\mathcal{A}}}}+\Big\langle \boldsymbol{Z}\boldsymbol{f}_{jn}, \boldsymbol{f}_{jn}\Big\rangle_{\boldsymbol{\theta}_{\Scale[0.5]{\mathcal{A}}}} - \Big\langle \boldsymbol{Z}\boldsymbol{f}_{in}, \boldsymbol{f}_{jn}\Big\rangle_{\boldsymbol{\theta}_{\Scale[0.5]{\mathcal{A}}}} - \Big\langle \boldsymbol{Z}\boldsymbol{f}_{jn}, \boldsymbol{f}_{in}\Big\rangle_{\boldsymbol{\theta}_{\Scale[0.5]{\mathcal{A}}}}.
        \end{aligned}
    \end{equation}
    Then
    \begin{equation}
        \begin{aligned}
            & \Big\langle \boldsymbol{Z}\boldsymbol{f}_{in}, \boldsymbol{f}_{in}\Big\rangle_{\boldsymbol{\theta}_{\Scale[0.5]{\mathcal{A}}}}+\Big\langle \boldsymbol{Z}\boldsymbol{f}_{jn}, \boldsymbol{f}_{jn}\Big\rangle_{\boldsymbol{\theta}_{\Scale[0.5]{\mathcal{A}}}}&=&\ \ \Big\langle \boldsymbol{Z}\boldsymbol{f}_{in}, \boldsymbol{f}_{jn}\Big\rangle_{\boldsymbol{\theta}_{\Scale[0.5]{\mathcal{A}}}} + \Big\langle \boldsymbol{Z}\boldsymbol{f}_{jn}, \boldsymbol{f}_{in}\Big\rangle_{\boldsymbol{\theta}_{\Scale[0.5]{\mathcal{A}}}} + \Big\langle \boldsymbol{Z}\boldsymbol{f}_{ij}, \boldsymbol{f}_{ij}\Big\rangle_{\boldsymbol{\theta}_{\Scale[0.5]{\mathcal{A}}}}\\[5pt]
            & &=&\ \ \frac{\hat{\theta}_n}{\hat{\theta}_n+\hat{\theta}_1}\left(\frac{1}{\ \hat{\theta}_n}+\frac{1}{\ \hat{\theta}_i}+\frac{1}{\ \hat{\theta}_n}+\frac{1}{\ \hat{\theta}_j}-\frac{1}{\ \hat{\theta}_i}-\frac{1}{\ \hat{\theta}_j}\right)\\[5pt]
            & &=&\ \ \frac{\hat{\theta}_n}{\hat{\theta}_n+\hat{\theta}_1}\cdot\frac{2}{\ \hat{\theta}_n}\\[5pt]
            & &=&\ \ \frac{2\ \hat{\theta}_n}{\hat{\theta}_n+\hat{\theta}_1}\cdot\Big\langle\boldsymbol{f}_{in},\boldsymbol{f}_{jn}\Big\rangle_{\boldsymbol{\theta}_{\Scale[0.5]{\mathcal{A}}}}.
        \end{aligned}
    \end{equation}
    Consequently, we have
    \begin{equation}
        \begin{aligned}
            & \Big\langle \boldsymbol{Z}\boldsymbol{f}, \boldsymbol{f}\Big\rangle_{\boldsymbol{\theta}_{\Scale[0.5]{\mathcal{A}}}} &=&\ \ \left\langle \boldsymbol{Z}\sum_{i=1}^{n-1}\alpha_i\boldsymbol{f}_{in}, \sum_{j=1}^{n-1}\alpha_j\boldsymbol{f}_{jn}\right\rangle_{\boldsymbol{\theta}_{\Scale[0.5]{\mathcal{A}}}}\\[5pt]
            & &=&\ \ \sum_{i=1}^{n-1}\alpha^2_i\Big\langle \boldsymbol{Z}\boldsymbol{f}_{in}, \boldsymbol{f}_{in}\Big\rangle_{\boldsymbol{\theta}_{\Scale[0.5]{\mathcal{A}}}} + \underset{i<j}{\sum}\alpha_i\alpha_j\Big(\Big\langle \boldsymbol{Z}\boldsymbol{f}_{in}, \boldsymbol{f}_{jn}\Big\rangle_{\boldsymbol{\theta}_{\Scale[0.5]{\mathcal{A}}}}+\Big\langle \boldsymbol{Z}\boldsymbol{f}_{jn}, \boldsymbol{f}_{in}\Big\rangle_{\boldsymbol{\theta}_{\Scale[0.5]{\mathcal{A}}}}\Big)\\[5pt]
            & &=&\ \ \frac{\hat{\theta}_n}{\hat{\theta}_n+\hat{\theta}_1} \sum_{i=1}^{n-1}\alpha^2_i\Big\langle\boldsymbol{f}_{in}, \boldsymbol{f}_{in}\Big\rangle_{\boldsymbol{\theta}_{\Scale[0.5]{\mathcal{A}}}} + \frac{\hat{\theta}_n}{\hat{\theta}_n+\hat{\theta}_1}\underset{i<j}{\sum}\alpha_i\alpha_j\Big(\Big\langle\boldsymbol{f}_{in}, \boldsymbol{f}_{jn}\Big\rangle_{\boldsymbol{\theta}_{\Scale[0.5]{\mathcal{A}}}}+\Big\langle\boldsymbol{f}_{jn}, \boldsymbol{f}_{in}\Big\rangle_{\boldsymbol{\theta}_{\Scale[0.5]{\mathcal{A}}}}\Big)\\[5pt]
            & &=&\ \ \frac{\hat{\theta}_n}{\hat{\theta}_n+\hat{\theta}_1}\left\langle\sum_{i=1}^{n-1}\alpha_i\boldsymbol{f}_{in},\sum_{j=1}^{n-1}\alpha_j\boldsymbol{f}_{jn}\right\rangle_{\boldsymbol{\theta}_{\Scale[0.5]{\mathcal{A}}}}\\[5pt]
            & &=&\ \ \frac{\hat{\theta}_n}{\hat{\theta}_n+\hat{\theta}_1}\Big\langle\boldsymbol{f},\boldsymbol{f}\Big\rangle_{\boldsymbol{\theta}_{\Scale[0.5]{\mathcal{A}}}}.
        \end{aligned}
    \end{equation}
    Now we have proved that
    \begin{equation}
        v(\boldsymbol{P}_{\mathcal{A}}) = \underset{\boldsymbol{f}\in\mathcal{F},\ \|\boldsymbol{f}\|_{\Scale[0.5]{\boldsymbol{\theta}_{\mathcal{A}}}}=1}{\textbf{\textit{sup}}} \Big\langle\big(\boldsymbol{I}-\boldsymbol{P}_{\mathcal{A}}\big)^{-1}\boldsymbol{f},\ \boldsymbol{f}\Big\rangle_{\boldsymbol{\theta}_{\mathcal{A}}} = \frac{\hat{\theta}_n}{\hat{\theta}_1+\hat{\theta}_n}
    \end{equation}
    with $\boldsymbol{P}_{\mathcal{A}}$ defined by \eqref{eq:irreversibiletransitionmatrix}. Along with Theorem \ref{thm:irrevisible_matrix}, we know that
    \begin{equation}
        \underset{\boldsymbol{P}_{\mathcal{A}}}{\textbf{\textit{inf}}}\ v(\boldsymbol{P}_{\mathcal{A}}) = \frac{\hat{\theta}_n}{\hat{\theta}_n+\hat{\theta}_1}.
    \end{equation}
    As a consequence, constructing the transition matrix like \eqref{eq:irreversibiletransitionmatrix} would lead the corresponding asymptotic variance to attain its lower bound. 
    \end{proof}
\end{document}